\documentclass[letterpaper, 10 pt, journal]{IEEEtran}

\IEEEoverridecommandlockouts                              

\usepackage{graphics} 
\usepackage{graphicx} 
\usepackage{float} 
\usepackage{amssymb}  
\usepackage{caption}
\usepackage{subcaption}
\usepackage{amsmath}

\usepackage{amsthm}
\newtheorem{definition}{Definition}
\newtheorem{proposition}{Proposition}
\DeclareMathOperator*{\argmax}{arg\,max}
\DeclareMathOperator*{\argmin}{arg\,min}
\usepackage{hyperref}

\include{utils.tex}

\usepackage{import}
\usepackage{xifthen}
\usepackage{pdfpages}
\usepackage{transparent}
\usepackage{hyperref}
\usepackage{bm}

\title{ \LARGE \bf Iterated Invariant EKF for Quadruped Robot Odometry}

 \author{Hilton Marques Souza Santana$^{1,2*}$, João Carlos Virgolino Soares$^{1}$, Sven Goffin$^{3}$, Ylenia Nisticò$^{1}$, Silvère Bonnabel$^{4}$, Claudio Semini$^{1}$, Marco Antonio Meggiolaro$^{2}$
 \thanks{*Corresponding author, {\tt\small hiltonmarquess@gmail.com}}
 \thanks{This study was financed in part by the Coordenação de Aperfeiçoamento de Pessoal de Nível Superior – Brasil (CAPES) – Finance Code 001, the Fundação de Amparo à Pesquisa do Estado do Rio de Janeiro (FAPERJ), the European Union – NextGenerationEU, the PNRR MUR Project PE000013 “Future Artificial Intelligence Research (FAIR)”. CNPq - Brasil also provided a fellowship for Prof. M.A. Meggiolaro.}
 \thanks{$^{1}$Dynamic Legged Systems Lab, Istituto Italiano di Tecnologia, Genoa, Italy.}
 \thanks{$^{2}$Department of Mechanical Engineering, Pontifical Catholic University of Rio de Janeiro, Brazil.}
 \thanks{$^{3}$Department of Electrical Engineering and Computer Science, University of Liège, Belgium.}
 \thanks{$^{4}$Mines Paris PSL, Centre for Robotics, 60 bd Saint-Michel, 75006 Paris, France.}
 } 

\begin{document}

\maketitle

\begin{abstract}
Kalman filter–based algorithms are fundamental for mobile robots, as they provide a computationally efficient solution to the challenging problem of state estimation. However, they rely on two main assumptions that are difficult to satisfy in practice: (a) the system dynamics must be linear with Gaussian process noise, and (b) the measurement model must also be linear with Gaussian measurement noise.
Previous works have extended assumption (a) to nonlinear spaces through the Invariant Extended Kalman Filter (IEKF), showing that it retains properties similar to those of the classical Kalman filter when the system dynamics are group-affine on a Lie group. More recently, the counterpart of assumption (b) for the same nonlinear setting was addressed in \cite{goffin2025}. By means of the proposed Iterated Invariant Extended Kalman Filter (IterIEKF), the authors of that work demonstrated that the update step exhibits several compatibility properties of the classical linear Kalman filter. 
In this work, we introduce a novel open-source state estimation algorithm for legged robots based on the IterIEKF. The update step of the proposed filter relies solely on proprioceptive measurements, exploiting kinematic constraints on foot velocity during contact and base-frame velocity, making it inherently robust to environmental conditions. Through extensive numerical simulations and evaluation on real-world datasets, we demonstrate that the IterIEKF outperforms the vanilla IEKF, the $\mathrm{SO}(3)$-based Kalman Filter, and its iterated variant in terms of both accuracy and consistency.
\end{abstract}

\section{INTRODUCTION}

Contact-aided inertial state estimation for legged robots is characterized by nonlinear dynamics, hybrid contact switching, and intrinsic unobservable directions such as global position and yaw. While Extended Kalman Filter (EKF)–based approaches remain widely used due to their computational efficiency, their statistical consistency critically depends on how estimation errors are modeled and linearized.

Over the past decade, filtering-based algorithms have been widely adopted for state estimation in legged robots. Most approaches follow the traditional complementary-filter framework used in navigation \cite[p. 393]{brown1997}, in which an Inertial Measurement Unit (IMU) is rigidly mounted on the robot’s body. The resulting high-frequency inertial measurements are integrated to provide a motion estimate that largely bypasses the need to model the platform’s full dynamics. However, due to sensor noise and biases, this integration rapidly drifts from the true trajectory. To correct this nominal estimate in a proprioceptive setting, observations are typically incorporated under a static-contact assumption. Common choices include comparing forward-kinematics predictions with body motion inferred from foot contacts \cite{bloesch2012,Hartley2019ContactaidedIE}, or using the average foot velocities expressed in the base frame as an estimate of the base velocity \cite{bloesch2013}. It was shown in these works \cite{bloesch2012,bloesch2013}, that by means of this proprioceptive measurement, 
the velocity in the base frame, pitch, and roll angles are observables.

Within the Kalman filtering framework, two main methodologies are commonly employed. The first is the classic EKF, which represents the robot’s orientation using either quaternions or the $\mathrm{SO}(3)$ group \cite{bloesch2012,camurri2020, nistico2025}. In our paper, this approach is referred to as the $\mathrm{SO}(3)$-EKF. The second methodology  \cite{Hartley2019ContactaidedIE,lin2023} is based on the Invariant Extended Kalman Filter (IEKF) \cite{barrau2016}, in which orientation and position are no longer treated as independent variables, enabling a more faithful propagation of the robot’s probability distribution.
\begin{figure}[t!]
		\centering
		\setlength{\fboxsep}{0pt}
		\fbox{\includegraphics[scale=0.19]{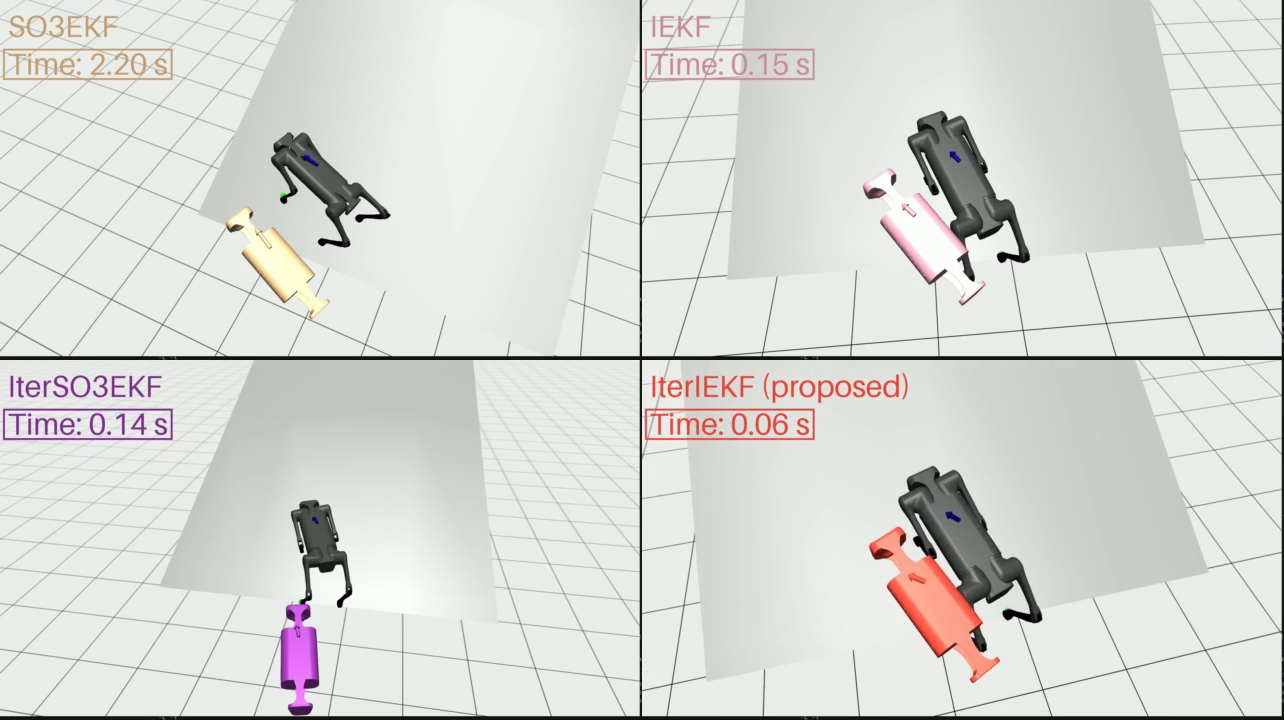}}
\caption{
	Time of convergence of the four filters considered in this work (IterIEKF, IEKF, $\mathrm{SO}(3)$-EKF and Iter$\mathrm{SO}(3)$-EKF) with the same initial state, process and measurement white noise. 
	The simulation stops when the error in the base velocity is less than $0.05\, \mathrm{m / s}$ and the error in gravity direction is less than $0.015\, \mathrm{rad}$.
The result shows that the IterIEKF converges faster than the other three filters.
}
		\label{fig:mujoco_simulation}
\end{figure}

Both methodologies share a common limitation: during the update step, the measurement function is linearized in order to obtain a Gaussian approximation of the posterior, whose mean is then used as the state estimate. However, it can be shown that the mean of this approximated Gaussian may differ substantially from a meaningful estimate of the true posterior distribution, and in particular from its mode. The mode, defined as the maximum of the posterior distribution, is computed by solving a nonlinear minimization problem using the Iterated Extended Kalman Filter~\cite[p. 279]{jazwinski1970}, provided that the robot state evolves in a vector space. An algorithmic approach to solve this minimization problem is to recast it as a Gauss–Newton iteration following \cite{bell1994iterated}. This methodology has been extended to nonlinear state spaces: \cite{bourmaud2016} considers Lie groups, while \cite{he2021} generalizes the framework to manifolds. In this work, we refer to the approach of \cite{bourmaud2016} as Iter$\mathrm{SO}(3)$-EKF. A limitation of this methodology is that it does not incorporate the invariant observer theory proposed in \cite{barrau2016}. A principled extension in this direction was recently proposed by Goffin et al.~\cite{goffin2025}, who introduced the Iterated Invariant EKF (IterIEKF). The IterIEKF is an invariant filtering approach that preserves some compatibility properties of the linear Kalman filter in the presence of noise-free measurements. By incorporating a Gauss–Newton-based iterative update, the IterIEKF computes a maximum a posteriori (MAP) estimate and can significantly improve state estimation performance in nonlinear systems.

However, the interaction between invariance and iteration in hybrid contact-aided legged systems remains insufficiently understood. In practice, velocity-based contact measurements and contact switching may partially violate strict invariance assumptions. It is therefore unclear whether iterative refinement can compensate for structural inconsistencies when invariance is not preserved, or whether iteration instead amplifies incorrect observability modeling. This work investigates this question through a complete geometric derivation and systematic evaluation of four filtering architectures: $\mathrm{SO(3)}$-EKF, Iter$\mathrm{SO(3)}$-EKF, IEKF, and IterIEKF (see Fig. \ref{fig:mujoco_simulation}), using Monte Carlo simulations and dataset analysis to clarify the effects of symmetry preservation and iterative refinement in convergence behavior and statistical consistency.

\subsection{Contributions}
\label{sec:contributions}

To the best of our knowledge, this is the first formulation of the IterIEKF for quadruped robots. Unlike~\cite{Hartley2019ContactaidedIE}, we rely only on velocity in the base frame as a measurement update. We compare our proposed filter with three other filters: IEKF, $\mathrm{SO(3)}$-EKF and Iter$\mathrm{SO(3)}$-EKF. The performance improvement is shown in simulations and a real-world dataset. Our key contributions can be summarized as follows:

\begin{itemize}

\item We derive formulations of invariant and non-invariant iterated filtering for contact-aided quadruped state estimation, explicitly accounting for foot contact velocity measurements expressed in the base frame. We emphasize that, to the best of our knowledge, the formulations of IterIEKF and Iter$\mathrm{SO}(3)$-EKF are presented for the first time in this setting.

\item We provide, to our knowledge, the first systematic study of the impact of iterative invariant updates on the accuracy and statistical consistency of legged robot state estimation. Through controlled Monte Carlo simulations, we show that IterIEKF significantly improves NEES consistency and accuracy compared to single-step invariant filtering.
More specifically, simulations indicate a reduction of up to $60 \%$ in the error of base velocity and gravity direction estimates. In real-world dataset, we observe a reduction of approximately $12 \%$ in the pitch angle estimation error.

\item We demonstrate that combining invariance with iterative refinement reduces convergence time under large initialization errors and measurement uncertainty, 
	while preserving statistical consistency. In some experiments, a reduction of up to $60 \%$ 
	in convergence time is observed for the IterIEKF to reach the same error level as the IEKF.

\end{itemize}

Aditionally, we open-source the proposed framework\footnote{\url{https://hilton-santana.github.io/Legged-IterIEKF/}}, including simulation tools, to facilitate reproducibility and further research.

The remainder of this paper is organized as follows. Sec. II shows the literature review. Sec. III outlines the mathematical background required to understand the proposed approach. Sec. IV presents a simple motivating example that highlights the characteristics of the IterIEKF in comparison with other methods. Sec. V describes the complete methodology for developing the IterIEKF and the Iter$\mathrm{SO}(3)$-EKF, including a discussion of uncertainty in the observation model. Sec. VI presents the experimental results and provides a detailed analysis. Sec. VII discusses the limitations of the proposed method, and Sec. VIII concludes the paper and outlines directions for future work.

\section{LITERATURE REVIEW}

State estimation for legged robots is fundamentally constrained by nonlinear inertial dynamics, intermittent contact constraints, and intrinsic unobservabilities (global position and yaw). Classical error-state EKF formulations~\cite{bloesch2012} operate in local coordinates and rely on first-order linearization around a nominal trajectory. While computationally efficient, these approaches are known to suffer from inconsistency when the linearized error dynamics do not preserve the system’s symmetry structure, as formalized in the context of SLAM in~\cite{barrau2015}.

Invariant filtering~\cite{barrau2016},~\cite{barrau2018} addresses this limitation by defining the estimation error on matrix Lie groups and exploiting the group-affine system structure. When the dynamics and observations satisfy appropriate invariance properties, the resulting error dynamics become autonomous, i.e., the error dynamic does not depend on the system state, leading to improved consistency and reduced sensitivity to linearization. The theoretical foundations of non-linear error-based filtering~\cite{barrau2015thesis}, and Lie group state estimation~\cite{sola2018},~\cite{hall2013} provide the structural justification for this approach.

For legged robots,~\cite{Hartley2019ContactaidedIE} formulated contact-aided inertial navigation on $\mathrm{SE}_{2+N}(3)$, demonstrating that invariant filtering improves accuracy relative to $\mathrm{SO}(3)$-EKF implementations. This formulation leverages the natural embedding of pose, velocity, and contact positions in an extended Lie group, aligning propagation with the geometry of $\mathrm{SE}_2(3)$. The geometric interpretation of IMU-driven dynamics and uncertainty transport further supports this construction~\cite{brossard2022}.

However, strict invariance requires the observation model to respect the underlying group action. In practical legged systems, measurements such as base-frame velocity or contact-relative constraints often violate full invariance, weakening theoretical guarantees. This occurs, for example, under the rigid contact assumption, where foot velocities are constrained to be zero in the world frame. In such cases, performance improvements become empirical rather than structurally guaranteed.
Legged locomotion introduces hybrid dynamics through contact switching. Early approaches relied on deterministic contact detection via force thresholds~\cite{bloesch2013},~\cite{wisth2023}, an assumption shown to be fragile on compliant or deformable terrain~\cite{fahmi2021}. From an estimation-theoretic standpoint, contact modeling affects both observability and covariance contraction.

Learning-based approaches~\cite{lin2021}~\cite{youm2024} infer contact events or velocity estimates from proprioceptive data and inject them as pseudo-measurements into an IEKF framework. While improving robustness to ambiguous contact conditions, these methods introduce model-dependent corrections whose statistical properties are difficult to characterize.

Robust formulations incorporating scale-variant cost functions within invariant filtering~\cite{santana2024} attempt to mitigate slippage and outliers without abandoning geometric structure. More recent hybrid architectures combine invariant filtering with neural compensators~\cite{lee2025}, seeking to preserve group-consistent propagation while correcting for unmodeled effects. Although promising, these methods shift part of the estimator behavior outside the geometric framework, potentially affecting observability preservation.

Multi-sensor state estimation systems such as Pronto~\cite{camurri2020}, MUSE~\cite{nistico2025} and VILENS~\cite{wisth2023} enhance robustness by incorporating exteroceptive sensing. However, these architectures often rely on classical filtering or smoothing back-ends, and do not explicitly exploit invariant error dynamics to enforce structural consistency.

Invariant formulations have also been extended to smoothing. Yoon et al.~\cite{Yoon2024} proposed a Lie-group-based framework that preserves symmetry and retains consistency properties analogous to those of the IEKF, while improving linearization accuracy. This demonstrates that invariance extends to batch estimators, but at the cost of increased computation and latency. In contrast, this work focuses on recursive iterated invariant filtering suitable for real-time locomotion.

Iterated EKF schemes were originally introduced to reduce linearization bias by repeatedly re-linearizing the measurement model~\cite{bell1994iterated}. On manifolds, such iterations require consistent retraction operators and covariance transport~\cite{he2021}. On Lie groups, these operations can be simplified due to the underlying group structure~\cite{bourmaud2016}. The recently proposed IterIEKF~\cite{goffin2025} further improves this framework by re-linearizing invariant measurement residuals rather than arbitrary ones.

Although iterative refinement improves convergence in strongly nonlinear regimes, it remains unclear whether iteration preserves, restores, or potentially degrades statistical consistency when the underlying error representation does not respect system symmetries. In particular, the practical role of iteration in hybrid contact-aided legged locomotion has not been systematically evaluated.

In summary, prior work establishes that symmetry preservation improves consistency and that iterative updates reduce linearization bias. However, practical legged systems deviate from ideal group-affine assumptions: contact switching, soft-terrain interactions~\cite{fahmi2021}, and velocity-based updates introduce departures from strict invariance. Robust invariant formulations~\cite{santana2024} and structure-preserving inertial filtering approaches (e.g.,~\cite{fink2020iros}) emphasize the importance of maintaining geometric coherence under such disturbances.

The existing literature therefore clarifies individual mechanisms (symmetry preservation, contact modeling, and iterative refinement) but does not resolve how they interact in hybrid legged locomotion. In particular, the effect of iterative invariant updates under partially violated invariance assumptions has not been systematically quantified.

\section{BACKGROUND}
This work considers two matrix Lie groups, \(\mathrm{SE}_2(3)\) introduced in \cite[p. 18]{barrau2015thesis} and \(\mathrm{SO}(3) \times \mathbb{R}^{6}\). These groups are diffeomorphic as manifolds but not isomorphic as groups. Consequently, they share the same topology while exhibiting different geometric structures.
Previous work has shown that incorporating the geometry of \(\mathrm{SE}_2(3)\) simplifies the IMU dynamics, yielding the so-called \emph{natural frame dynamics} \cite{barrau2022}. Although $\mathrm{SO}(3) \times \mathbb{R}^{6}$ and $\mathrm{SE}_2(3)$ are diffeomorphic, endowing the state space with the Lie group structure of $\mathrm{SE}_2(3)$ induces a different connection and parallelization. Under this geometry, the IMU dynamics acquire a more structured form, which leads to an improved estimation performance \cite[pg. 58]{brossard2020}.

Every matrix Lie Group $G$ has an associated Lie Algebra $\mathfrak{g}$, which is a vector space,   simpler than $G$.
Both $G$ and $\mathfrak{g}$ have the same dimension $n$. We define the isomorphisms between $\mathfrak{g}$ and $\mathbb{R}^{n}$ 
as follows:
\begin{equation}
\label{eq:}
\begin{aligned}
	\text{Hat}: \mathbb{R}^{n} \to \mathfrak{g}, 
	\mathbf{v} \mapsto \mathbf{v}^{\wedge} =  \sum_{i=1}^{n}v_{i}E_{i},
\end{aligned}
\end{equation}
\begin{equation}
\label{eq:}
\begin{aligned}
\text{Vee}: \mathfrak{g} \to \mathbb{R}^{n}, 
\mathbf{v}^{\wedge} \mapsto (\mathbf{v}^{\wedge})^{\vee} = \mathbf{v} = \sum_{i=1}^{n} v_{i}\mathbf{e}_{i},
\end{aligned}
\end{equation}
where $\{E_i\}$ is a basis for $\mathfrak{g}$ and $\{ \mathbf{e}_{i}\}$ is a basis for $\mathbb{R}^{n}$. In this work, we extensively use the $\mathbb{R}^{n}$ 
representation of $\mathfrak{g}$. Furthermore, $\mathfrak{g}$ and $G$ are related by the following local homeomorphism \cite[p. 179]{stillwell2008}:
\begin{equation}
\label{eq:exp_map}
\begin{aligned}
	\exp_G: \mathfrak{g} \to G, 
	\mathbf{v}^{\wedge} \mapsto  \exp_G(\mathbf{v}^{\wedge}) = \sum_{i=0}^{\infty} \frac{{\mathbf{v}^{\wedge}}^{i}}{i!}
\end{aligned},
\end{equation}
with its inverse locally defined as $\log_G: G \to \mathfrak{g}$. 
We also define $\mathrm{Exp}_G(\mathbf{v}) := \exp_G(\mathbf{v}^{\wedge})$ and
$\mathrm{Log}_G := \log_G(\mathcal{X})^{\vee}$ for all $ \mathbf{v} \in \mathbb{R}^n, \mathcal{X} \in G$.
Sometimes the subscript $G$ is omitted when the Lie Group is clear from the context.
The exponential and logarithm maps for the Lie groups considered in this work are defined in Appendix \ref{appendix:appendix_formulas}.
Following \cite{sola2018}, we also overload the following operations:
\begin{equation}
\label{eq:main_operators}
\begin{aligned}
	\text{left-}\oplus: G \times \mathbb{R}^{n} \to  G,\mathcal{X} \oplus \mathbf{u} &= \mathcal{X} \mathrm{Exp}(\mathbf{u}), \\
	\text{right-}\oplus: \mathbb{R}^{n} \times G \to G, \mathbf{u} \oplus \mathcal{X} &= \mathrm{Exp}(\mathbf{u})\mathcal{X}. \\
\end{aligned}
\end{equation}
A common approximation of the exponential map for the case $\| \mathbf{u}\| \approx 0$ is using the right-Jacobian, $\mathcal{J}_{r,G}(\cdot)$, which allows us to write \cite{sola2018}:
\begin{equation}
\label{eq:right_jacobian}
\mathrm{Exp}_G(\mathbf{a} + \mathbf{u}) \approx \mathrm{Exp}_G(\mathbf{a}) \oplus \mathcal{J}_{r,G}(\mathbf{a}) \mathbf{u}.
\end{equation}

Let $\mathcal{T}_\mathcal{X}$ denote the tangent space at $\mathcal{X} \in G$, then $\mathcal{T}_\mathcal{X}$ 
is isomorphic to $\mathcal{T}_\mathcal{E}$, the tangent space at the identity element $\mathcal{E},$ by the adjoint map:
\begin{equation}
\label{eq:adjoint_map}
\text{Ad}_\mathcal{X} :\mathbb{R}^n \to \mathbb{R}^n,	\mathbf{u} \to \text{Ad}_\mathcal{X}\mathbf{u} =  (\mathcal{X} \mathbf{u}^{\wedge} \mathcal{X}^{-1})^{\vee}.	
\end{equation}

The following definition (Definition \ref{def:submanifold}) is fundamental for the IterIEKF formulation:
\begin{definition}[Observed Sets \cite{goffin2025}]
  \label{def:submanifold}
  Let $G$ be a Lie Group that is represented linearly by matrices of $\mathbb{R}^{d \times d}$. 
  The submanifold of $G$  denoted by $S_{\mathcal{X}^{-1}\mathbf{d} = \mathbf{y}}$  is defined as:
  \begin{equation}
      S_{\mathcal{X}^{-1} \mathbf{d} = \mathbf{y}} := \{\mathcal{X} \in G \mid \mathcal{X}^{-1} \mathbf{d} = \mathbf{y}\}, 
    \quad \mathbf{d}, \mathbf{y} \in  \mathbb{R}^{d}.
  \end{equation}
\end{definition}
\begin{proposition}
\label{prop:coset}
Observed sets are right-cosets of the stabilizer subgroup $G_\mathbf{d}$: 
\begin{equation}
  G_\mathbf{d} := \{\mathcal{X} \in G \mid \mathcal{X} \mathbf{d} = \mathbf{d}\}.
\end{equation}
\end{proposition}
\begin{proof}
	See Appendix \ref{appendix:appendix_proof}. 
\end{proof}
Proposition \ref{prop:coset} gives a geometric interpretation
for the observed sets defined in \cite{goffin2025}: they are cosets of a stabilizer subgroup of the Lie Group $G$.
 
For the state of the robot, we assume the Single Rigid Body assumption and treat the robot as a floating body. We estimate
the orientation of the robot $\mathbf{R} \in  \mathrm{SO}(3)$, the body velocity $\mathbf{v} \in  \mathbb{R}^{3}$ 
and the centroid position $\mathbf{p} \in  \mathbb{R}^{3}$. All of these quantities are in relation to an inertial frame. 
We embed the state into two Lie groups, the $\mathrm{SE}_2(3) = \mathrm{SO}(3) \ltimes \mathbb{R}^{6}$, 
introduced in \cite{barrau2016}
\begin{equation}
    \mathcal{X} := (\mathbf{R}, \mathbf{x}) := (\mathbf{R}, \mathbf{v}, \mathbf{p}) 
    := 
\begin{bmatrix}
\mathbf{R} & \mathbf{v} & \mathbf{p} \\
\mathbf{0}_{1,3} & 1 & 0 \\
\mathbf{0}_{1,3} & 0 & 1 
\end{bmatrix} \in  \mathbb{R}^{5 \times 5},
\end{equation}
with the inverse given as:
\begin{equation}
\label{eq:state_inverse}
\mathcal{X}^{-1} := (\mathbf{R}^{T}, -\mathbf{R}^{T} \mathbf{x}) = (\mathbf{R}^{T}, -\mathbf{R}^{T}\mathbf{v}, -\mathbf{R}^{T} \mathbf{p}).
\end{equation}
The other embedding is given by:
\begin{equation}
\label{so3_true_state}
\boldsymbol{x}_i := \{\mathbf{R}_i, \mathbf{v}_i, \mathbf{p}_i\} \in \mathrm{SO}(3) \times \mathbb{R}^{6}.
\end{equation}
Finally, we adopt a Bayesian approach, where the true (unknown) state at instant $t_i$, $\mathcal{X}_i$ 
and $\boldsymbol{x}_i$,
are random variables with known initial distribution \cite[p. 91]{shalom2002}. 
The initial distribution is defined in the Lie Algebra $\mathfrak{g}$ and projected onto $G$ with
the exponential map. For $\mathrm{SE}_2(3)$ we have: 
\begin{equation}
\begin{aligned}
    \boldsymbol{\xi}_{0} \sim \mathcal{N}(\mathbf{0}_{9,1}, \mathbf{P}_0), \\
    \bar{\mathcal{X}}_0 = \boldsymbol{\xi}_0 \oplus \hat{\mathcal{X}}_0, \bar{\mathcal{X}_0} \sim \mathcal{N}_R(\hat{\mathcal{X}}_0, \mathbf{P}_0),
\end{aligned}
\end{equation}
where $\boldsymbol{\xi}_0 \in \mathbb{R}^9$ is the initial error, $\bar{\mathcal{X}}_0$ is the initial nominal state, $\mathbf{P}_0 \in  \mathbb{R}^{9 \times 9}$ is a covariance matrix, $\hat{\mathcal{X}}_0 \in  \mathrm{SE}_2(3)$ is a known mean and 
$\mathcal{N}_R(\hat{\mathcal{X}}_0, \mathbf{P}_0)$ is the right projection of the tangent covariance 
known as concentrated Gaussian \cite{wang2006error, barrau2018}. Analogously, for $\mathrm{SO}(3) \times \mathbb{R}^{6}$ we have:
\begin{equation}
\begin{aligned}
    \delta \boldsymbol{x}_{0} \sim \mathcal{N}(\mathbf{0}_{9,1}, \mathbf{\Sigma}_0 \in  \mathbb{R}^{9 \times  9}), \\
    \bar{\boldsymbol{x}}_0 = \delta \boldsymbol{x}_0 \oplus \hat{\boldsymbol{x}}_0, 
    \bar{\boldsymbol{x}}_0 \sim \mathcal{N}_R(\hat{\boldsymbol{x}}_0, \mathbf{\Sigma}_0).
\end{aligned}
\end{equation}

In this work, we represent the state with respect to the world frame. Since the measurement is expressed in the base frame, this choice leads to a right-invariant error formulation~\cite{barrau2022}. Although the primary observable quantity, the base velocity, is measured in the body frame, we adopt this formulation to facilitate the integration of the proposed method into existing navigation pipelines. For completeness, in Appendix \ref{appendix:appendix_robocentric} 
we derive an analogous robocentric formulation, which leads to a left-invariant error.

\section{MOTIVATING EXAMPLE}
\label{sec:motivating_example}
\begin{figure*}[t]
		\centering
    \subfloat[\protect\label{fig:setup_slam}]{
        \includegraphics[scale=0.38]{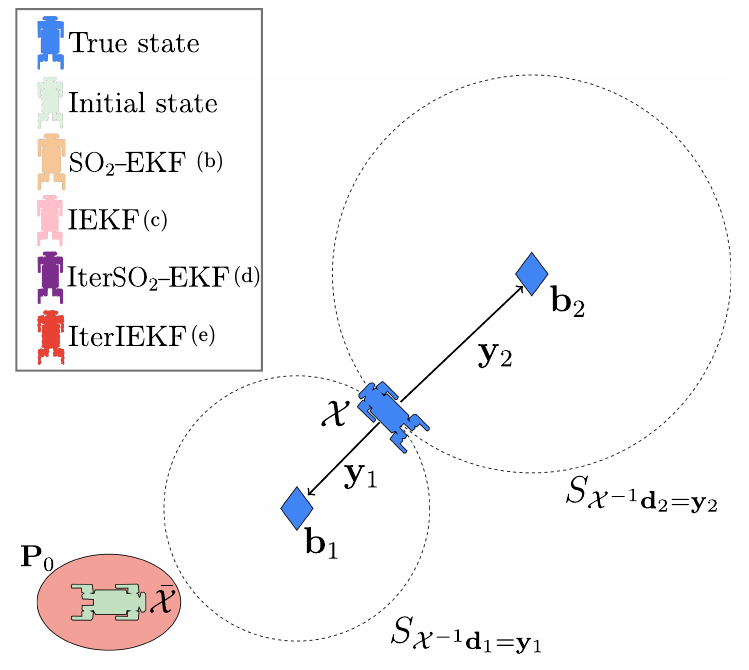}
    }
    \subfloat[\protect\label{fig:so2rn2ekf_slam}]{
       \includegraphics[scale=0.38]{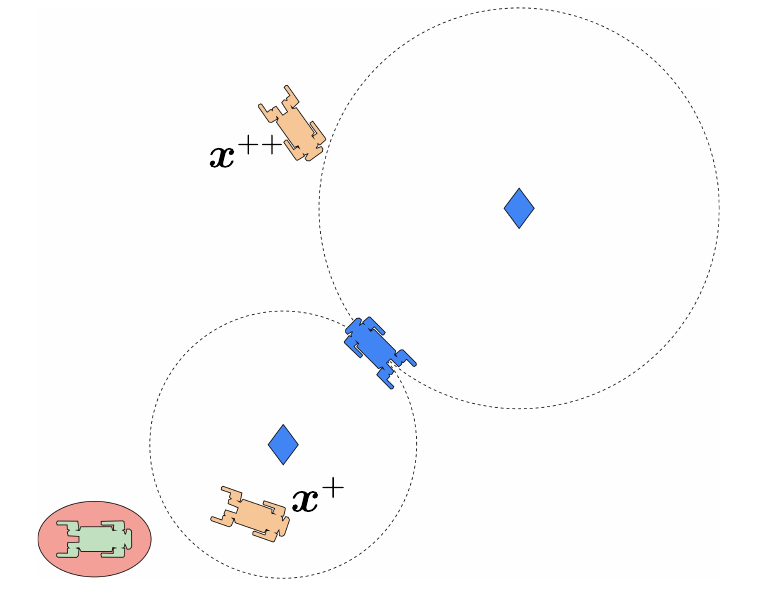}
    }
    \subfloat[\protect\label{fig:iekf_slam}]{
        \includegraphics[scale=0.38]{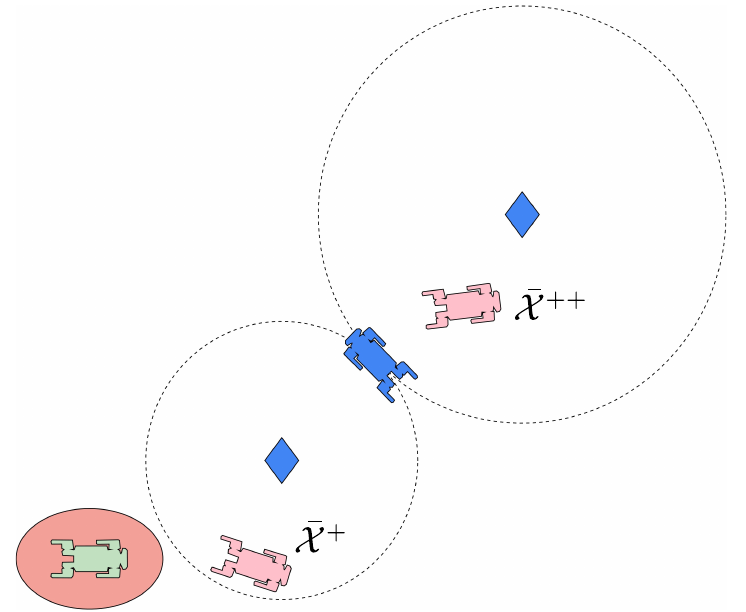}
    }
   \hfill
    \subfloat[\protect\label{fig:iterso2r2nekf}]{
        \includegraphics[scale=0.38]{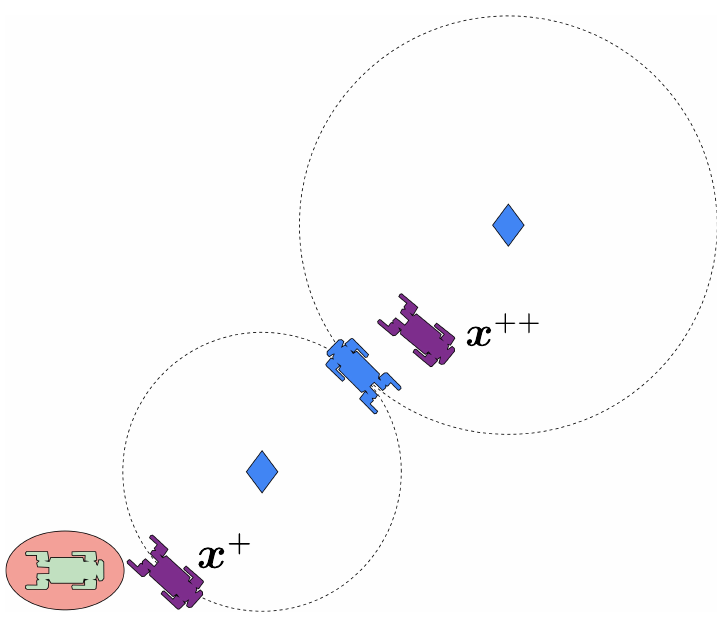}
    }
    \subfloat[\protect\label{fig:iiekf_slam}]{
        \includegraphics[scale=0.38]{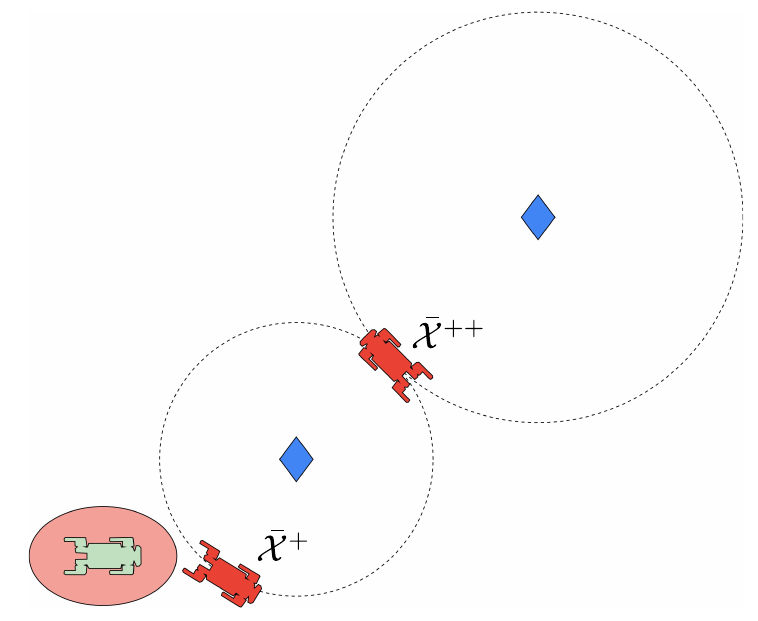}
    }
    \caption{a) Setup of a simple localization problem with two known landmarks with almost noise-free measurements. 
             b) Two consecutive updates with $\mathrm{SO}(2)$-EKF. No update lies in the observed sets.
             c) Two consecutive updates with IEKF using almost noise-free measurements. 
             d) Two consecutive updates with Iter$\mathrm{SO}(2)$-EKF. Only the first update lies in the observed sets. 
             The second update did not converge.
             e) Two consecutive updates with IterIEKF. Both updates lie in the observed sets, being the last update on their intersection.}
    \label{fig:slam}
\end{figure*}
In this section, we illustrate the differences between the four filtering approaches considered 
in this work: $\mathrm{SO}(3)$-EKF, IEKF, Iter$\mathrm{SO}(3)$-EKF and IterIEKF.
For this, we only consider the update step of a simple localization problem with known landmarks (Fig. \ref{fig:slam}).
The details of the formulation of each filter will be given in the next sections.
We consider a 2D quadruped robot moving in the plane, and the state given by the robot's orientation (yaw angle), velocity and position. We
parametrize the robot's state as an element of two, diffeomorphic by $\phi: \mathrm{SE}_2(2) \to  \mathrm{SO}_2 \times  \mathbb{R}^{4}$, Lie groups: 
$\mathrm{SE}_2(2)$ for IEKF and IterIEKF, and $\mathrm{SO}(2)\times \mathbb{R}^{4}$ for $\mathrm{SO}(2)$-EKF and Iter$\mathrm{SO}(2)$-EKF.
Let $\mathcal{X} =: (\mathbf{R}, \mathbf{v}, \mathbf{p}) \in \mathrm{SE}_2(2)$ - 
and equivalently
$\boldsymbol{x} := \phi(\mathcal{X}) \in \mathrm{SO}(2) \times \mathbb{R}^{4}$ 
- be the true state of the robot, 
$\{\mathbf{b_1}, \mathbf{b}_2\} \in  \mathbb{R}^{2}$ be two landmarks, and $\{\mathbf{y}_1, \mathbf{y}_2\} \in  \mathbb{R}^{2}$ 
denote two almost noise-free measurements of $\mathcal{X}$ with relation to the landmarks:
\begin{equation}
\begin{aligned}
\mathbf{y}_1 := 
\begin{bmatrix}
\mathbf{R}^{T} (\mathbf{b}_1 - \mathbf{p}) \\
0 \\
1
\end{bmatrix} = \\
\begin{bmatrix}
\mathbf{R}^{T} & -\mathbf{R}^{T}\mathbf{v} & -\mathbf{R}^{T} \mathbf{p} \\
\mathbf{0}_{2,1}  &  1  & 0 \\
\mathbf{0}_{2,1} & 0 &1 
\end{bmatrix}
\begin{bmatrix}
\mathbf{b}_1 \\
0 \\
1
\end{bmatrix} := \mathcal{X}^{-1}\mathbf{d}_1,
\end{aligned}
\end{equation}
\begin{equation}
\begin{aligned}
\mathbf{y}_2 := \mathcal{X}^{-1} \mathbf{d}_2 = \mathcal{X}^{-1} 
\begin{bmatrix} \mathbf{b}_{2} & 0 & 1 \end{bmatrix} ^{T}
.\end{aligned}
\end{equation}
The measurement models are in invariant form for $\mathrm{SE}_2(2)$, as defined in \cite{barrau2016}.
The observed sets for these two measurements are circumferences centered in the landmarks as shown in Fig. \ref{fig:setup_slam}.
On the other hand, the measurement $\mathbf{y} = \mathbf{R}^{T}(\mathbf{b} - \mathbf{p})$ is not in invariant form
for $\mathrm{SO}(2) \times \mathbb{R}^{2 \times 2}$. To be invariant, the measurement model should be in the 
form $\mathcal{X} \mathbf{d}$ or $\mathcal{X}^{-1} \mathbf{d}$.

Let $\bar{\mathcal{X}} \in \mathrm{SE}_2(2) \sim \mathcal{N}_R(\mathbf{0}_9, \mathbf{P})$ - and equivalently  $\bar{\boldsymbol{x}} = \phi(\bar{\mathcal{X}})$ -
be the current noisy estimate of the state. We consider no uncertainty in
the velocity error, so that $\mathbf{P} = \text{diag}((e^{R})^2,0,0,(e^{x})^2,(e^{y})^2)$, where $e = (\mathcal{X} \bar{\mathcal{X}}^{-1})^{\vee}$. 
In Fig. \ref{fig:slam} we show the results of two consecutive update steps with IEKF, $\mathrm{SO}(2)$-EKF, Iter$\mathrm{SO}(2)$-EKF, and IterIEKF.
In Fig. \ref{fig:iekf_slam} and Fig. \ref{fig:so2rn2ekf_slam} we show the updates with IEKF and $\mathrm{SO}(2)$-EKF.
It can be seen that after the update step, there is no guarantee that the state will be in the observed sets
$S_{\mathcal{X}^{-1} \mathbf{d}_1 = \mathbf{y}_1}$ and $S_{\mathcal{X}^{-1} \mathbf{d}_2=\mathbf{y}_1}$.
In Fig. \ref{fig:iterso2r2nekf} we show the updates with Iter$\mathrm{SO}(2)$-EKF. In this case, the first update lands the state
inside the observed set, but the second update did not converge.
On the other hand, with IterIEKF in Fig. \ref{fig:iiekf_slam}, both updates land the state inside the observed set, so that $\mathcal{X}^{++}$ 
lies at the intersection of the two observed sets, which corresponds to the true state. This happens because after the update step, 
the IterIEKF constrains the search space to the first observed set, 
and then the optimization process can find the convergence easier than Iter$\mathrm{SO}(2)$-EKF.

We emphasize that the IterIEKF update admits a clear geometric interpretation.  
Let $\mathbf{P}_{\mathbf{R}}$ be a covariance matrix with dominant uncertainty in rotation and negligible uncertainty in velocity and position. Similarly, let $\mathbf{P}_{\mathbf{p}}$ represent dominant uncertainty in position, with negligible uncertainty in rotation and velocity.
Figure~\ref{fig:compare_covariance} illustrates the IterIEKF update in both cases. When rotational uncertainty dominates ($\mathbf{P}_{\mathbf{R}}$), the updated state exhibits a significant change in orientation while the position remains nearly unchanged. In contrast, when positional uncertainty dominates ($\mathbf{P}_{\mathbf{p}}$), the update results in an almost pure translation, since the rotation is nearly certain.
Let $\hat{\boldsymbol{\xi}}_{\mathbf{R}}$ (and analogously $\hat{\boldsymbol{\xi}}_{\mathbf{p}}$) denote the final error estimate associated with covariance $\mathbf{P}_{\mathbf{R}}$. Consider the one-parameter subgroup
$
H = \left\{ \mathrm{Exp}\!\left(s \hat{\boldsymbol{\xi}}_{\mathbf{R}}\right) \mid s \in \mathbb{R} \right\},
$
generated by $\hat{\boldsymbol{\xi}}_{\mathbf{R}}^{\wedge} \in \mathfrak{g}$~\cite[pg.~56]{hall2013}.
The updated state can be interpreted as the intersection between the right coset $H \bar{\mathcal{X}}$ and the measurement constraint set $S_{\mathcal{X}^{-1}\mathbf{d}=\mathbf{y}}$, which is by Proposition~\ref{prop:coset} the right coset $G_{\mathbf{d}}\mathcal{X}$. When both cosets form nonlinear submanifolds of $G$, computing this intersection becomes a nonlinear problem. On the other hand, under covariance $\mathbf{P}_{\mathbf{p}}$, the right coset $\mathrm{Exp}\!\left(s \hat{\boldsymbol{\xi}}_{\mathbf{p}}\right)\bar{\mathcal{X}}$ is linear. In this case, IterIEKF reduces to the standard IEKF. In essence, IterIEKF can be interpreted as searching for the intersection of cosets of subgroups of $G$.
\begin{figure}[h]
    \centering
    \includegraphics[scale=0.45]{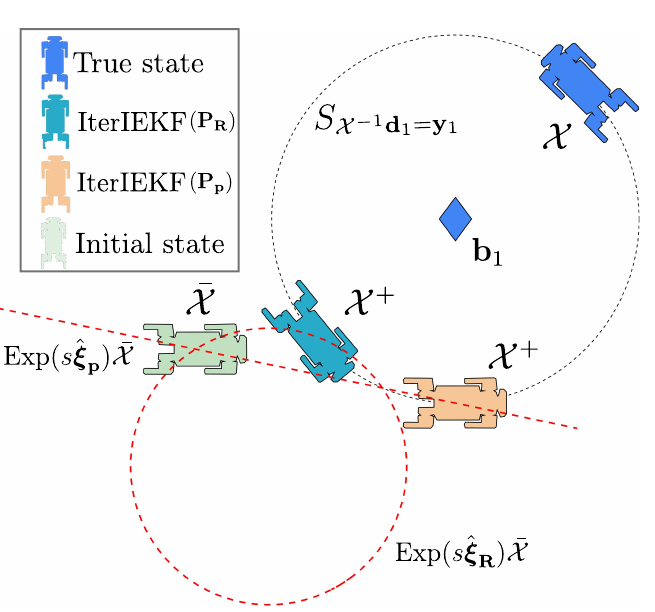}
    \caption{Comparison of update step with IterIEKF in two scenarios: first with a prevalent uncertainty in the rotational error ($\mathbf{P}_\mathbf{R}$) 
    and the second with a prevalent positional error ($\mathbf{P}_\mathbf{p}$). The dotted red curves are the right-cosets $\mathrm{Exp}(s \hat{\boldsymbol{\xi}}_{\mathbf{R}}) \bar{\mathcal{X}}$ and $\mathrm{Exp}(s \hat{\boldsymbol{\xi}}_{\mathbf{p}}) \bar{\mathcal{X}}$, $s\in \mathbb{R}$. }
    \label{fig:compare_covariance}
\end{figure}

\section{METHODOLOGY}
\subsection{Base velocity as measurement model}
\label{sec:base_velocity_as_measurement_model}

In this work, we restrict the measurement model to foot velocity during contact and exclude forward kinematics-based foot position measurements. Iterated filtering techniques are particularly sensitive to covariance mischaracterization. While the covariance of base velocity measurements can be empirically estimated using motion capture systems or GNSS in outdoor environments, the uncertainty associated with forward kinematics-based position measurements is more difficult to model accurately. It depends not only on encoder noise, but also on joint compliance, kinematic parameter errors, structural flexibility, and contact effects, which introduce state-dependent and potentially biased errors that are not easily captured by a fixed Gaussian covariance model. Since iterated filters rely heavily on accurate uncertainty characterization to ensure consistency and convergence, incorporating measurements with poorly modeled covariance may degrade performance. To avoid introducing such unmodeled uncertainties, we therefore limit our measurement model to stance foot velocity constraints. Instead of including foot contact positions in the state, in a way that is reminiscent of SLAM \cite[p. 3] {bloesch2012}, we consider foot detection as null velocity measurement (up to noise), along the lines of the velocity-update (ZUPT) of pedestrian navigation \cite{foxlin2005pedestrian}. 

The feet velocity as a measurement was first proposed for quadruped robots in \cite{bloesch2013}. It can be obtained as follows:
\begin{equation}
\label{eq:foot_velocity}
\tilde{\mathbf{v}}_{f,B}^{i} = \boldsymbol{\omega}^{\times}_{i,B} \mathbf{p}_{f,B} + J_\mathcal{K} \dot{\mathbf{q}}_{i,B},
\end{equation}
where the subscript $B$ denotes the base frame, $\mathbf{p}_f = \mathcal{K}(\mathbf{q})$ is the foot contact position obtained by forward kinematics and $J_\mathcal{K}$
is the linear Jacobian of $\mathcal{K}$. Moreover, under the static contact assumption, we have the following measurement for
the base velocity:
\begin{equation}
\label{eq:velocity_base_frame}
\tilde{\mathbf{v}}_{B,i} := -\frac{1}{|S|}\sum_{i \in S} \tilde{\mathbf{v}}_{f,B}^{i} = \mathbf{v}_B + \mathbf{w}_{f,i},
\end{equation}
where $S$ is a set with the IDs of all legs in contact and $\mathbf{w}_{f,i} \sim \mathcal{N}(\mathbf{0}_{3,1}, \mathbf{Q}_f)$. Note that the covariance matrix $\mathbf{Q}_f$ depends on the noise of three distinct sources: angular velocity measured by the IMU, 
kinematics calibration and joint encoders. Due to the complexity of its evaluation we follow the convention of 
previous works \cite{bloesch2013, i2ekflo} and assume $\mathbf{Q}_f$ as a user-prescribed parameter. To understand the structure of this covariance matrix, we 
performed a simulation on MuJoCo \cite{mujoco} with the QuadrupedPyMPC library \cite{turrisi2024}. We let the robot perform an eight-shape trajectory for $8$ seconds on flat terrain with trotting gait. We compared the ground truth base velocity 
with the one obtained from \eqref{eq:foot_velocity} and computed the covariance matrix of the error. The results are shown in Fig. \ref{fig:cov_foot_velocity}. It can be seen that the covariance matrix is not isotropic, with the largest uncertainty in the $z$ direction. The second largest covariance is in the 
forward direction.
In Fig. \ref{fig:mujoco_contact} we can see a comparison between the ground truth base velocity and the one obtained from~\eqref{eq:foot_velocity}. We observed that when the two feet are in the swing phase,~\eqref{eq:foot_velocity} is a good estimate of the base velocity (Fig. \ref{fig:velocity_swing}). In contrast,  during the touchdown event, when the feet transition from swing to stance phase, the estimation error increases, predominantly in the $z$ direction (Fig. \ref{fig:velocity_contact}).
\begin{figure}
    \centering
    \includegraphics[width=0.47\textwidth]{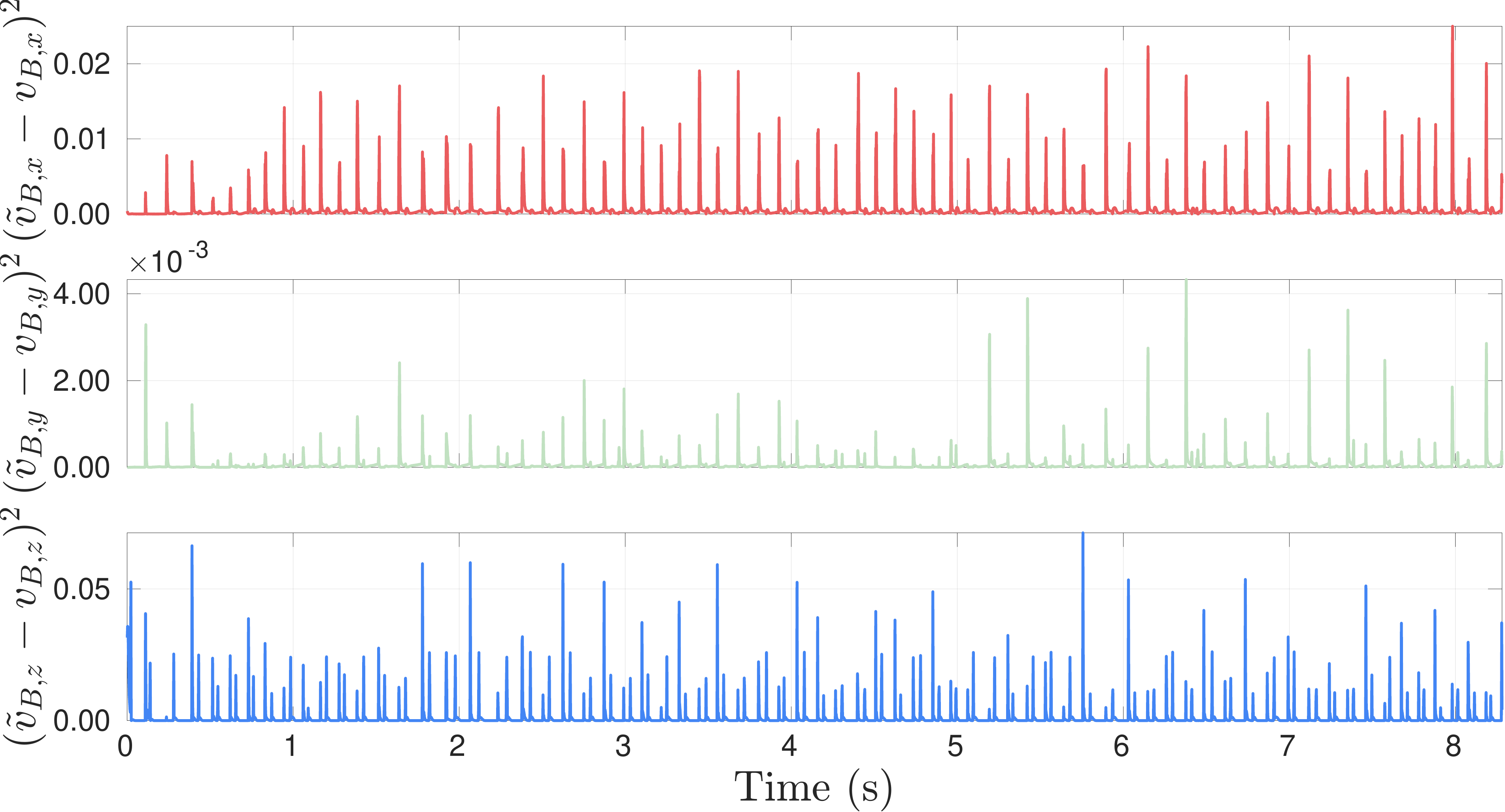}
    \caption{Covariance matrix of the base velocity obtained from foot velocities in contact on a virtual environment. The covariance matrix is not isotropic, with the largest uncertainty in the $z$ direction. The mean for each direction was given by $\mathbf{Q}_f = \text{diag}(8.02\times10^{-4},6.28 \times10^{-5},1.25 \times 10^{-3})$. The video of the simulation can be watched \href{https://youtu.be/CS5u8OvGXVM}{here}.}
    \label{fig:cov_foot_velocity}
\end{figure}
\begin{figure}[!h]
    \centering
    \subfloat[\protect\label{fig:velocity_swing}]{
        \includegraphics[width=0.2\textwidth,height=0.17\textwidth]{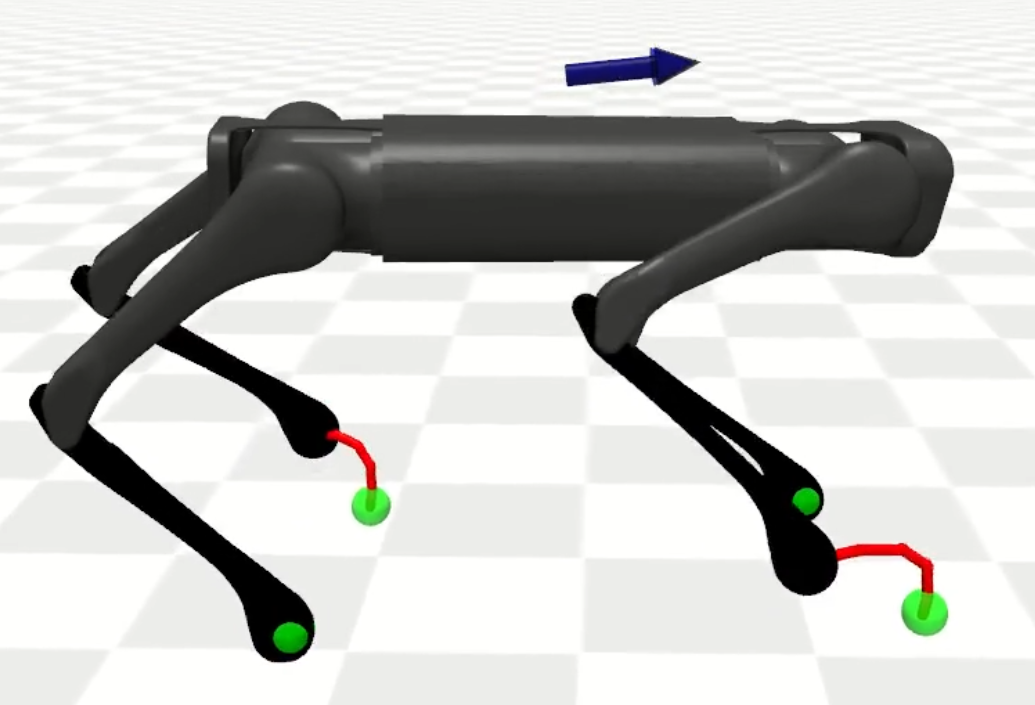}
    }
    \subfloat[\protect\label{fig:velocity_contact}]{
        \includegraphics[width=0.2\textwidth,height=0.17\textwidth]{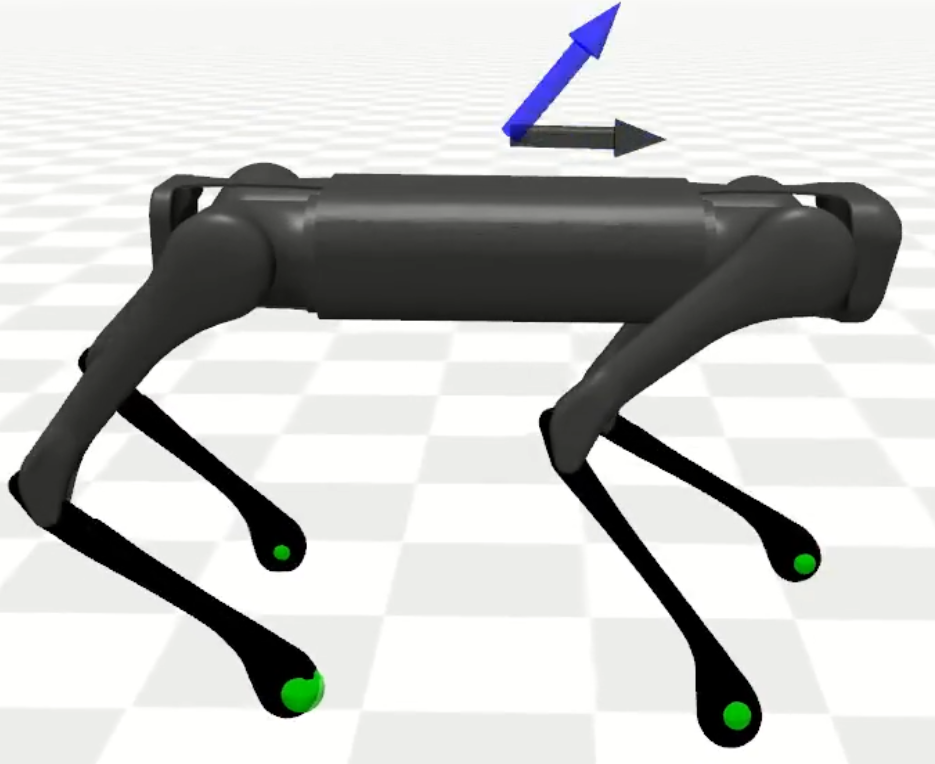}
    }
    \caption{Comparison between true base velocity in the base frame (black arrow) with the one obtained from feet velocities in contact (blue arrow).
        a) During the swing phase, the foot velocity is a good estimate of the base velocity.
    b) During the touchdown event, there is an increase in the error in the $z$ and forward directions.}
    \label{fig:mujoco_contact}
\end{figure}

Finally, we recall that we need to transform the measured velocity from the base frame 
to the IMU frame, since all filtering operations are performed there. 
Let $\{{}^{I}\mathbf{R}_{B}, {}^{B}\mathbf{r}_{BI}\}$, ${{}^{I}\mathbf{R}_{B} \in \mathrm{SO}(3)}$ and ${}^{B}\mathbf{r}_{BI} \in \mathbb{R}^{3}$, be the coordinate transformation from the base frame to the IMU frame, then we can write: 
\begin{equation}
    \tilde{\mathbf{v}}_{I,i}
=
{}^{I}\mathbf{R}_{B}
\left(
\tilde{\mathbf{v}}_{B,i}
+
\tilde{\boldsymbol{\omega}}_{B,i}
\times
{}^{B}\mathbf{r}_{BI}
\right) \approx {}^{I}\mathbf{R}_{B}\tilde{\mathbf{v}}_{B,i},
\end{equation}
where we considered the quantity $\tilde{\boldsymbol{\omega}}_{B,i}
\times
{}^{B}\mathbf{r}_{BI}$ negligible in relation to $\tilde{\mathbf{v}}_{B,i}$, because for quadruped robots the IMU frame is always close to the body frame. The new covariance will be transformed according $\mathbf{Q}_f :={}^I\mathbf{R}_{B}\mathbf{Q}_f {}^I\mathbf{R}_{B}^{T}$.

\subsection{IEKF and IterIEKF}
Following previous works \cite{bloesch2012,Hartley2019ContactaidedIE}, we adopt the Single Rigid Body assumption and approximate the robot as a floating box moving in $\mathbb{R}^3$.
Then, the dynamics of the box are governed by the linear acceleration $\tilde{\mathbf{a}}_I$ and angular 
velocity $\tilde{\boldsymbol{\omega}}_I$ measured by the IMU:
\begin{equation}
\begin{aligned}
	\tilde{\boldsymbol{\omega}}_{I,i} = \boldsymbol{\omega}_{I,i} + \mathbf{w}_{g,i}, \mathbf{w}_{g,i} 
  \sim \mathcal{N}(\mathbf{0}_{3,1}, \mathbf{Q}_g \in  \mathbb{R}^{3 \times  3}), \\
	\tilde{\mathbf{a}}_{I,i} = \mathbf{a}_{I,i} + \mathbf{w}_{a,i}, \mathbf{w}_{a,i}  
  \sim \mathcal{N}(\mathbf{0}_{3,1}, \mathbf{Q}_a \in  \mathbb{R}^{3 \times  3})
,\end{aligned}
\end{equation}
where the subscript $I$ denotes the IMU frame. We assume that the measurements as piecewise constant in $[t_i, t_i + \Delta t]$ and corrupted with white Gaussian noise. Notice that we assume, for simplicity, negligible IMU biases in this work.
We define a new variable $\mathbf{w}_i$ to accommodate all input noise:
\begin{equation}
\mathbf{w}_i := 
\begin{bmatrix} \mathbf{w}_{g,i} \\ \mathbf{w}_{a,i} \end{bmatrix} 
\sim \mathcal{N}(\mathbf{0}_{6,1}, \mathbf{Q}_i), \mathbf{Q}_{i} := 
\begin{bmatrix}
\mathbf{Q}_{g} & \mathbf{0}_{3} \\
\mathbf{0}_{3} & \mathbf{Q}_{a} \\
\end{bmatrix}
        .
\end{equation}

Santana et al. \cite{santana2024} showed that the prediction step can be cast as the natural frame dynamics defined
in \cite{barrau2022}:
\begin{equation}
\label{eq:prediction_step}
\begin{aligned}
    \overline{\mathcal{X}}_{i|i-1} := \mathcal{F}_{i-1}(\overline{\mathcal{X}}_{i-1}, 
\tilde{\boldsymbol{\omega}}_{I,i-1}, \tilde{\mathbf{a}}_{I,i-1}, \mathbf{w}_{i-1}, \Delta t) \approx \\
\begin{bmatrix}
\bar{\mathbf{R}}_{i-1}\mathrm{Exp}({\boldsymbol{\omega}}_{I,i-1}\Delta t) \mathrm{Exp}(\mathbf{G}_{i-1}^{1}\mathbf{w}_{g,i-1}) \\
\bar{\mathbf{v}}_{i-1} + (\bar{\mathbf{R}}_{i-1} \tilde{\mathbf{a}}_{I,i-1} + \mathbf{g}) \Delta t \\
\bar{\mathbf{p}}_{i-1} + \bar{\mathbf{v}}_{i-1}\Delta t +  (\bar{\mathbf{R}}_{i-1}\tilde{\mathbf{a}}_{I,i-1} + \mathbf{g})\frac{\Delta t^{2}}{2} \\
\end{bmatrix} =  \\
\begin{bmatrix}
\bar{\mathbf{R}}_{i-1} \mathrm{Exp}(\boldsymbol{\omega}_{I,i+1}\Delta t) \exp(\mathbf{G}_{i-1}^{1}\mathbf{w}_{g,i-1}) \\
\mathbf{F}_{i-1}\mathbf{x}_{i-1} + \mathbf{d}_{i-1} + \bar{\mathbf{R}}_{i-1} * \mathbf{s}_{i-1} + \bar{\mathbf{R}}_{i-1} *  \mathbf{G}^{2}_{i-1}\mathbf{w}_{a,i-1} \end{bmatrix} = \\
\mathcal{W}_{i-1}\Phi_{i-1}(\bar{\mathcal{X}}_{i-1}){\bar{\mathcal{Y}}}_{i-1}(\mathbf{w}_{i-1}),
\end{aligned}
\end{equation}
where $*$ denotes an element-wise multiplication: $\mathbf{R} * \mathbf{x} := [\mathbf{R} \mathbf{v}, \mathbf{R} \mathbf{p}]^{T}$, $\bar{(\cdot)}$ denotes the nominal state, $\mathbf{g} = [0,0,-9.81]$ is the gravity vector and
\begin{equation}
\begin{aligned}
&\mathcal{W}_{i-1} = (\mathbf{I}_3, \mathbf{d}_{i-1}), \\
&\bar{\mathcal{Y}}_{i-1} = ( \mathrm{Exp}_{\mathrm{SO}(3)}({\boldsymbol{\omega}}_{I,i-1}\Delta t) \mathrm{Exp}_{\mathrm{SO}(3)}(\mathbf{G}_{i-1}^{1} \mathbf{w}_{g,i-1}),\\ &\mathbf{s}_{i-1} + \mathbf{G}_{i-1}^{2} \mathbf{w}_{a,i-1} )  , \\
&\Phi_{i-1}(\overline{\mathcal{X}}_{i-1})= (\bar{\mathbf{R}}_{i-1}, \mathbf{F}_{i-1} \bar{\mathbf{x}}_{i-1}), \\
&\mathbf{F}_{i-1} = \begin{bmatrix}
\mathbf{I}_3 & \mathbf{0}_{3} \\
\Delta t \mathbf{I}_3 & \mathbf{I}_3  \\
\end{bmatrix}, \\
&\mathbf{d}_{i-1} = \begin{bmatrix}
\Delta t \mathbf{g} \\
\frac{\Delta t^{2}}{2} \mathbf{g}
\end{bmatrix},
\mathbf{s}_{i-1} = 
\begin{bmatrix}
 {\mathbf{a}}_{I,i-1}\Delta t \\
 {\mathbf{a}}_{I,i-1} \frac{\Delta t^{2}}{2}
\end{bmatrix}, \\
&\mathbf{G}_{i-1}^{1} = -\mathcal{J}_{r,\mathrm{SO}(3)}(\tilde{\boldsymbol{\omega}}_{I,i-1}\Delta t)  \Delta t, 
\mathbf{G}_{i-1}^{2} =  
\begin{bmatrix}
\mathbf{I}_3 \Delta t \\
\mathbf{I}_3\frac{\Delta t^2}{2} 
\end{bmatrix}.
\end{aligned}
\end{equation}
Notice that $\Phi_i:\mathrm{SE_2}(3)\to \mathrm{SE_2}(3)$ is an automorphism which ensures
that the covariance propagation of the (right) invariant error without process noise is exact and autonomous \cite{barrau2016, barrau2022} (see also Theorem 18 in \cite{barrau2019}).
In the presence of process noise, the approximate propagation was obtained in \cite{brossard2022}\cite{santana2024} as follows:
\begin{equation}
\label{eq:error_propagation}
\boldsymbol{\xi}_{i|i-1} \approx \mathbf{A}_{i-1}\boldsymbol{\xi}_{i-1} + \mathbf{B}_{i-1} \mathbf{w}_{i-1},
\end{equation}
where
\begin{equation}
\label{eq:B_noise}
\begin{aligned}
    \mathbf{A}_{i-1}:= \text{Ad}_{\mathcal{W}_{i-1}} \mathbf{M}_{i-1} ,
\mathbf{B}_{i-1} :=  \text{Ad}_{\bar{\mathcal{X}}_{i|i-1}} \mathbf{G}_{i-1},\\
\mathbf{M}_{i-1} =\begin{bmatrix}
\mathbf{I}_{3} & \mathbf{0}_{3,9} \\
\mathbf{0}_{9,3} & \mathbf{F}_{i-1} \\
\end{bmatrix},
\mathbf{G}_{i-1} = \begin{bmatrix}
\mathbf{G}_{i-1}^{1} & \mathbf{0}_{3} \\
\mathbf{0}_{6,3}  & \mathbf{G}_{i-1}^{\mathbf{*}2}
\end{bmatrix}, \\
\mathbf{G}_{i-1}^{\mathbf{*}2} =  
-\begin{bmatrix} 
\mathrm{Exp}_{\mathrm{SO_3}}(-\tilde{\boldsymbol{\omega}}_{I,i-1} \Delta t ) \Delta t \\
\mathrm{Exp}_{\mathrm{SO_3}}(-\tilde{\boldsymbol{\omega}}_{I,i-1} \Delta t ) \frac{\Delta t^{2}}{2}
\end{bmatrix}.
\end{aligned}
\end{equation}
The covariance propagation is given by:
\begin{equation}
    \mathbf{P}_{i|i-1} \approx \mathbf{A}_{i-1} \mathbf{P}_{i-1} \mathbf{A}_{i-1}^{T} 
    + \mathbf{B}_{i-1} \mathbf{Q}_{i-1} \mathbf{B}_{i-1}^{T}.
\end{equation}
In the following, we outline the modifications introduced in the measurement update of the IterIEKF proposed in \cite{goffin2025} with respect to the classical IEKF update step.

If we embed the measurement into $\mathbb{R}^{5}$ then, it can be obtained as a noisy action of the true state $\mathcal{X}_i$:
\begin{equation}
\begin{aligned}
\mathbf{y}_i := \begin{bmatrix} \tilde{\mathbf{v}}_{I,i} \\ -1 \\ 0 \end{bmatrix} = 
\begin{bmatrix} \mathbf{R}^{T}_i \mathbf{v}_i + \mathbf{w}_{f,i} \\ -1 \\ 0 \end{bmatrix} = \\
\begin{bmatrix} \mathbf{R}^{T}_i & -\mathbf{R}^{T}_i \mathbf{v}_i & -\mathbf{R}^{T}_i \mathbf{p}_i  \\
\mathbf{0}_{3,1} & 1 & 0 \\
0 & 0 & 1 \end{bmatrix}\begin{bmatrix}
\mathbf{0}_{3,1} \\
-1 \\
0
\end{bmatrix} + \begin{bmatrix}
\mathbf{w}_{f,i} \\
0 \\
0
\end{bmatrix} = \\
\mathcal{X}^{-1}_{i} \mathbf{d} + \mathbf{n}_i,
\end{aligned}
\end{equation}
where
\begin{equation}
\begin{aligned}
\mathbf{w}_{f,i} \sim \mathcal{N}(\mathbf{0}_{3,1}, \mathbf{Q}_f), \\
\mathbf{n}_i \sim \mathcal{N}\Big(\mathbf{0}_{5,1}, \mathbf{N}_{i} := 
    \begin{bmatrix}
        \mathbf{I}_{3} \\
        \mathbf{0}_{2,3}
    \end{bmatrix}
    \mathbf{Q}_{f}
    \begin{bmatrix}
        \mathbf{I}_{3} & \mathbf{0}_{3,2} 
\end{bmatrix} \Big),
\end{aligned}
\end{equation}
is white Gaussian noise.

For the right-invariant error, the innovation is defined as~\cite{barrau2022}:
\begin{equation}
    \label{eq:innov}
    \begin{aligned}
        \mathbf{z}_{i} := \bar{\mathcal{X}}_{i|i-1} \mathbf{y}_{i} - \mathbf{d} = \\
        \bar{\mathcal{X}}_{i|i-1} \mathcal{X}^{-1}_{i} \mathbf{d} - \mathbf{d} + \bar{\mathcal{X}}_{i|i-1} \mathbf{n} = \\
        \delta\mathcal{X}_{i}^{-1} \mathbf{d} - \mathbf{d} + \bar{\mathcal{X}}_{i|i-1} \mathbf{n} = \\
        \mathrm{Exp}(\boldsymbol{\xi}_{i})^{-1} \mathbf{d} - \mathbf{d} + \bar{\mathcal{X}}_{i|i-1} \mathbf{n},
    \end{aligned}
\end{equation}
where $\delta\mathcal{X}_{i} :=  \mathcal{X}_{i} \bar{\mathcal{X}}_{i|i-1}^{-1} :=\mathrm{Exp}(\boldsymbol{\xi}_{i})$ represents the current right-invariant error in the Lie Group. In the IEKF, the error $\boldsymbol{\xi}_{i}$ is assumed to be small, and the following
approximation is applied:
\begin{equation}
    \label{eq:iekf}
    \begin{aligned}
        \delta \mathcal{X}_{i}^{-1} \mathbf{d} - \mathbf{d} =  
        \mathrm{Exp}(\boldsymbol{\xi}_{i})^{-1}\mathbf{d} - \mathbf{d} \approx \\
        (\mathbf{I}_{9} - \boldsymbol{\xi}_{i}^{\wedge}) \mathbf{d} - \mathbf{d} =  \\
        -\boldsymbol{\xi}_{i}^{\wedge}\mathbf{d} = \\
        -\begin{bmatrix}
            (\boldsymbol{\xi}_{i}^{R})^{\wedge} (\mathbf{0}_{3,1}) + \boldsymbol{\xi}_{i}^{\mathbf{v}} (-1)  + \boldsymbol{\xi}_{i}^{\mathbf{p}}(0) \\
            \mathbf{0}_{1,3}(\mathbf{0}_{3,1}) + 0(-1) + 0(0) \\
            \mathbf{0}_{1,3}(\mathbf{0}_{3,1}) + 0(-1) + 0(0)
        \end{bmatrix} = \\
        \begin{bmatrix}
            \mathbf{0}_{3,1}^{\wedge}  & \mathbf{I}_{3}  & \mathbf{0}_{3} \\
            \mathbf{0}_{1,3} & \mathbf{0}_{1,3} & \mathbf{0}_{1,3} \\
            \mathbf{0}_{1,3} & \mathbf{0}_{1,3} &  \mathbf{0}_{1,3}
        \end{bmatrix}
        \begin{bmatrix}
            \boldsymbol{\xi}_{i}^{R} \\
            \boldsymbol{\xi}_{i}^{\mathbf{v}} \\
            \boldsymbol{\xi}_{i}^{\mathbf{p}}
        \end{bmatrix} = \mathbf{H} \boldsymbol{\xi}_{i},
    \end{aligned} 
\end{equation}
where $\mathrm{Exp}(\boldsymbol{\xi}_i) \approx (\mathbf{I}_9 + \boldsymbol{\xi}^{\wedge}_i)$. Then, it follows that:
\begin{equation}
    \mathbf{z}_{i} \approx \mathbf{H}  \boldsymbol{\xi}_{i} + \hat{\mathbf{n}}_i, 
\end{equation}
where $\hat{\mathbf{n}}_i \sim \mathcal{N}(\mathbf{0}_{5,1},\hat{\mathbf{N}}_i := \bar{\mathcal{X}}_{i|i-1}\mathbf{N}_i \bar{\mathcal{X}}_{i|i-1}^{T})$
and the estimate of the error $\hat{\boldsymbol{\xi}}_{i}$ is obtained via EKF equations at the tangent space $\mathbb{R}^{5}$.
Observing that the last two lines of $\mathbf{z}_{i}$ are zeros, we can reduce its dimension to $3$ as follows:
\begin{equation}
    \mathbf{\tilde z}_{i}:=[\mathbf{z}_{i}]_{1:3},
\end{equation}
and then we have:
\begin{equation}
    \mathbf{\tilde z}_{i} \approx \mathbf{\tilde H}_i  \boldsymbol{\xi}_{i} + \hat{\tilde{\mathbf{n}}}_i, 
\end{equation}
where
\begin{equation}
    \label{eq:tilde_H}
    \mathbf{\tilde H}_{i}:=
    \begin{bmatrix}
        \mathbf{0}_{3} & \mathbf{I}_{3} & \mathbf{0}_{3}
    \end{bmatrix} , 
\end{equation}
and $\hat{\tilde{\mathbf{n}}}_i \sim \mathcal{N}(\mathbf{0}_{3,1}, \hat{\tilde{\mathbf{N}}}_i := \bar{\mathbf{R}}_{i|i-1} \mathbf{Q}_f  \bar{\mathbf{R}}_{i|i-1}^{T})$.

In this work, we follow \cite{goffin2025} and define the estimate $\hat{\boldsymbol{\xi}}_{i}$ as 
the mode of the posterior distribution:
\begin{equation}
    \label{eq:minimization}
    \begin{aligned}
        \hat{\boldsymbol{\xi}}_{i} := 
        \argmax_{\boldsymbol{\xi}_i}  p(\boldsymbol{\xi}_i| \mathbf{y}_{i}) =  
        \argmax_{\boldsymbol{\xi}_i}  p(\boldsymbol{\xi}_i| \mathbf{z}_{i}) = \\ 
        \argmax_{\boldsymbol{\xi}_i} p(\boldsymbol{\xi}_i) p(\mathbf{z}_{i}| \boldsymbol{\xi}_i) =  \\
        \argmax_{\boldsymbol{\xi}_i} \mathcal{N}(\boldsymbol{\xi}_i; \mathbf{0}_{9,1}, \mathbf{P}_{i|i-1}) 
        \mathcal{N}(\mathbf{z}_{i}; \mathrm{Exp}(\boldsymbol{\xi}_i)^{-1} \mathbf{d} - \mathbf{d}, \hat{\mathbf{N}}_i) = \\
        \argmin_{\boldsymbol{\xi}_i} 
        \frac{1}{2}\| \boldsymbol{\xi}_i \|_{\mathbf{P}_{i|i-1}}^2 + 
        \frac{1}{2}\| \mathbf{z}_{i} - \mathrm{Exp}(\boldsymbol{\xi}_i)^{-1} \mathbf{d} + \mathbf{d}   \|_{\hat{\mathbf{N}}_{i}}^2,
    \end{aligned}
\end{equation}
where, in the last equality, the negative of its logarithm was applied. The equality
$p(\boldsymbol{\xi}_i \mid \mathbf{y}_i) = p(\boldsymbol{\xi}_i \mid \mathbf{z}_i)$
holds because $\mathbf{y}_i\mapsto \mathbf{z}_i$ is a bijection (see~\eqref{eq:innov}). The justification for the form of the conditional distribution $p(\mathbf{z}_i \mid \boldsymbol{\xi}_i)$ can be found in~\cite[p. 45]{shalom2002}. The mean of the prior error is taken as zero after each update~\cite{santana2024}, \cite[p. 210]{farrell2008}. Finally, $\| \mathbf{x} \|_{\mathbf{A}}^2 = \mathbf{x}^{T} \mathbf{A}^{-1} \mathbf{x}$ denotes the Mahalanobis distance. 

In~\cite{goffin2025}, two important results were demonstrated. First,~\eqref{eq:minimization} can be solved using the Gauss-Newton method, provided a linearization of $\mathrm{Exp}(\boldsymbol{\xi}_{i})^{-1}\mathbf{d}$ is available. Second, when~\eqref{eq:minimization} is used as the update step with highly precise measurements, that is, $\hat{\mathbf{N}}_{i}\to 0$, the resulting filter exhibits properties similar to those of the linear Kalman Filter that both the standard EKF and standard IEKF lack.

To obtain the linearization, we first recall \cite{sola2018}
\begin{equation}
    \label{eq:inv}
    \mathrm{Exp}(\boldsymbol{\xi}_i + \boldsymbol{\delta})^{-1} = \mathrm{Exp}(-\boldsymbol{\xi}_i - \boldsymbol{\delta}),
\end{equation}
then, following Lemma 2 of \cite{goffin2025}:
\begin{equation}
    \begin{aligned}
        \mathrm{Exp}(\boldsymbol{\xi}_i + \boldsymbol{\delta})^{-1} \mathbf{d} - \mathbf{d} = \mathrm{Exp}(-\boldsymbol{\xi}_i - \boldsymbol{\delta})\mathbf{d} - \mathbf{d} 
        \stackrel{\eqref{eq:right_jacobian}}{\approx} \\
        \mathrm{Exp}(-\boldsymbol{\xi}_i)\mathrm{Exp}(-\mathcal{J}_{r}(-\boldsymbol{\xi}_i) \boldsymbol{\delta} ) \mathbf{d} - \mathbf{d}
        \stackrel{\eqref{eq:inv}}{=}  \\
        \mathrm{Exp}(-\boldsymbol{\xi}_i)\mathrm{Exp}(\mathcal{J}_{r}(-\boldsymbol{\xi}_i)\boldsymbol{\delta}) ^{-1} \mathbf{d}  - \mathbf{d}
        \stackrel{\eqref{eq:iekf}}{\approx}  \\
        \mathrm{Exp}(-\boldsymbol{\xi}_i) (\mathbf{d} + \mathbf{H}_{i} \mathcal{J}_r(-\boldsymbol{\xi}_i) \boldsymbol{\delta}) - \mathbf{d}
        = 
        \\ \mathrm{Exp}(-\boldsymbol{\xi}_i) \mathbf{d} +  \mathrm{Exp}(-\boldsymbol{\xi}_i) \mathbf{H}_{i} \mathcal{J}_{r}(-\boldsymbol{\xi}_i) \boldsymbol{\delta} - \mathbf{d} \stackrel{\eqref{eq:tilde_H}}{=} 
        \\ [\mathrm{Exp}(-\boldsymbol{\xi}_i) \mathbf{d} - \mathbf{d}]_{1:3} +  \mathrm{Exp}_{\mathrm{SO}(3)}(-[\boldsymbol{\xi}_i]_{1:3}) \tilde{\mathbf{H}}_{i} \mathcal{J}_{r}(-\boldsymbol{\xi}_i) \boldsymbol{\delta} 
    .\end{aligned}
\end{equation}

From Proposition 2 in \cite{goffin2025}, or, more generally, from the methodology of the iterated EKF \cite{bell1994iterated}, follows the Gauss-Newton sequence for 
the optimization problem in~\eqref{eq:minimization}:
\begin{equation}
    \label{eq:next_H}
    \tilde{\mathbf{H}}_{i}^{j+1} = \mathrm{Exp}_{\mathrm{SO}(3)}(-[\boldsymbol{\xi}_{i}^{j}]_{1:3}) 
    \tilde{\mathbf{H}}_{i} \mathcal{J}_{r}(-\boldsymbol{\xi}_{i}^{j}),
\end{equation}
\begin{equation}
    \mathbf{S}_i^{j+1} = \tilde{\mathbf{H}}_{i}^{j+1} \mathbf{P}_{i|i-1} (\tilde{\mathbf{H}}_{i}^{j+1})^{T} + \hat{\tilde{\mathbf{N}}}_{i},
\end{equation}
\begin{equation}
    \mathbf{K}^{j+1}_{i} = \mathbf{P}_{i|i-1}(\tilde{\mathbf{H}}_{i}^{j+1})^{T}
    (\mathbf{S}_i^{j+1})^{-1},
\end{equation}
\begin{equation}
    \boldsymbol{\xi}_{i}^{j+1} = 
    \mathbf{K}_{i}^{j+1} (\tilde{\mathbf{z}}_{i} - [\mathrm{Exp}(-\boldsymbol{\xi}_{i}^{j})\mathbf{d} - \mathbf{d}]_{1:3} + 
    \tilde{\mathbf{H}}_{i}^{j+1} \boldsymbol{\xi}_{i}^{j}),
\end{equation}
\begin{equation}
  \boldsymbol{\xi}_{i}^{0} := \mathbf{0}_{9,1}.
\end{equation}
Finally, let $k$ denote the last iteration, then the update step for the right-invariant error IterIEKF
is given as follows:
\begin{equation}
    \begin{aligned}
        \hat{\boldsymbol{\xi}}_{i} := \boldsymbol{\xi}_{i}^{k}, \\
    \end{aligned}
\end{equation}
\begin{equation}
    \bar{\mathcal{X}}_{i} := \bar{\mathcal{X}}_{i|i} := \hat{\boldsymbol{\xi}}_{i} \oplus   \bar{\mathcal{X}}_{i|i-1},
\end{equation}
\begin{equation}
    \label{eq:covariance_update}
    \mathbf{P}_{i|i} = (\mathbf{I}_9 - \mathbf{K}_{i}^{0} \mathbf{H}_{i}^{0}) \mathbf{P}_{i|i-1},
\end{equation}
where we just need to update the covariance matrix once since $\mathbf{H}_i$ does not depend on 
$\boldsymbol{\xi}_i$ \cite{goffin2025}. To obtain the last iteration, we implemented as stopping
criterion the condition $\| \boldsymbol{\xi}^{j-1} - \boldsymbol{\xi}^j \| < \delta$, 
where $\delta = 10^{-4}$. We also implemented a second optional stopping criterion based on the 
computation of the loss in~\eqref{eq:minimization}. If it increases, the algorithm stops.

\subsection{${SO}(3)$-EKF and Iter${SO}(3)$-EKF}

In this section we develop the $\mathrm{SO}(3)$-EKF and its iterated version to serve as a 
baseline for the proposed method, although the later has never been formulated in the context of legged robot state estimation before, to our knowledge, and thus can be considered as a secondary contribution, as mentioned in Sec. \ref{sec:contributions}.
The $\mathrm{SO}(3)$-EKF formulation is similar to the quaternion-based one proposed in \cite{bloesch2013}. The main differences are:
(a) a world-centric view instead of robot-centric, (b) we do not consider IMU biases, and (c) we implement an error-state Extended Kalman Filter instead of an Unscented Kalman Filter.
The true state was defined in~\eqref{so3_true_state}, 
while the nominal state is $\bar{\boldsymbol{x}}_i := \{\bar{\mathbf{R}}_i, \bar{\mathbf{v}}_i, \bar{\mathbf{p}}_i\}$.
We define the right error-state at each iteration as:
\begin{equation}
    \begin{aligned}
        \delta \boldsymbol{x}_i := \{\delta \boldsymbol{\phi}_i, \delta \mathbf{v}_i , \delta \mathbf{p}_i \}  := \\
        \{\mathrm{Log}_{\mathrm{SO}(3)}(\mathbf{R}_i \bar{\mathbf{R}}_i^{T}), \mathbf{v}_i - \bar{\mathbf{v}}_i, \mathbf{p}_i - \bar{\mathbf{p}}_i\},\\
        \delta \boldsymbol{x}_i \sim \mathcal{N}({\mathbf{0}_9, \mathbf{\Sigma}
        _i}),
    \end{aligned}
\end{equation}
where $\mathbf{\Sigma}_i \in \mathbb{R}^{9 \times  9}$, represents the (right) error covariance. While prediction step follows~\eqref{eq:prediction_step}, the evolution of the error-state is given approximately by (\cite[p. 57]{sola2017},\cite[p. 15]{forster2017}): 
\begin{equation}
    \delta \boldsymbol{\phi}_{i} \approx \delta \boldsymbol{\phi}_{i-1} + 
    \mathrm{Adj}_{\bar{\mathbf{R}}_{i|i-1}} \mathbf{G}^{1}_{i-1} \mathbf{w}_{g, i-1},
\end{equation}
\begin{equation}
    \delta \mathbf{v}_{i} \approx 
    \delta \mathbf{v}_{i-1} - (\bar{\mathbf{R}}_{i-1}\tilde{\mathbf{a}}_{I,i-1})^{\wedge} \delta \boldsymbol{\phi}_{i-1} \Delta t - 
    \bar{\mathbf{R}}_{i-1} \mathbf{w}_{a,i-1} \Delta t,
\end{equation}
\begin{equation}
    \begin{aligned}
        \delta \mathbf{p}_{i} \approx 
        \delta \mathbf{p}_{i-1} + \delta \mathbf{v}_{i-1} \Delta t 
        - (\bar{\mathbf{R}}_{i-1}\tilde{\mathbf{a}}_{I,i-1})^{\wedge} \delta \boldsymbol{\phi}_{i-1} \frac{\Delta t^{2}}{2} \\
        -  \bar{\mathbf{R}}_{i-1} \mathbf{w}_{a,{i-1}} \frac{\Delta t^{2}}{2}
    .\end{aligned}
\end{equation}
The covariance propagation is given approximately by:
\begin{equation}
    \mathbf{\Sigma}_{i|i-1} \approx \bar{\mathbf{A}}_{i-1} \mathbf{\Sigma}_{i-1} \bar{\mathbf{A}}_{i-1}^{T} 
    + \bar{\mathbf{B}}_{i-1} \mathbf{Q}_{i-1} \bar{\mathbf{B}}_{i-1}^{T},
\end{equation}
where
\begin{equation}
    \begin{aligned}
        \bar{\mathbf{A}}_{i-1} = \begin{bmatrix}
            \mathbf{I}_3 & \mathbf{0}_{3} & \mathbf{0}_{3} \\
            -(\bar{\mathbf{R}}_{i-1}\tilde{\mathbf{a}}_{I,i-1})^{\wedge}\Delta t & \mathbf{I}_{3} & \mathbf{0}_{3} \\
            -(\bar{\mathbf{R}}_{i-1} \tilde{\mathbf{a}}_{I,i-1})^{\wedge} \frac{\Delta t^{2}}{2} &  \mathbf{I}_3 \Delta t & \mathbf{I}_{3}
        \end{bmatrix}, \\
        \bar{\mathbf{B}}_{i-1} = 
        \begin{bmatrix}
            \mathrm{Ad}_{\bar{\mathbf{R}}_{i|i-1}}\mathbf{G}_{i-1}^{1} & \mathbf{0}_{3} \\
            \mathbf{0}_{3} & -\bar{\mathbf{R}}_{i-1}\Delta t \\
            \mathbf{0}_{3} & -\bar{\mathbf{R}}_{i-1} \frac{\Delta t^{2}}{2}
        \end{bmatrix}
    .\end{aligned}
\end{equation}
Although $\mathrm{SE}_2(2)$ and $\mathrm{SO}(3) \times \mathbb{R}^{6}$ are diffeomorphic as smooth manifolds, they do not share the same Lie group structure. As a result, the map $\Phi$ in~\eqref{eq:prediction_step} is not a group automorphism of $\mathrm{SO}(3) \times \mathbb{R}^{6}$. Consequently, the corresponding noise-free error dynamics depend on the current state and are not autonomous (see Theorem~18 in \cite{barrau2019}).

For the $\mathrm{SO}(3)$-EKF, the innovation is defined as \cite{bloesch2012}:
\begin{equation}
    \begin{aligned}
        \delta \mathbf{z}_{i} := \tilde{\mathbf{v}}_{B,i} - \bar{\mathbf{R}}_i^{T} \bar{\mathbf{v}}_i = \\
        (\delta \boldsymbol{\phi}_i \oplus \bar{\mathbf{R}}_{i} )^{T}(\bar{\mathbf{v}}_i + \delta \mathbf{v}_i) 
        - \bar{\mathbf{R}}^{T}_i \bar{\mathbf{v}}_i + \mathbf{w}_{f,i} \approx \\
        \bar{\mathbf{R}}_i^{T} \bar{\mathbf{v}}_i^{\wedge} \delta \boldsymbol{\phi}_i +
        \bar{\mathbf{R}}_i^{T} \delta \mathbf{v}_i + 
        \mathbf{w}_{f,i},
    \end{aligned}
\end{equation}
then we define 
\begin{equation}
    \bar{\mathbf{H}}_{i}(\bar{\boldsymbol{x}}_i) := \bar{\mathbf{R}}_i^{T}
    \begin{bmatrix}
        \bar{\mathbf{v}}_{i}^{\wedge} & \mathbf{I}_3 & \mathbf{0}_{3}
    \end{bmatrix},
\end{equation}
and we consider $\mathrm{Exp}(\delta \boldsymbol{\phi}_i) \approx (\mathbf{I}_3 + \delta \boldsymbol{\phi}_i^{\wedge})$ and 
$\delta \boldsymbol{\phi}_i^{\wedge} \delta \mathbf{v}_i \approx \mathbf{0}_3$. Notice that the measurement model depends on the
nominal state $\bar{\boldsymbol{x}}_i$ through $\bar{\mathbf{H}}_i$, because the measurement function is not 
an invariant observer for the group $\mathrm{SO}(3) \times \mathbb{R}^{6}$.
Finally, the update step is obtained with 
the well-known EKF equations~\cite{bloesch2012}:
\begin{equation}
    \bar{\mathbf{S}}_i = \bar{\mathbf{H}}_{i} \mathbf{\Sigma}_{i|i-1} \bar{\mathbf{H}}_{i}^{T} + \mathbf{Q}_f,
\end{equation}
\begin{equation}
    \bar{\mathbf{K}}_i = \mathbf{\Sigma}_{i|i-1} \bar{\mathbf{H}}_{i}^{T} \bar{\mathbf{S}}_i^{-1},
\end{equation}
\begin{equation}
    \delta \hat{\boldsymbol{x}}_{i} = \bar{\mathbf{K}}_i \delta \mathbf{z}_{i},
\end{equation}
\begin{equation}
    \mathbf{\Sigma}_{i|i} = (\mathbf{I}_9 - \bar{\mathbf{K}}_i \bar{\mathbf{H}}_{i}) \mathbf{\Sigma}_{i|i-1}.
\end{equation}

The Iter$\mathrm{SO}(3)$-EKF can be derived by considering the minimization problem analogous to~\eqref{eq:minimization}:
\begin{equation}
    \label{eq:minimization_so3}
    \begin{aligned}
&\argmin_{\delta \boldsymbol{x}} 
\frac{1}{2}\|\delta \boldsymbol{x} \|_{\mathbf{\Sigma}_{i|i-1}}^2 + \\  
&\frac{1}{2}\| \tilde{\mathbf{v}}_{B,i} - \underbrace{(\delta \boldsymbol{\phi} \oplus \bar{\mathbf{R}}_{i|i-1})^{T} 
    (\bar{\mathbf{v}}_{i|i-1} + 
\delta \mathbf{v})}_{\mathbf{f}(\delta \boldsymbol{x})} \|^2_{\mathbf{Q}_f}.
        \end{aligned}
    \end{equation}
Analogously to the IterIEKF, we can solve~\eqref{eq:minimization_so3} using the Gauss-Newton method. We just
need to provide a linearization of $\mathbf{f}(\delta \boldsymbol{x})$:
\begin{equation}
    \begin{aligned}
        (\mathrm{Exp}(\delta \boldsymbol{\phi} + \boldsymbol{\epsilon}_{r})\bar{\mathbf{R}}_{i})^{T}
        (\bar{\mathbf{v}}_{i} + \delta \mathbf{v} + \boldsymbol{\epsilon}_{v}) 
        \approx \\
        [(\mathrm{Exp}(\delta \boldsymbol{\phi})
        \mathrm{Exp}(\mathcal{J}_{r, \mathrm{SO}(3)}(\delta \boldsymbol{\phi})\boldsymbol{\epsilon}_{r}) \bar{\mathbf{R}}_{i}]^{T} 
        (\bar{\mathbf{v}}_{i} + \delta \mathbf{v} + \boldsymbol{\epsilon}_{v}) 
        \approx \\
        [\mathrm{Exp}(\delta \boldsymbol{\phi})(\mathbf{I}_{3} + [\mathcal{J}_{r,\mathrm{SO}(3)}(\delta \boldsymbol{\phi}) \boldsymbol{\epsilon}_{r}]^{\wedge}) \bar{\mathbf{R}}_{i}]^{T}(\bar{\mathbf{v}}_{i} + \delta \mathbf{v} + \boldsymbol{\epsilon}_{v}) 
        \approx \\ 
        [\mathrm{Exp}(\delta \boldsymbol{\phi}) \bar{\mathbf{R}}_{i}]^{T} (\bar{\mathbf{v}}_{i} + \delta \mathbf{v}) 
        + \\ \bar{\mathbf{R}}_{i}^{T} \mathrm{Exp}(-\delta \boldsymbol{\phi})
        (\bar{\mathbf{v}}_{i} + \delta \mathbf{v})^{\wedge} 
        \mathcal{J}_{r, \mathrm{SO}(3)}(-\delta \boldsymbol{\phi})  \boldsymbol{\epsilon}_{r} + \\ 
        \bar{\mathbf{R}}_{i}^{T} \mathrm{Exp}(-\delta \boldsymbol{\phi}) \boldsymbol{\epsilon}_{v},
    \end{aligned}
\end{equation}
where in the last approximation we considered $\boldsymbol{\epsilon}_r^{\wedge} \boldsymbol{\epsilon}_v \approx \mathbf{0}_3$. Then, 
the Gauss-Newton sequence for Iter$\mathrm{SO}(3)$-EKF is given as follows:
\begin{equation}
    \begin{aligned}
        \bar{\mathbf{H}}_{i}^{j+1} =
        (\delta \boldsymbol{\phi}_i^{j} \oplus \bar{\mathbf{R}}_{i})^{T}\\
        \begin{bmatrix}
            (\bar{\mathbf{v}}_{i} + \delta \mathbf{v}_i^{j})^{\wedge} 
            \mathcal{J}_{r, \mathrm{SO}(3)}(-\delta \boldsymbol{\phi}_i^{j}) &
            \mathbf{I}_3 &
            \mathbf{0}_3
        \end{bmatrix},
    \end{aligned}
\end{equation}
\begin{equation}
    \bar{\mathbf{S}}_i^{j+1} = \bar{\mathbf{H}}_{i}^{j+1} \mathbf{\Sigma}_{i|i-1} (\bar{\mathbf{H}}_{i}^{j+1})^{T} + \mathbf{Q}_f,
\end{equation}
\begin{equation}
    \bar{\mathbf{K}}^{j+1}_{i} = \mathbf{\Sigma}_{i|i-1}(\bar{\mathbf{H}}_{i}^{j+1})^{T}
    (\bar{\mathbf{S}}_i^{j+1})^{-1},
\end{equation}
\begin{equation}
    \delta \boldsymbol{x}_{i}^{j+1} = 
    \bar{\mathbf{K}}_{i}^{j+1} (\tilde{\mathbf{v}}_{B,i} - \mathbf{f}(\delta \boldsymbol{x}_i^{j}) + 
    \bar{\mathbf{H}}_{i}^{j+1} \delta \boldsymbol{x}_{i}^{j}),
\end{equation}
\begin{equation}
    \bar{\mathbf{H}}_{i}^{0} := \bar{\mathbf{H}}_{i}(\mathbf{0}_{9,1}),\delta \boldsymbol{x}_{i}^{0} := \mathbf{0}_{9,1}.
\end{equation}
Let $k$ denote the last iteration, then the update step for the Iter$\mathrm{SO}(3)$-EKF is given as follows:
\begin{equation}
    \delta \hat{\boldsymbol{x}}_{i} := \delta \boldsymbol{x}_{i}^{k},
\end{equation}
\begin{equation}
    \bar{\boldsymbol{x}}_i := \bar{\boldsymbol{x}}_{i|i} =\delta \hat{\boldsymbol{x}}_{i} \oplus \bar{\boldsymbol{x}}_{i|i-1},
\end{equation}
\begin{equation}
    \mathbf{\Sigma}_{i|i} = (\mathbf{I}_9 - \bar{\mathbf{K}}_{i}^{k} \bar{\mathbf{H}}_{i}^{k}) \mathbf{\Sigma}_{i|i-1}.
\end{equation}
\begin{figure}
    \centering
    \includegraphics[width=0.4\textwidth]{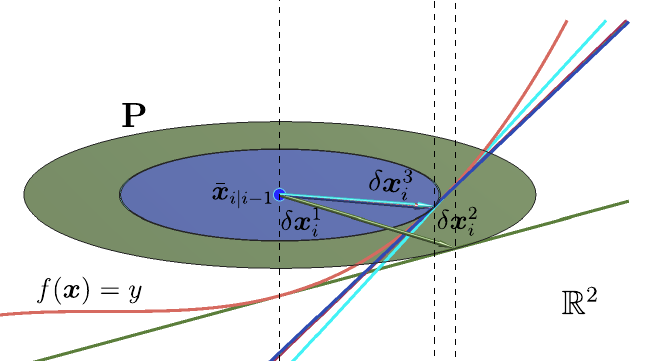}
    \caption{Geometrical interpretation of the Iterated Kalman Filter. The red curve is the level set of the  
        non-linear measurement  ${f}(\boldsymbol{x}) = y$. At each iteration $j$, 
        the measurement function is linearized around the state  $\bar{\boldsymbol{x}}_{i|i-1} + \delta \boldsymbol{x}^{j-1}_i$, and the first contact point 
        of the ellipsoid (prior) with the tangent plane gives an estimate of the error $\delta \boldsymbol{x}_i$.
    }
    \label{fig:iterated_kalman_filter}
\end{figure}
The update step of Iter$\mathrm{SO}(3)$-EKF can be seen as a natural extension of the Iterated Kalman Filter presented in \cite{bell1994iterated}.
The main difference is that the optimization is being performed on a Lie group instead of $\mathbb{R}^{n}$, as the state variable encoding the orientation does not live in a vector space. For completeness, 
we give a geometrical interpretation of the Iterated Kalman Filter in Fig. \ref{fig:iterated_kalman_filter}. This also helps to justify why we use the 
last Jacobian $\bar{\mathbf{H}}_{i}^{k}$ to update the covariance matrix. 
We can see in Fig. \ref{fig:iterated_kalman_filter} that at each iteration $j$, the measurement function is linearized around the current estimate of the state, 
and the first contact point of the ellipsoid (prior) with the tangent plane gives an estimate of the error $\delta \boldsymbol{x}_i$. Therefore, at the last iteration $k$,
the covariance matrix should be updated with the last Jacobian $\bar{\mathbf{H}}_{i}^{k}$. 

\subsection{Observability analysis}
For the observability analysis, we consider the dynamics of the right-invariant error without process noise and measurement noise:
\begin{equation}
    \begin{aligned}
        \boldsymbol{\xi}_{i}  = \mathbf{A}_{i} \boldsymbol{\xi}_{i-1} =  
        \begin{bmatrix}
            \mathbf{I}_{3} & \mathbf{0}_{3} & \mathbf{0}_{3} \\
            \Delta t\mathbf{g}_{\times}  & \mathbf{I}_{3} & \mathbf{0}_{3}  \\
            \frac{\Delta t^{2}}{2} \mathbf{g}_{\times}  & \mathbf{I}_{3} & \mathbf{I}_{3}
        \end{bmatrix}  \boldsymbol{\xi}_{i-1},
        \\
        \mathbf{z}_{i} = \mathrm{Exp}(\boldsymbol{\xi}_i)^{-1} \mathbf{d} - \mathbf{d}.
    \end{aligned}
\end{equation}
Following \cite[p. 21]{Hartley2019ContactaidedIE}, we built the observability matrix:
\begin{equation}
    \label{eq:observability}
    \begin{aligned}
        \begin{bmatrix}
            \mathbf{z}_{i} \\
            \mathbf{z}_{i+1} \\
            \mathbf{z}_{i+2}
        \end{bmatrix} = 
        \begin{bmatrix}
            \mathrm{Exp}(\boldsymbol{\xi}_{i})^{-1} \mathbf{d} - \mathbf{d} \\
            \mathrm{Exp}(\mathbf{A}_{i} \boldsymbol{\xi}_{i})^{-1} \mathbf{d} - \mathbf{d}  \\
            \mathrm{Exp}(\mathbf{A}_{i}^{2} \boldsymbol{\xi}_{i})^{-1} \mathbf{d} - \mathbf{d}
        \end{bmatrix} \approx \\
        \begin{bmatrix}
            \mathbf{H} \\
            \mathbf{H} \mathbf{A}_{i}\\
            \mathbf{H}\mathbf{A}_{i}^{2}
        \end{bmatrix} \boldsymbol{\xi}_{i} =  
        \underbrace{ \begin{bmatrix}
                \mathbf{0}_{3} & \mathbf{I}_{3} & \mathbf{0}_{3} \\
                \Delta t \mathbf{g}_{\times} & \mathbf{I}_{3} & \mathbf{0}_{3} \\
                2 \Delta t \mathbf{g}_{\times} &  \mathbf{I}_{3} & \mathbf{0}_{3}
        \end{bmatrix} }_{ \mathcal{O} }\boldsymbol{\xi}_{i}
    .\end{aligned}
\end{equation}
By inspecting $\mathcal{O}$ in~\eqref{eq:observability}, we see that the position is not observable. Moreover, as the third column of $\mathbf{g}_{\times}$ is composed of zeros, we have that the yaw angle is also not observable. This conclusion was the same as the one obtained in \cite{bloesch2013},
although our approach required far less computation as it leverages natural $\mathrm{SE}_2(3)$ Lie group coordinates, in which the observablity matrix is independent of the state, contrary to a more standard analysis \cite{bloesch2013}. Notice that the fact that the velocity part of the right-invariant error is observable means that the 
velocity in the IMU frame is observable. This is shown in the following calculation \cite[p. 5]{santana2024}:
\begin{equation}
    \begin{aligned}
        e^{\mathbf{v}} :=  
        \mathbf{v} - \mathbf{R}\bar{\mathbf{R}}^{T} \bar{\mathbf{v}}  =
        \mathcal{J}_{l,\mathrm{SO}(3)}(\boldsymbol{\xi}^{R}) \underbrace{\boldsymbol{\xi}^{\mathbf{v}}}_{\to \mathbf{0}_3} \to  \mathbf{0}_3 \implies \\
        \bar{\mathbf{R}}^{T}\bar{\mathbf{v}} \to \mathbf{R}^{T} \mathbf{v}.
    \end{aligned}
\end{equation}
As the estimation error converges to zero, the velocity expressed in the base frame also converges, demonstrating its observability. For this reason, the velocity is presented in the base frame in the Results section.

Analogously, we have that any rotation is observable except for the one around the gravity vector. 
Mathematically, this means that $\xi_x^{R} \to  0$ and $\xi_y^{R} \to  0$ and, in the limit, $\boldsymbol{\xi}^{\mathbf{R}} \to \xi_z^{R} \mathbf{e}_3$, $\mathbf{e}_3 := [0 \,0\,1]^T$. 
Moreover, this implies that only the gravity direction in the base frame, $\mathbf{u}:= \mathbf{R}^{T} \mathbf{e}_3$, is observable: 
\begin{equation}
    \label{eq:vertical_observability}
    \begin{aligned}
        \mathbf{u} = \mathbf{R}^{T} \mathbf{e}_{3} =  \\
        \bar{\mathbf{R}}^{T} \mathrm{Exp}_{\mathrm{SO}(3)}(-\boldsymbol{\xi}^{\mathbf{R}}) =  \\
        \bar{\mathbf{R}}^{T}[\mathbf{I}_{3} - (\alpha \mathbf{I}_{3} - \beta (\boldsymbol{\xi}^{\mathbf{R}})^{\wedge}) (\boldsymbol{\xi}^{\mathbf{R}})^{\wedge}] \mathbf{e}_{3} \stackrel{\boldsymbol{\xi}^{\mathbf{R}} \to \xi_z^{R} \mathbf{e}_3}{\implies}  \\
        \bar{\mathbf{R}}^{T}[\mathbf{e}_{3} -
        \xi_{z}^{R}(\alpha \mathbf{I}_{3} - \beta \xi_{z}^{R} \mathbf{e}_{3}^{\wedge})\underbrace{ \mathbf{e}_{3}^{\wedge} \mathbf{e}_{3} }_{ \mathbf{0}_{3}})] = 
        \bar{\mathbf{R}}^{T} \mathbf{e}_{3} := \bar{\mathbf{u}},
    \end{aligned}
\end{equation}
where $\alpha$ and $\beta$ are coefficients of the Rodrigues' rotation formula \cite{sola2018}. 
We denote by roll and 
pitch angles the following components of $\mathbf{u}$: $\theta := \mathrm{atan2}(u_y, u_z)$ and $\phi := \mathrm{asin}(u_x)$ 
\cite[p. 114]{corke2023}.

\subsection{Geometric Interpretation of Invariant Iterated Kalman Filter in the face of noise-free measurements}
\label{section:geometric-intuition}

\begin{figure*}[t]
    \centering
    \subfloat[\protect\label{fig:oblique_projection}]{
\def\svgwidth{\columnwidth}%
\resizebox{0.40\textwidth}{!}{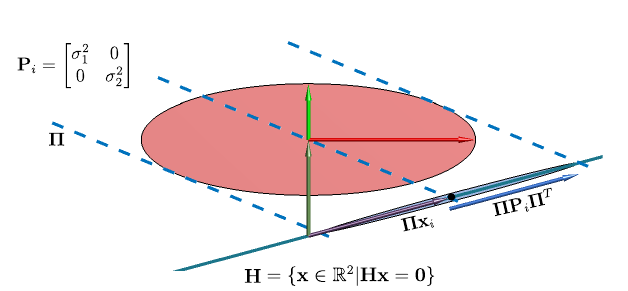}%

    }
    \subfloat[\protect\label{fig:iiekf_geom}]{
\def\svgwidth{\columnwidth}%
\resizebox{0.30\textwidth}{!}{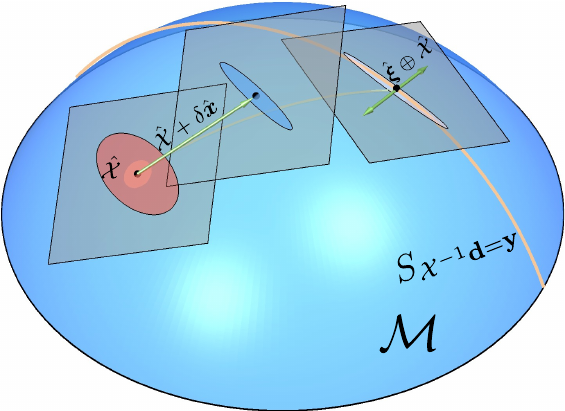}%

    }
    \caption{a) The Kalman Filter update can be seen as an oblique projection on the observed set. 
        b) Update step with EKF ($\hat{\mathcal{X}} + \delta \hat{\boldsymbol{x}}$) and IterIEKF ($\hat{\boldsymbol{\xi}} \oplus \hat{\mathcal{X}}$). While after the update with IterIEKF the state is guaranteed to belong to the 
    constraint and there is no uncertainty orthogonal to it, neither of these properties can be guaranteed by the EKF.}
\end{figure*}

First, recall that in the face of noise-free measurements, the linear Kalman Filter update is given by 
\cite[p. 4]{goffin2025}:
\begin{equation}
    \label{eq:lin_kalman}
    \begin{aligned}
        \mathbf{x}_{i + 1} = \mathbf{x}_i + \mathbf{P}_i\mathbf{H}^{T} (\mathbf{H} \mathbf{P}_i \mathbf{H}^{T})^{-1} 
        (\mathbf{y}_i - \mathbf{H} \mathbf{x}_i) = \\ \mathbf{x}_i + \mathbf{K}(\mathbf{y}_i - \mathbf{H} \mathbf{x}_i)
    .\end{aligned}
\end{equation}
If we consider $\mathbf{y}_i = \mathbf{0}$, then~\eqref{eq:lin_kalman} can be rewritten as:
\begin{equation}
    \label{eq:oblique_projection}
    \mathbf{x}_{i + 1} = (\mathbf{I} - \mathbf{K} \mathbf{H}) \mathbf{x}_i = \mathbf{\Pi} \mathbf{x}_i.
\end{equation}
It happens that~\eqref{eq:oblique_projection} has a simple geometric interpretation: it is the first contact of the distribution $\mathcal{N}(\mathbf{x}_i, \mathbf{P}_i)$
with the set $\text{Ker } \mathbf{H}$. This point of contact can be obtained by the oblique projection (defined by $\mathbf{\Pi}$) of $\mathbf{x}_i$ on $\text{Ker } \mathbf{H}$, i.e. $\mathbf{\Pi} \mathbf{x}_i$. This idea is illustrated in 
Figure \ref{fig:oblique_projection} for a 2D
example. It can be observed that the current covariance is collapsed in the direction of $\mathbf{\Pi}$ and the covariance of the new state, $\mathbf{P}_{i+1} = \mathbf{\Pi} \mathbf{P}_i \mathbf{\Pi}^{T}$,
is contained in $\text{Ker } \mathbf{H}$. 
This implies that there is no uncertainty along $\text{Im }\mathbf{H}$, i.e. $\mathbf{H} \mathbf{P}_{i+1} \mathbf{H}^{T} = \mathbf{0}$.
The same reasoning was used in \cite[p. 4]{ko2005} for dealing with the exact state's constraint. 

The main breakthrough of \cite{barrau2019,goffin2025} was showing that a natural extension to this methodology exists
when $\mathcal{X}$ belongs to a particular nonlinear manifold $\mathcal{M}$ that arises in robot state estimation. 
Let $S_{\mathcal{X}_i^{-1}\mathbf{d}=\mathbf{y}_i}$ denote the observed set associated to a noise-free measurement $\mathbf{y}_i$ (see Definition \ref{def:submanifold}). Then, it can be
shown that, after the update with IterIEKF, two properties are ensured for $\hat{\mathcal{X}}_{i}$ and $\mathbf{P}_{i|i}$: 
\begin{enumerate}

\item
The state will belong to the observed set. We prove this property as follows. As discussed in the previous section, 
if the Gauss-Newton descent properly converges, say to $\hat{\boldsymbol{\xi}}_{i}$, then the 
constraint $\mathbf{z}_i - \mathrm{Exp}(\hat{\boldsymbol{\xi}}_{i})^{-1} \mathbf{d} + \mathbf{d} = \mathbf{0}_3$ is satisfied, which implies:
\begin{equation}
\begin{aligned}
    \mathbf{z}_i - \mathrm{Exp}(-\hat{\boldsymbol{\xi}}_{i}) \mathbf{d} - \mathbf{d} = \mathbf{0}_3 \implies \\
    \mathrm{Exp}(-\hat{\boldsymbol{\xi}}_{i}) \mathbf{d} - \mathbf{d} = \bar{\mathcal{X}}_i\mathbf{y}_i - \mathbf{d} \implies \\
    \mathrm{Exp}(-\hat{\boldsymbol{\xi}}_{i}) \mathbf{d} = 
    \bar{\mathcal{X}}_i \mathbf{y}_i \implies \\
    (\mathrm{Exp}(\hat{\boldsymbol{\xi}}_{i})\bar{\mathcal{X}}_i)^{-1} \mathbf{d} = \mathbf{y}_i,
\end{aligned}
\end{equation}
then, after the update step, 
$\bar{\mathcal{X}}_{i} = \hat{\boldsymbol{\xi}}_{i} \oplus \bar{\mathcal{X}}_{i|i-i} \in  S_{\mathcal{X}^{-1}_i\mathbf{d} = \mathbf{y}_i}$, and the velocity in the base frame 
$\hat{\mathbf{R}}_{i}^{T}\hat{\mathbf{v}}_{i}$ will be identical to the true value in face of a noise-free measurement.

\item There will be no uncertainty transversal to the set, i.e., $\text{Im } \mathbf{P}_{i} = T_{\hat{\mathcal{X}}} S$, where the tangent space is expressed in the Lie algebra associated to right translations on the group (see Theorem 2 in \cite{goffin2025}). 

\end{enumerate}

On the other hand, no counterpart of those properties
exists for the vanilla EKF, and only the first property is ensured for the iterated EKF on Lie group (represented in this work by Iter$\mathrm{SO}(3)$-EKF). This is illustrated in Fig. \ref{fig:iiekf_geom}.

In the following, 
we describe the topology of the observed set for the proposed measurement model. 
For simplicity, we choose to illustrate in two dimensions using the group $\mathrm{SE}_2(2)$, with $\mathbf{d} = [\mathbf{0}_{2,1}, 1,0]^{T}$.
For this problem, the isotropy group $G_\mathbf{d}$ is given by:
\begin{equation}
    G_\mathbf{d} = \{(\mathbf{R}, \{\mathbf{v}, \mathbf{p} \} ) \in  \mathrm{SE_2(2)} | \mathbf{R} \in \mathrm{SO}(2), \mathbf{v} = \mathbf{0}_2, \mathbf{p} \in  \mathbb{R}^{2} \},
\end{equation}
and the observed set is $G_\mathbf{d} \mathcal{X}$, where $\mathcal{X}$ is a fixed true state. 
Notice that, different from the observed set discussed in Section \ref{sec:motivating_example}, 
the rotation-translation component in this case is not bounded. This reflects the fact that position and yaw angle are unobservable.

Now, we describe the influence of the covariance matrix in the update state found by IterIEKF. Let
$\mathbf{P} = \mathrm{diag}(P_\theta, P_x^{v}, P_y^{v}, P_x^{x}, P_y^{x})$ and the error state ${\boldsymbol{\xi}}_{i|i-1} = \{\xi_\theta, \xi_x^{v}, \xi_y^{v}, \xi_x^{p}, \xi_y^{p}\} \sim \mathcal{N}(\mathbf{0}, \mathbf{P})$. 
As shown in Fig. \ref{fig:compare_covariance}, the estimated error $\hat{\boldsymbol{\xi}}_{i}$ is dependent on the computed innovation $\mathbf{z}_i$
and current covariance $\mathbf{P}$. Similarly, we now demonstrate how its components impact the proposed update step. In Fig. \ref{fig:comparece_covariances_vel}, 
we considered three initial states, spaced by 1m along the $y$-axis, $\bar{\mathcal{X}} = \{i(\pi/3), i+1 , i+1 , 0, i\}$ with $i \in \{0,1,2\}$, and a true state 
$\mathcal{X} = \{\pi/3, 1, 1, 4, 2\}$. We simulated an update with IterIEKF using an almost noise-free true base velocity.
First, note that if $P_x^{v} = P_y^{v} = 0$, then trivially $\hat{\boldsymbol{\xi}}_{i} = \mathbf{0}_5$, since there is no uncertainty in the velocity error. 
Moreover, if $P_x^{v} = P_y^{v} = c$, then the solution is $\hat{\boldsymbol{\xi}}_i = \{0,\xi_x,\xi_y,0,0\} = \{0,\mathbf{R}^{T} \mathbf{v} - \bar{\mathbf{v}}_i\}$, 
which is a linear problem that can be solved exactly with IEKF. 
The situation becomes more interesting when $P_x^{v} = 0$ or $P_y^{v} = 0$. In this case, there is only one non-zero component in the velocity error, $\xi_x$ or $\xi_y$. Let us say that $P_y^{v} = 0$, then the only way to appear a component $y$ in the error velocity (in the base frame) is by means of a rotational component $\xi_\theta$ in the error. 
Analogously, when $P_x^v = 0$, $\xi_{\theta}$ is also required in order to appear a component $x$ in the error velocity.
This makes the problem nonlinear and unsolvable by IEKF. However, IterIEKF can solve it exactly. The results show that even in highly nonlinear cases, the update with IterIEKF successfully recovered the true base velocity. The number of iterations required for convergence in $i=2$ was $5$ for $P_y^{v} = 0$ and $20$ for $P_x^{v} = 0$. 
We also considered an intermediate case where $P_y^{v} / P_x^{v} = 2$, which required $6$ iterations to converge. As expected, the position and yaw angle remain unobservable.
\begin{figure}[t]
    \centering
    \subfloat[\protect\label{fig:comparece_covariances_vel}]{
        \includegraphics[scale=0.35]{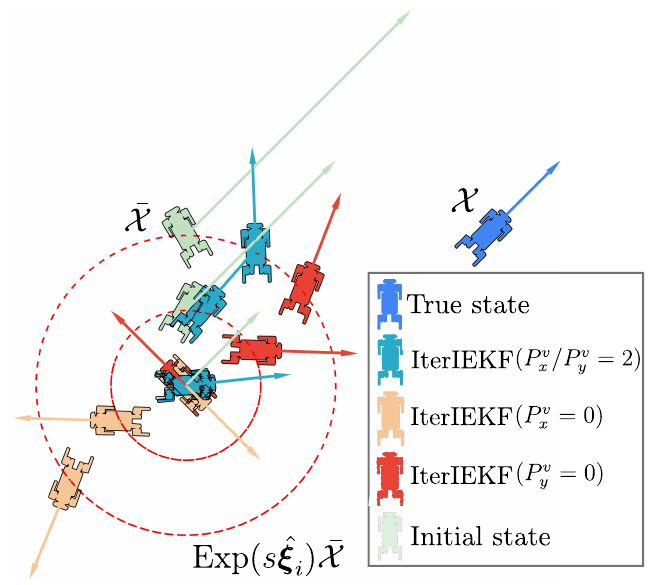}
    }
    \hfill
    \subfloat[\protect\label{fig:compare_update_vel}]{
         \includegraphics[scale=0.3]{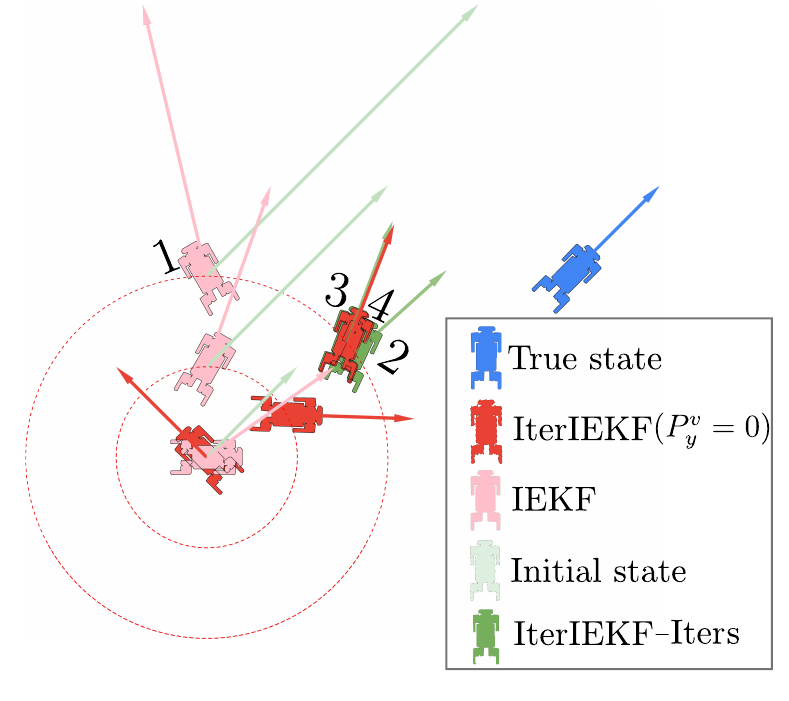}
    }
    \caption{Comparison of the update step using IterIEKF a) for different values of $\mathbf{P}$ b) and IEKF with almost noise-free velocity measurements in the base frame.}
    
\end{figure}

To illustrate the difference between IterIEKF and IEKF, we also performed an update step with IEKF under the same conditions as before. 
The result is shown in Fig. \ref{fig:compare_update_vel}, where the iterations of IterIEKF are also highlighted. We considered the case $P_y^{v} = 0$ to enforce a nonlinear problem. 
We can observe that IEKF is unable to perform a rotation in order to match the measured base velocity, it is only able to correct the velocity linearly.

\subsection{On the modification of the update step to avoid abrupt jumps}
As discussed in Section~\ref{section:geometric-intuition}, IterIEKF applies at each iteration a corrective rotation about the $z$-axis in order to enforce the linear velocity measurement in the base frame. 
In simulation, however, we observed that this mechanism may introduce noticeable jumps in the estimated yaw angle, as illustrated in Fig.~\ref{fig:without_modification}. 
It is important to emphasize that these yaw variations do not affect the quantities constrained by the measurement model: the linear velocity expressed in the base frame, as well as roll and pitch angles, remain unchanged. 
This behavior reflects the fact that yaw is non-observable under the considered sensing configuration. Consequently, the filter is free to adjust the yaw angle without violating the measurement constraints, which can lead to abrupt but dynamically consistent changes.
While this does not compromise the estimator consistency, such discontinuities may be undesirable in applications that rely directly on the yaw estimate, e.g., for navigation or high-level planning. 
To improve practical usability, we therefore introduce an optional post-update correction step that removes the spurious $z$-axis rotation while preserving all measurement-consistent quantities. 
Specifically, after the IterIEKF update, we apply
\begin{equation}
    \label{eq:modification}
    \hat{\mathcal{X}}_{i|i} := 
    \begin{bmatrix} 
        \mathbf{0}_{1,2} & -\hat{\xi}_{z,i}^{R} & \mathbf{0}_{1,6} 
    \end{bmatrix}^{T}
    \oplus 
    \hat{\mathcal{X}}_{i|i},
\end{equation}
which cancels the unnecessary yaw component introduced during the iterative correction. 
The same simulation with this optional modification is shown in Fig.~\ref{fig:with_modification}.
\begin{figure}[t]
    \centering
    \subfloat[\protect\label{fig:without_modification}]{
        \includegraphics[scale=0.05]{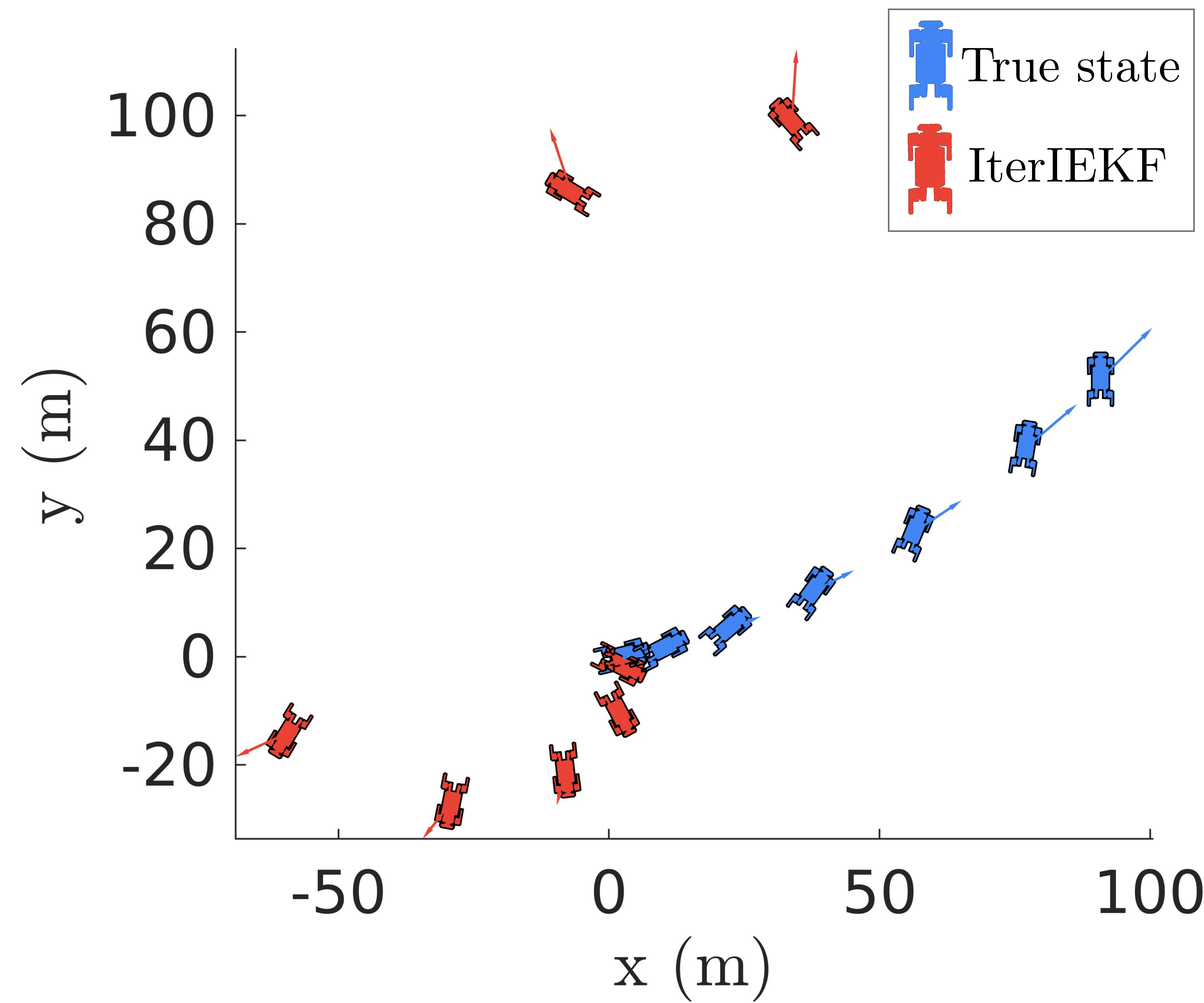}
    }
    \subfloat[\protect\label{fig:with_modification}]{
        \includegraphics[scale=0.035]{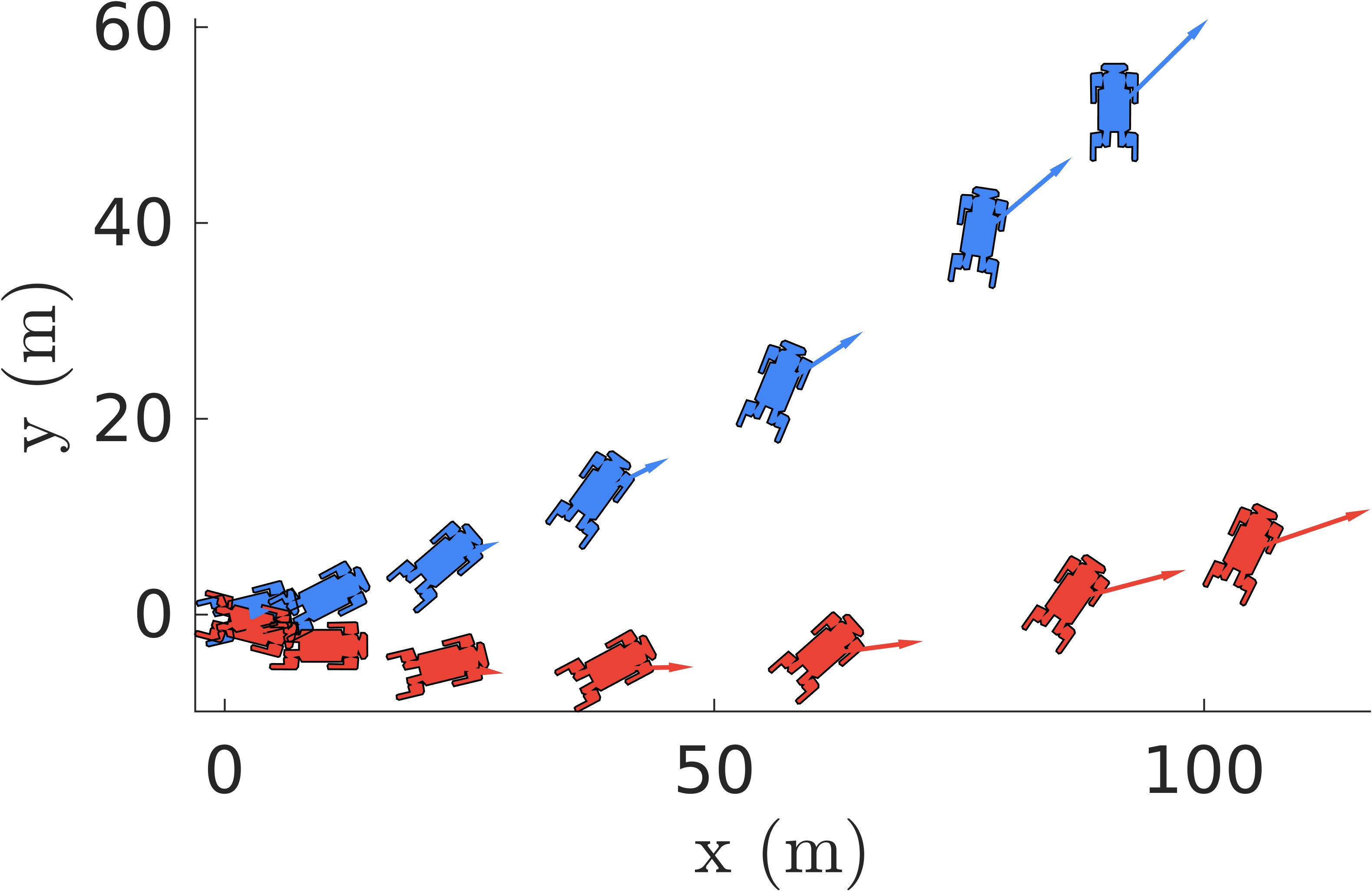}
    }
    \caption{a) The IterIEKF update (in red) causes abrupt jumps in the position estimate. 
    b) Mitigation of the abrupt changes caused by the update step, applying the modification proposed in proposed in~\eqref{eq:modification}.}
\end{figure}

\section{RESULTS}
In the following sections, we compare the IEKF, IterIEKF, $\mathrm{SO}(3)$-EKF, and Iter$\mathrm{SO}(3)$-EKF
across four experimental settings. The first evaluates the convergence time in a MuJoCo simulation, 
the second consists of a Monte Carlo study with synthetic data, 
the third is a Monte Carlo study using data from a quadruped robot simulator, 
and the fourth presents results using a real-world dataset.

Monte Carlo simulations are performed by first defining an initial true state $\mathcal{X}_0$, white-noise models for the IMU measurements and base velocity, and an initial error covariance $\mathbf{P}_0$. At the beginning of each realization $k$, an error vector $\boldsymbol{\xi}_0^{k}$ is sampled from $\mathcal{N}(\mathbf{0}_{9,1}, \mathbf{P}_0)$. Then, for the invariant filters, the simulation is initialized as 
$\bar{\mathcal{X}}_0 = (- \boldsymbol{\xi}_0^{k}) \oplus \mathcal{X}_0 = (\bar{\mathbf{R}}_0, \bar{\mathbf{v}}_0, \bar{\mathbf{x}}_0)$.
For the non-invariant filters, first we find a linear relation between $\boldsymbol{\xi}_0^{k}$ and $\delta \boldsymbol{x}_0^{k}$ which is given by \cite[pg. 39]{Hartley2019ContactaidedIE}:
\begin{equation}
\delta \boldsymbol{x}_{0}^{k} \approx  
\begin{bmatrix}
 \mathbf{I}_{3} & \mathbf{0}_{3} & \mathbf{0}_{3} \\
- \bar{\mathbf{v}}_{0}^{\wedge} & \mathbf{I}_{3} & \mathbf{0}_{3} \\
- \bar{\mathbf{p}}_{0}^{\wedge} & \mathbf{0}_{3} & \mathbf{I}_{3} 
\end{bmatrix}
\boldsymbol{\xi}_{0}^{k} 
= \mathbf{J}^{k} \boldsymbol{\xi}_0^{k},
\end{equation}
thus it follows the covariance initialization $\mathbf{\Sigma}_0^{k} = \mathbf{J}^{k} \mathbf{P}_0 \mathbf{J}^{k,T}$ and the nominal state is taken 
as $\bar{\boldsymbol{x}}_0 = \delta \boldsymbol{x}_0^{k} \oplus \boldsymbol{x}_0$.
At each time step, the estimation error with respect to the ground truth is computed and used to evaluate the average Normalized Estimation Error Squared (NEES) as a measure of filter consistency
\cite[p. 165]{shalom2002}:
\begin{equation}
    \label{eq:nees}
    \begin{aligned}
        \mathrm{NEES}_i^{\mathrm{SE_2}(3)} = \frac{1}{K} \sum_{k=1}^{K} \boldsymbol{\xi}_{i}^{k,T} (\mathbf{P}_{i}^{k})^{-1} 
        \boldsymbol{\xi}_{i}^{k}, \\
        \mathrm{NEES}_i^{\mathrm{SO}(3) \times \mathbb{R}^{6}} = \frac{1}{K} \sum_{k=1}^{K} 
        {\delta \boldsymbol{x}}_{i}^{k,T} (\mathbf{\Sigma}_{i}^{k})^{-1} {\delta  \boldsymbol{x}}_{i}^{k}, \\
        \boldsymbol{\xi}_{i}^{k} = \mathrm{Log}_{\mathrm{SE}_{2}(3)} ({\mathcal{X}}_{i} (\bar{\mathcal{X}}_{i|i})^{-1}), \\
        \delta \boldsymbol{x}_{i}^{k} = 
        \mathrm{Log}_{\mathrm{SO_3 \times \mathbb{R}^{6}}} 
        (\boldsymbol{x}_{i} (\bar{\boldsymbol{x}}_{i|i})^{-1}),
    \end{aligned}
\end{equation}
where $K$ is the total number of realizations.

In this work, we considerer two metrics, Mean Absolute Error (MAE), and  Root Mean Square Error (RMSE), of the observable state, composed of base velocity in the base frame 
($\bar{\mathbf{R}}_{i|i}^{T} \bar{\mathbf{v}}_{i|i}$), gravity direction ($\bar{\mathbf{u}}_{i|i}$), pitch angle 
($\bar{\phi}_{i|i}$), and roll angle ($\bar{\mathbf{\theta}}_{i|i}$). We use MAE for cases in which a large initial error was introduced, and RMSE for cases in which no such error was present. They are given as follows:
{\small
\begin{equation}
	\mathrm{MAE}_{\mathbf{x}} = \frac{1}{n} \sum_{i=1}^{n}  e(\bar{\mathbf{x}}_{i|i}, \mathbf{x}_{i}),\;
	\mathrm{RMSE}_{\mathbf{x}} := \sqrt{\frac{1}{n} \sum_{i=1}^{n} e(\bar{\mathbf{x}}_{i|i},\mathbf{x}_{i})^{2}}, 
\end{equation}
}
\noindent where $n$ is the number of timestamps, equal to the frequency of the proprioceptive sensors multiplied by the duration of the trajectory, 
and $e(\mathbf{x}_i, \bar{\mathbf{x}}_{i|i})$ is a metric that measures the error between the estimate $\bar{\mathbf{x}}_{i|i}$ and the 
corresponding interpolated reference $\mathbf{x}_i$ at time step $i$. For most estimates, this metric is the Euclidean norm, 
$e := \| \mathbf{x}_i - \bar{\mathbf{x}}_{i|i} \|$. For the gravity direction, however, we use 
$e := \mathrm{acos}(\bar{\mathbf{u}}_{i|i} \cdot \mathbf{u}_i)$, which provides a more natural interpretation of the error.
\subsection{State Estimation in MuJoCo Simulation under large initial error}
In this experiment, we evaluate the performance of the four filters on the eight-shape trajectory described in Section 
\ref{sec:base_velocity_as_measurement_model}. To increase the level of difficulty, 
a ramp is introduced at the start of the trajectory, and the initial state is perturbed with a large error, resulting in the following 
nominal initial state:
\begin{equation}
\begin{aligned}
\bar{\mathcal{X}}_{0} = \begin{bmatrix}
\bar{\mathbf{R}}_{0} & \bar{\mathbf{v}}_{0} & \bar{\mathbf{p}}_{0} \\
\mathbf{0}_{3,1} & 1 & 0 \\
\mathbf{0}_{3,1} & 0 & 1
\end{bmatrix}, \;
\boldsymbol{x}_{0} = \{ \bar{\mathbf{R}}_{0}, \bar{\mathbf{v}}_{0}, \bar{\mathbf{p}}_{0} \},  \\
\bar{\mathbf{R}}_{0} = 
\begin{bmatrix}
0 & 0 & -1 \\
0 & 1 & 0 \\
1 & 0 &  0
\end{bmatrix}, 
\bar{\mathbf{v}}_{0} = \begin{bmatrix}
2.5 \\
2.0 \\
3.0
\end{bmatrix},
\bar{\mathbf{p}}_{0} = 
\begin{bmatrix}
-0.5 \\
0.0 \\
0.2
\end{bmatrix}.
\end{aligned}
\end{equation}
IMU measurements are corrupted with covariance $\mathbf{Q}_i = 10^{-2}\mathbf{I}_6$, while the base velocity measurement covariance
is estimated from a preliminary run and given by
\begin{equation}
	\mathbf{w}_{f,i} \sim \mathcal{N}(\mathbf{0}_{3,1}, \mathrm{diag}(7.84, 4.00, 25.00) \times 10^{-4}).
\end{equation}
Performance is assessed in terms of convergence time to reach predefined error thresholds of $0.05\, \mathrm{m/s}$ for base velocity and
$0.015\, \mathrm{rad}$ for the gravity direction. As shown in Fig.~\ref{fig:mujoco_simulation}, IterIEKF converges $2.5$ times faster than IEKF, 
while Iter$\mathrm{SO}(3)$-EKF achieves a similar convergence rate to IEKF. This demonstrates that the iterated filtering approach significantly improves convergence speed in the presence of large initial errors.

\subsection{Monte Carlo Simulation with synthetic data}
In this experiment, we assume that the IMU is measuring a constant linear acceleration and angular velocity.  Moreover,
we consider as measurement covariance the matrix $\mathrm{diag}(10^{-4}, 10^{-5}, 10^{-3})$ which is a typical 
covariance matrix for a simulated quadruped robot walking on a flat terrain 
(see Section \ref{sec:base_velocity_as_measurement_model}).
Once fixing the IMU measurements, we apply the prediction and update steps for $m$ times in a row. 
We repeat this procedure for $K$ realizations.
\begin{figure*}[t]
    \centering
    \subfloat[\protect\label{fig:exp_1_error_vel}]{
        \includegraphics[width=0.32\textwidth]{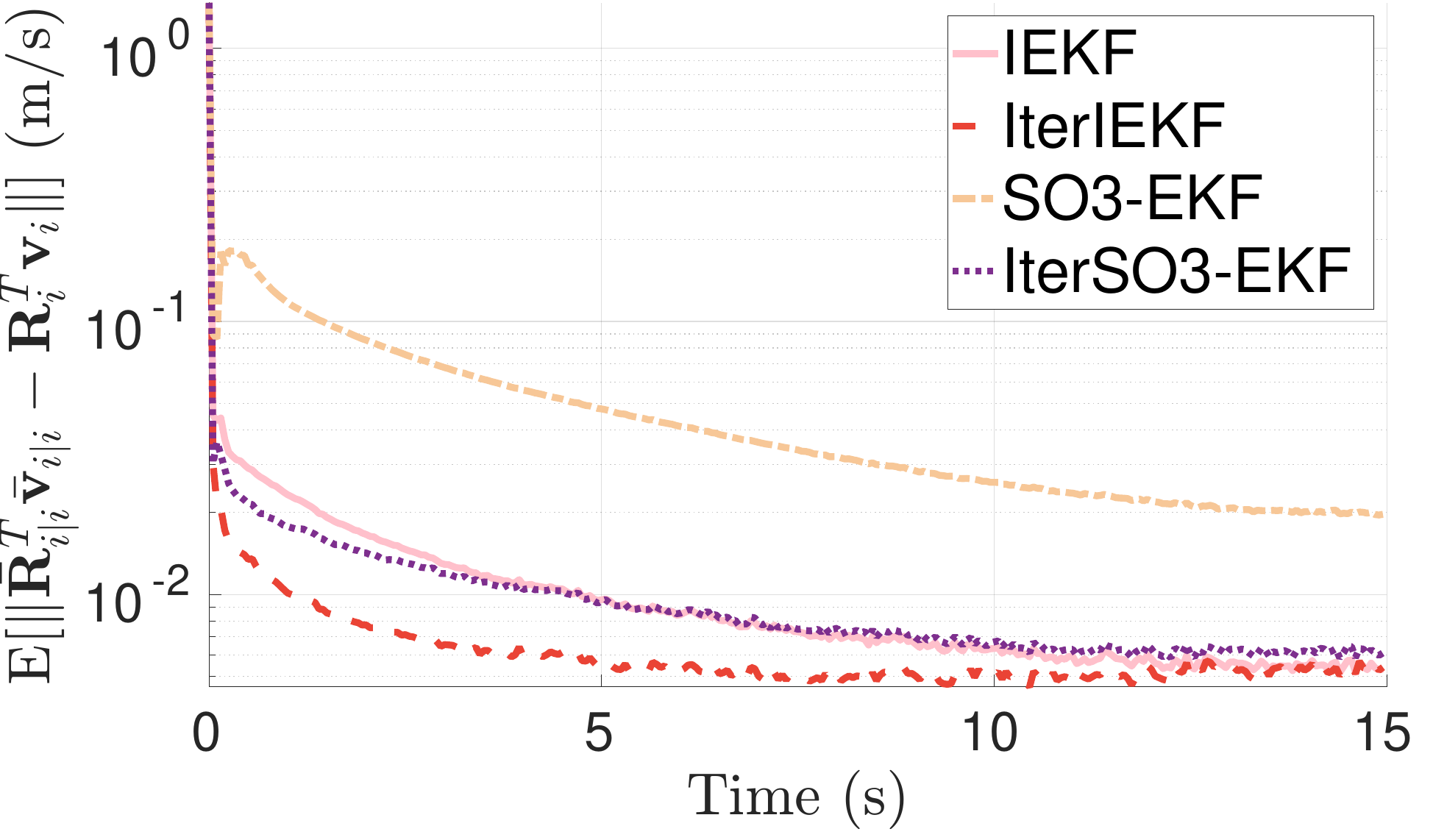}
    }
    \subfloat[\protect\label{fig:exp_1_error_u}]{
        \includegraphics[width=0.32\textwidth]{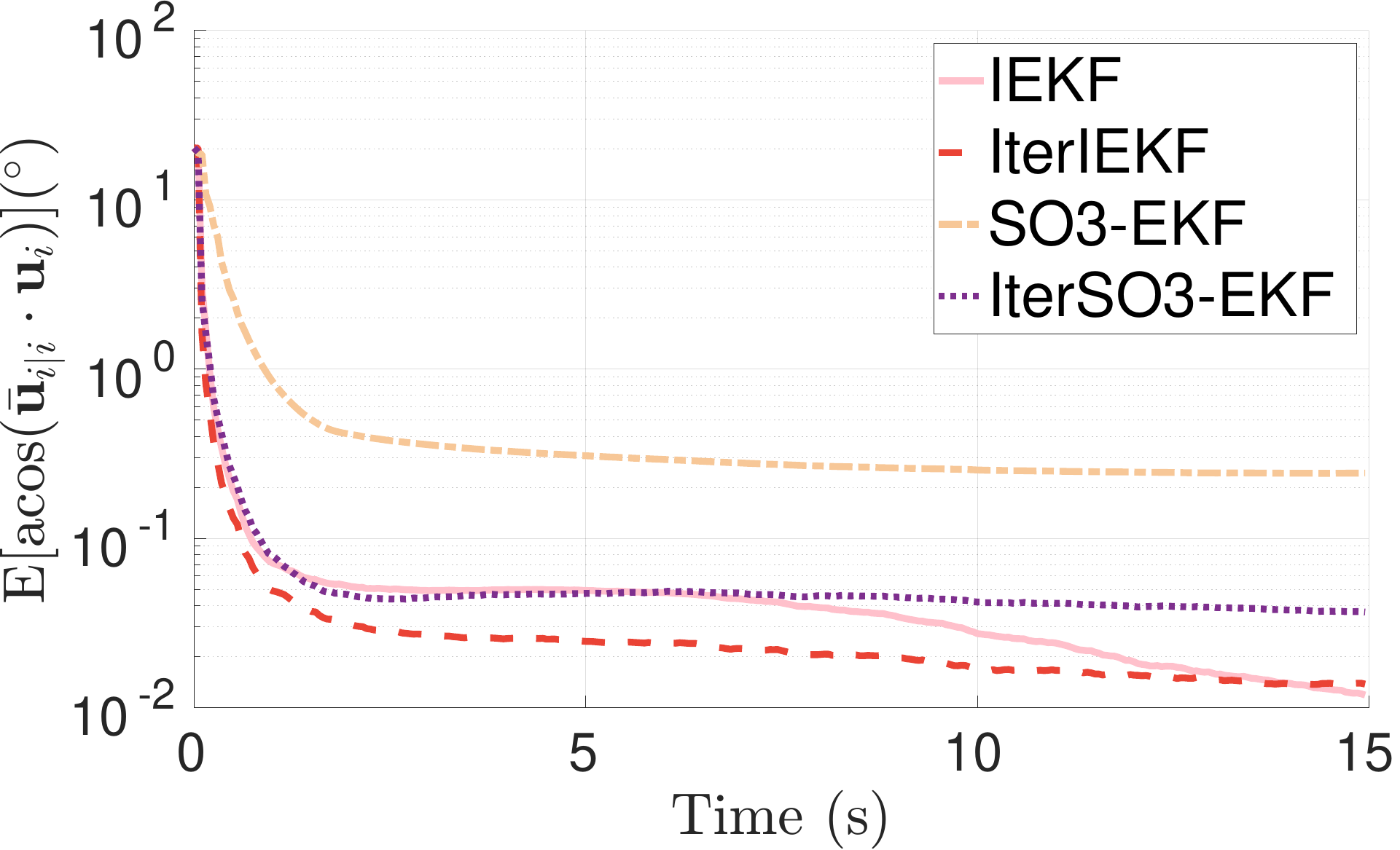}
    }
    \subfloat[\protect\label{fig:exp_1_nees}]{
        \includegraphics[width=0.32\textwidth]{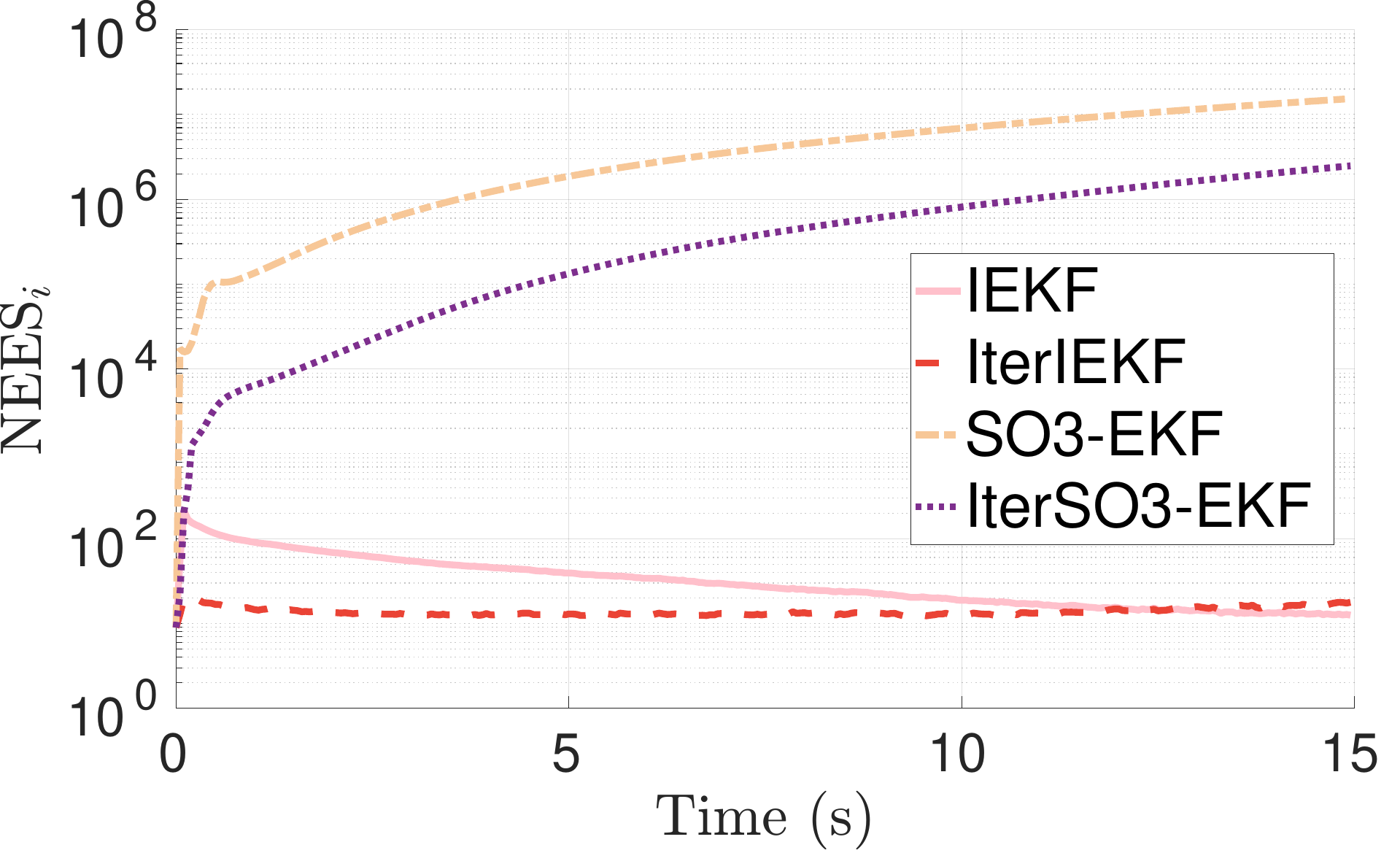}
    }
    \hfill
    \subfloat[\protect\label{fig:exp_1_iterations}]{
        \includegraphics[scale=0.22]{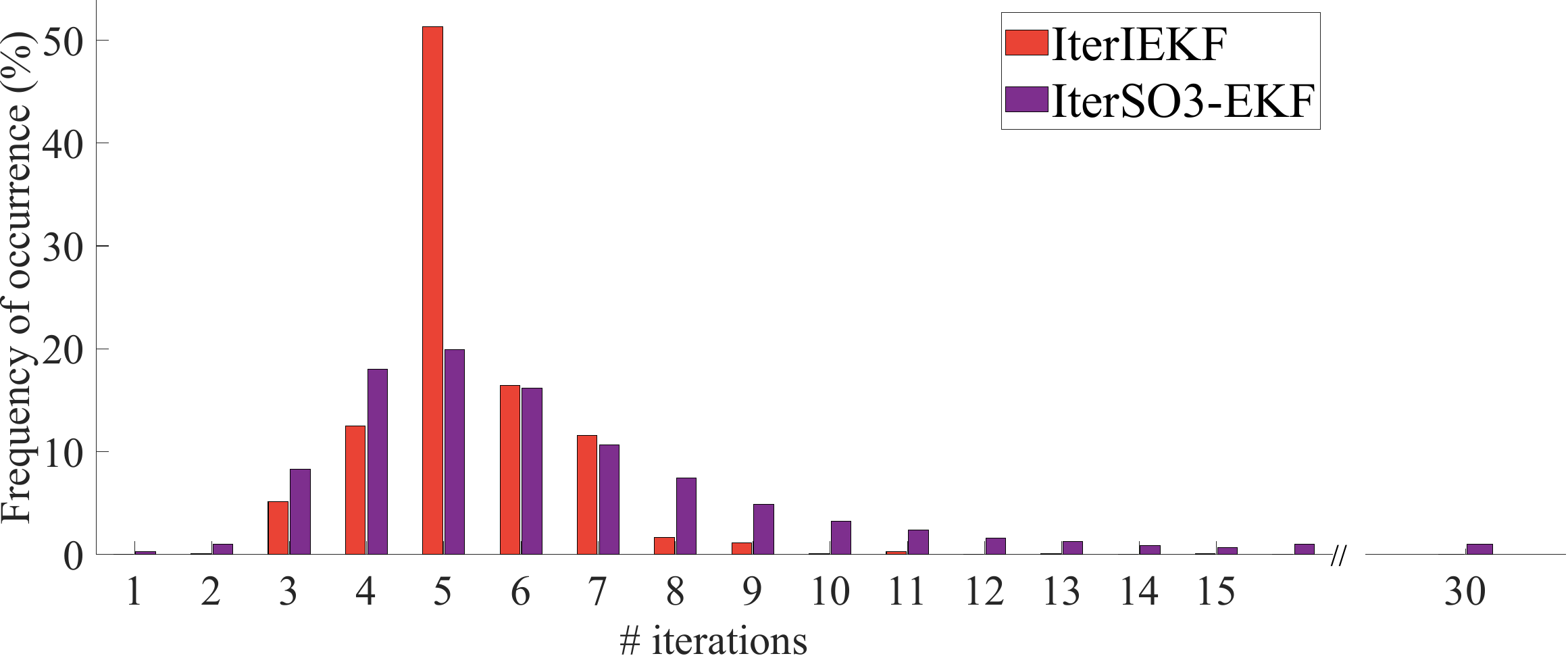}
    }
    \subfloat[\protect\label{fig:exp_1_realization}]{
        \includegraphics[scale=0.05]{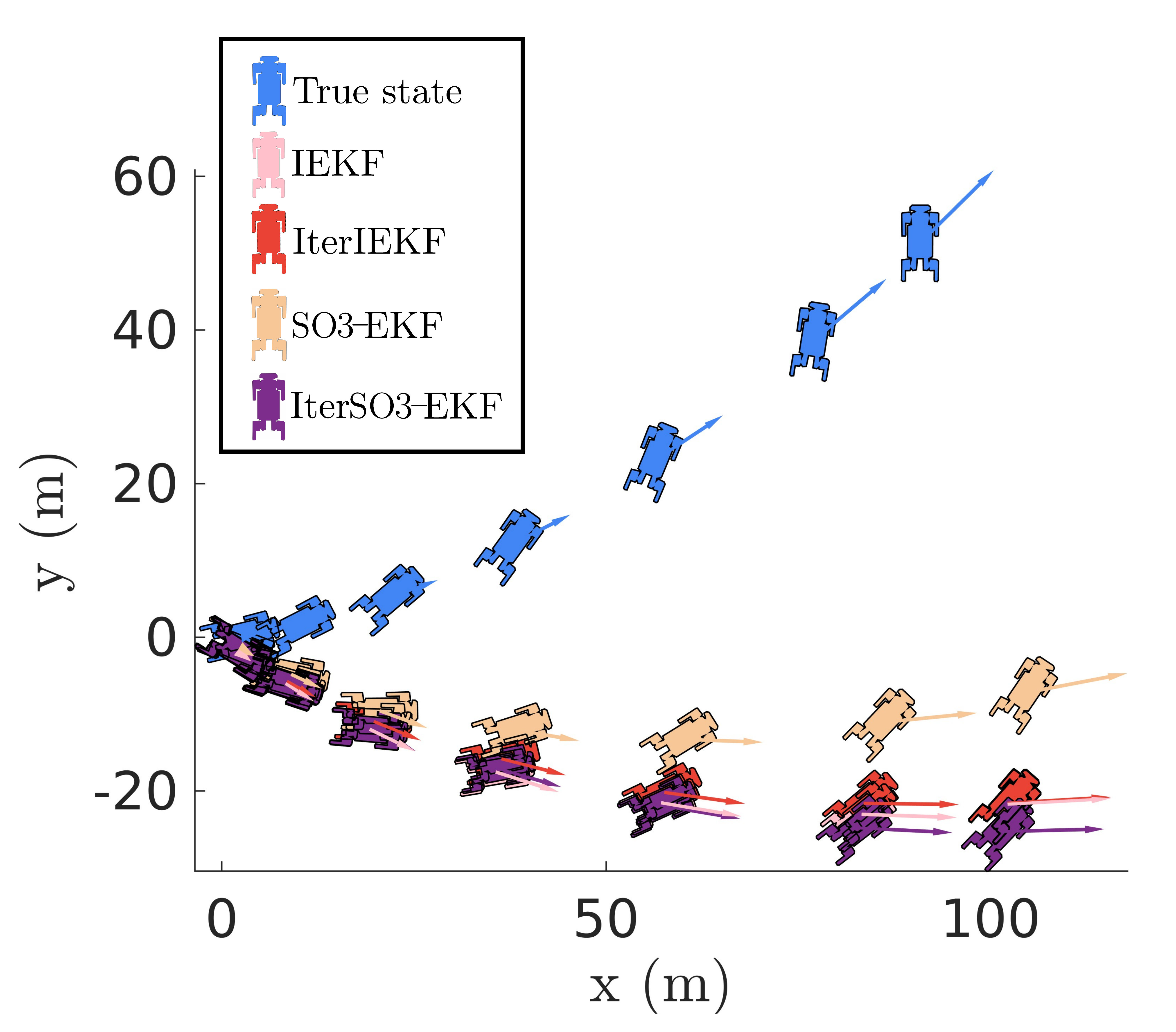}
    }
    \caption{Results of the numerical Monte Carlo experiment with synthetic data. 
        a) Average normalized error for the velocity in the base frame. 
        b) Average normalized error for the gravity direction.
        c) Average NEES. 
        d) Histogram for the number of iterations required for IterIEKF and Iter$\mathrm{SO}(3)$-EKF convergence.
        e) One realization of the Monte Carlo simulation. 
    }
    \label{fig:experiment_1}
\end{figure*}

We also assume white noise acceleration in all directions. For the motion to remain in the $x-y$ plane, we only assume gyroscope white noise in the $z$-direction (yaw angle).
All parameters are described as follows: 
\begin{equation}
    \label{eq:simulation_noise}
    \begin{aligned}
        \mathbf{P}_{0} = \text{diag}\left(\left(\frac{\pi}{12}\right)^{2} \mathbf{I}_3, (0.1)^2\mathbf{I}_3, (0.1)^2\mathbf{I}_3 \right), \\
        \mathbf{w}_{a,i} \sim \mathcal{N}(\mathbf{0}_{3,1}, 5\times 10^{-3}\mathbf{I}_{3}), \\
        \mathbf{w}_{g,i} \sim \mathcal{N}(\mathbf{0}_{3,1}, \text{diag}(0, 0, 3 \times  10^{-2})), \\
        \mathbf{w}_{f,i} \sim \mathcal{N}(\mathbf{0}_{3,1},  \mathrm{diag}(10^{-4}, 10^{-5}, 10^{-3}), \\
        \mathbf{a}_{I,i} = \begin{bmatrix}
            1 \\
            0 \\
            9.81
        \end{bmatrix}, 
        \boldsymbol{\omega}_{I,i} = 
        \begin{bmatrix} 0 \\ 0 \\ \frac{\pi}{30}  \end{bmatrix} , \\ 
        m =300, \Delta t = 0.05, K = 100 .
    \end{aligned}
\end{equation}
We present the main simulation results in Fig.~\ref{fig:experiment_1}. Figures~\ref{fig:exp_1_error_vel} and \ref{fig:exp_1_error_u} show the average norm of the velocity error and the average of gravity direction error, respectively. Both IterIEKF and Iter$\mathrm{SO}(3)$-EKF converged faster than their classical counterparts. 
More specifically, IterIEKF required $40 \%$ less time to achieve the same error as IEKF. To measure accuracy, we computed the MAE over the first $5 \, \mathrm{s}$ 
of the average error. IterIEKF reduced the error in the velocity by $27 \%$ and in the gravity direction by $10 \%$ with respect to IEKF.
Figure~\ref{fig:exp_1_nees} reports the average NEES, which, for this problem, should be close to $9$. Here, we observe the false observability issue that is common in the $\mathrm{SO}(3)$-EKF methodology: the filter becomes overconfident because the observability matrix for this problem is incorrectly modeled as fully observable. 
This issue is not resolved by IterSO(3)-EKF. In contrast, IterIEKF remains close to 9 from the start, indicating consistent uncertainty estimates. Fig.~\ref{fig:exp_1_iterations} compares the number of iterations required during the update step: both methods behave similarly, typically converging in about six iterations. Finally, Fig.~\ref{fig:exp_1_realization} shows a Monte Carlo realization, highlighting that position and yaw are unobservable, while body-frame velocity is accurately estimated.
\looseness=-1
\vspace{-0.2cm}

\subsection{Monte Carlo Simulation with a simulated quadruped robot on an irregular terrain}
In this section, we perform a Monte Carlo simulation of a quadruped robot walking along the forward direction on 
a simulated terrain of non-constant curvature. The terrain was chosen to introduce significant pitch variations. It is composed of three bumps, 
which present a variability of the robot's pitch angle. We started the state estimation 
methods when the robot is already in the middle of the first bump, as shown in Fig. \ref{fig:terain}. 
\begin{figure}[b]
		\centering
    \includegraphics[scale=0.15]{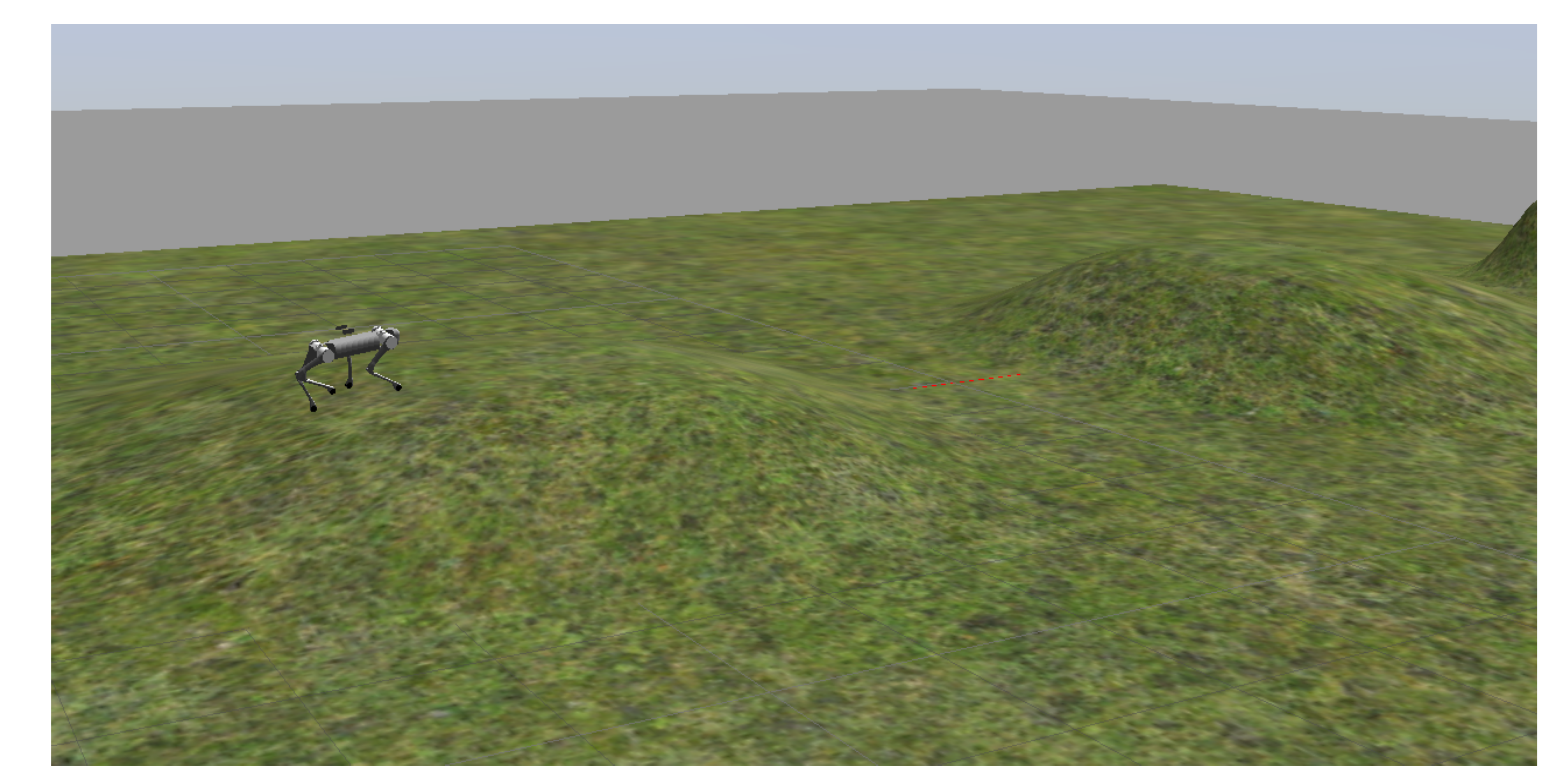}
    \caption{Simulated quadruped robot on a challenging terrain using Gazebo. 
    We start the estimation as soon the robot enters the first bump. The whole 
    trajectory can be seen \href{https://youtu.be/AyD8_3kJHw4}{here}.
}
		\label{fig:terain}
\end{figure}
\begin{figure*}[t!]
		\centering
    \subfloat[\protect\label{fig:exp_2_error_vel}]{
        \includegraphics[width=0.32\textwidth]{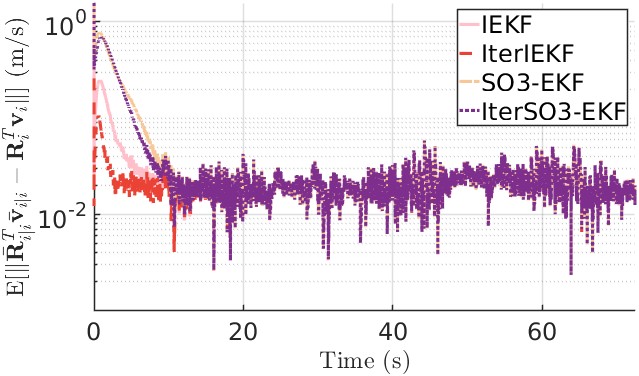}
    }
		\subfloat[\protect\label{fig:exp_2_error_u}]{
			\includegraphics[width=0.32\textwidth]{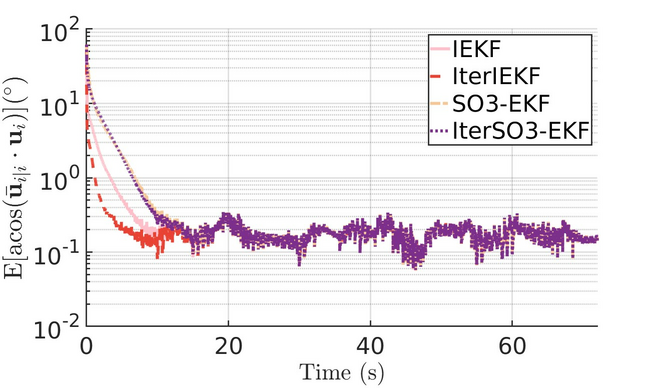}
    }
    \subfloat[\protect\label{fig:exp_2_nees}]{
        \includegraphics[width=0.32\textwidth]{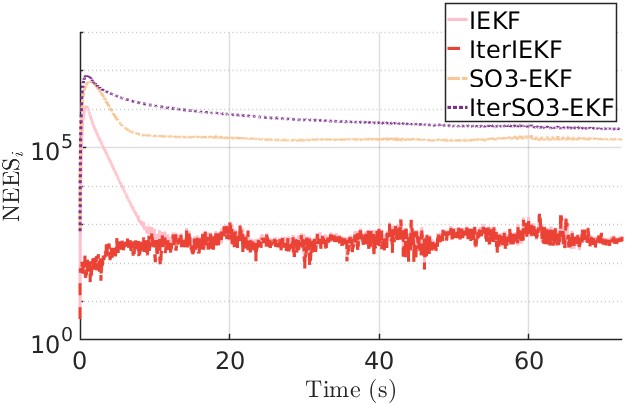}
    }
   \hfill
	 \subfloat[\protect\label{fig:exp_2_iterations}]{
		 \includegraphics[scale=0.055]{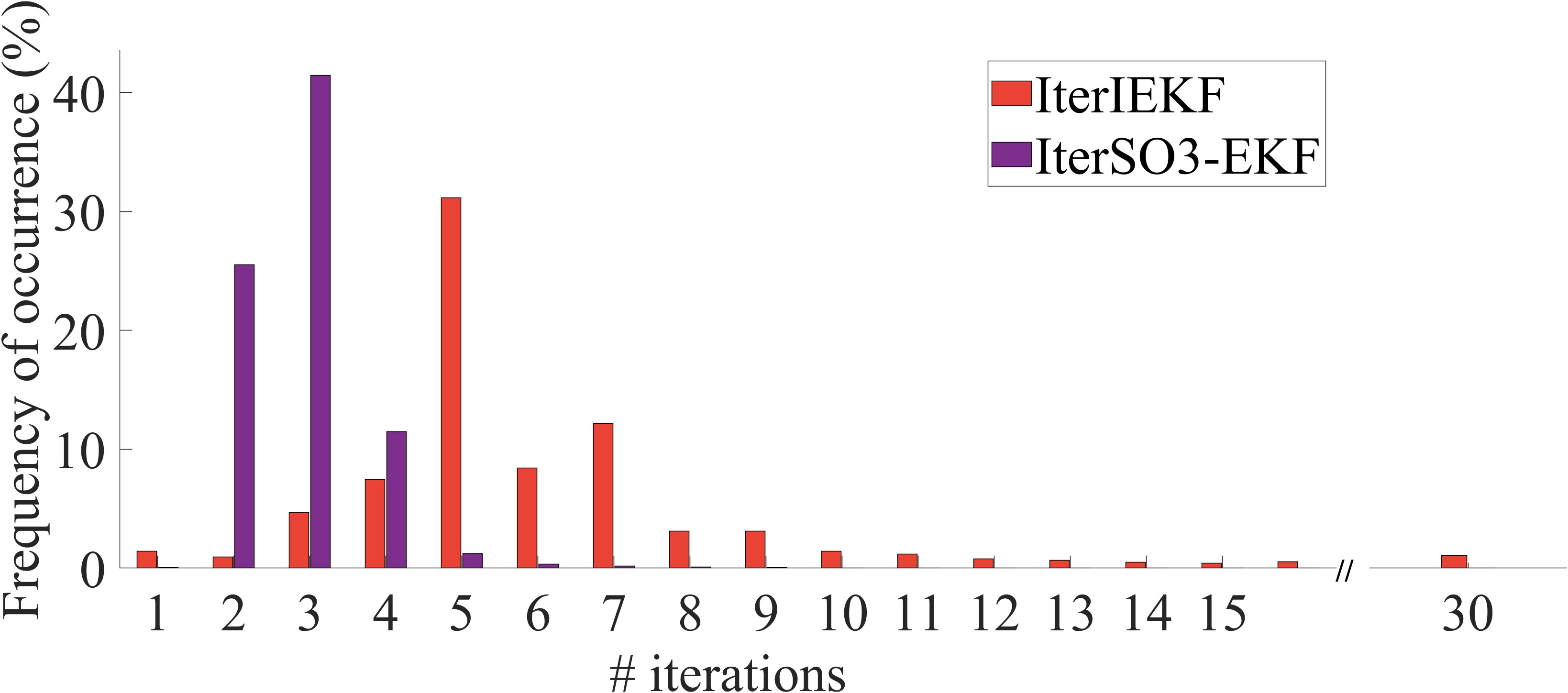}
    }
    \subfloat[\protect\label{fig:exp_2_pitch}]{
        \includegraphics[scale=0.175]{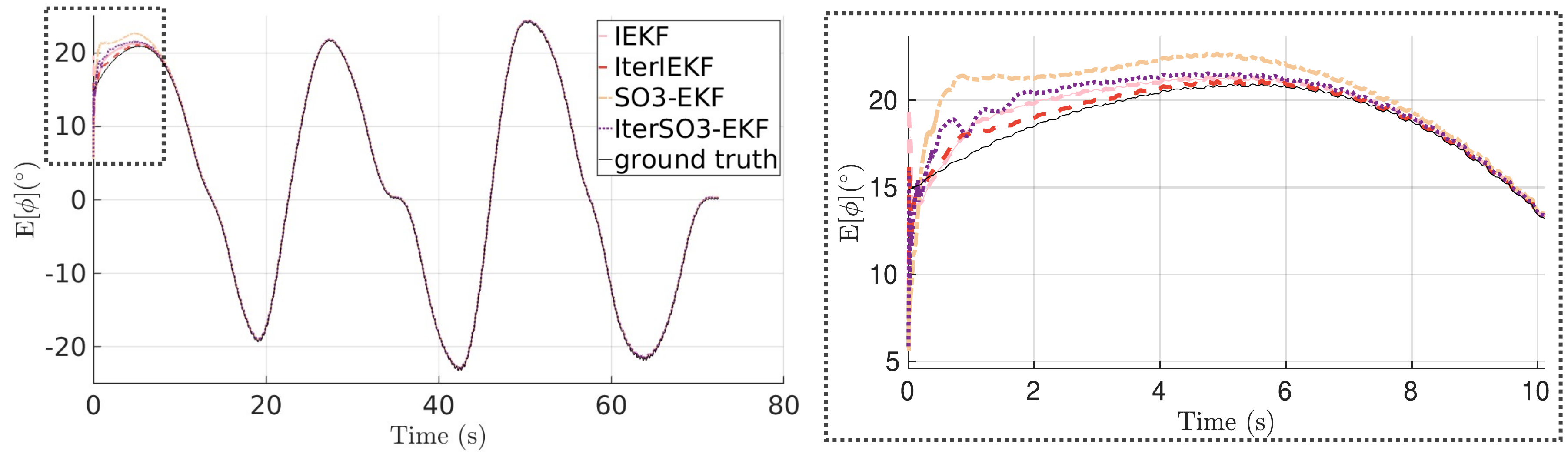}
    }
    \caption{Results of the numerical Monte Carlo experiment with a simulated quadruped robot on a challenging terrain. 
        a) Average normalized error for the velocity in the base frame. 
        b) Average normalized error for the gravity direction.
        c) Average NEES. 
        d) Histogram for the number of iterations required for IterIEKF and Iter$\mathrm{SO}(3)$-EKF convergence.
        e) Average estimated pitch angle for each method. 
    }
    \label{fig:experiment_2}
\end{figure*}
The initial covariance, IMU's measurements covariance, and feet velocity covariance were defined as follows:
\begin{equation}
 \label{eq:simulation_noise}
 \begin{aligned}
 \mathbf{P}_{0} = \text{diag}\left(\left(\frac{\pi}{4}\right)^2 \mathbf{I}_3, (1.0)^{2} \mathbf{I}_3, (0.1)^{2} \mathbf{I}_3 \right), \\
 \mathbf{w}_{a,i} \sim \mathcal{N}(\mathbf{0}_{3,1}, 4.0 \times 10^{-2}\mathbf{I}_{3}), \\
 \mathbf{w}_{g,i} \sim \mathcal{N}(\mathbf{0}_{3,1}, 2.0 \times 10^{-3} \mathbf{I}_3), \\
 \mathbf{w}_{f,i} \sim \mathcal{N}(\mathbf{0}_{3,1}, 
 \mathrm{diag}(0.502, 0.740, 5.149) \times 10^{-4}), \\ 
  K = 50.
 \end{aligned}
\end{equation}
The measurement covariance $\mathbf{w}_{f,i}$ was obtained by a first run of the simulation. 
We computed the sample covariance matrix based on the history of the velocities in the base frame 
obtained with~\eqref{eq:velocity_base_frame}. We assumed that a foot
is in contact if the $z$  component of its ground reaction force was above a fixed threshold of $30\, \mathrm{N}$. The calculation was corrupted by the gyroscope noise
$\mathbf{w}_{g,i}$ defined in~\eqref{eq:simulation_noise}.

The result is shown in Fig.~\ref{fig:experiment_2}. Figures~\ref{fig:exp_2_error_vel}
and~\ref{fig:exp_2_error_u} reports the average normalized error of the velocity (expressed in the base frame)
and of the gravity vector, respectively. We observe that IterIEKF required $60 \%$ less time to achieve the same error as IEKF.
In terms of accuracy, IterIEKF reduced the MAE over the first $5 \, \mathrm{s}$ in the velocity by $63 \%$ and in the gravity direction by $60 \%$ with respect to IEKF.
Furthermore, the other iterated approach, Iter$\mathrm{SO}(3)$-EKF, only improves upon $\mathrm{SO}(3)$-EKF
and performs worse than the IEKF. This behavior is due to its lack of consistency.
As indicated by the average NEES in Fig.~\ref{fig:exp_2_nees}, Iter$\mathrm{SO}(3)$-EKF exhibits the poorest
consistency among all methods, whereas IterIEKF shows the best. Figure~\ref{fig:exp_2_iterations} also reports the
number of iterations required for convergence over $50$ Monte Carlo realizations. We can see that IterIEKF
requires more iterations (most of the time $5$), while Iter$\mathrm{SO}(3)$-EKF converges most of the time 
in $3$ iterations.
Finally, Fig.~\ref{fig:exp_2_pitch} shows the average estimated pitch angle $\phi$. IterIEKF achieves a very
small error relative to the ground truth starting approximately $3\,\mathrm{s}$ after the simulation begins,
whereas the other methods require more time to reach a comparable accuracy.
\looseness=-1

\subsection{Real-world dataset evaluation}

In this section, we evaluate the performance of the proposed filter using a public dataset that contains a real-world experiment conducted with the Anymal D robot \cite{hutter2016anymal}. 
The dataset employed in this evaluation is 
\texttt{Pilatus-Hike2}\footnote{\url{https://grand-tour.leggedrobotics.com/dataset\#mission-19-on}}, 
which is part of the GrandTour dataset \cite{frey2026grandtour}. The recorded trajectory follows a zigzag trail commonly used by hikers to ascend 
Mount Pilatus in Switzerland. The robot first traverses the trail downhill, descending approximately $30\, \mathrm{m}$ in altitude, and then repeats 
the same path in the opposite direction, ascending approximately $22\, \mathrm{m}$. The duration of the trajectory is approximately $400\,\mathrm{s}$, 
and the total distance traveled by the robot is around $250\,\mathrm{m}$. Some parts of the trajectory are illustrated in 
Fig.~\ref{fig:samples_pilatus}. A key characteristic of this experiment is the significant variation in terrain inclination along the trajectory. 
For instance, at the beginning of the run, the robot descends a staircase with an inclination of approximately $30^\circ$, as shown in 
Fig.~\ref{fig:inclination}. Although the dataset includes exteroceptive information, we use only proprioceptive measurements. Specifically, 
we rely on the IMU, which provides accelerometer and gyroscope data; joint-state measurements, including joint positions and velocities; 
and foot force sensors that measure Ground Reaction Forces (GRFs). All of these sensors have an acquisition frequency of $400\, \mathrm{Hz}$. A robot foot is assumed to be in contact with the ground when the $z$-component of the GRF exceeds a fixed threshold of $60\,\mathrm{N}$.

\begin{figure}[h]

    \centering
    \subfloat[\protect\label{fig:samples_pilatus}]{
        \includegraphics[scale=0.49]{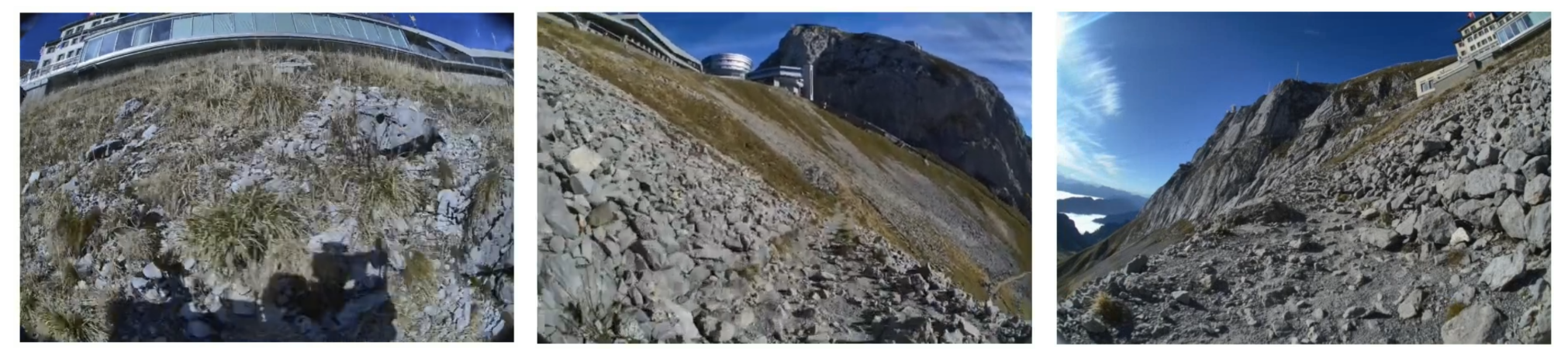}
    }
    \hfill
    \subfloat[\protect\label{fig:inclination}]{
        \includegraphics[scale=0.51]{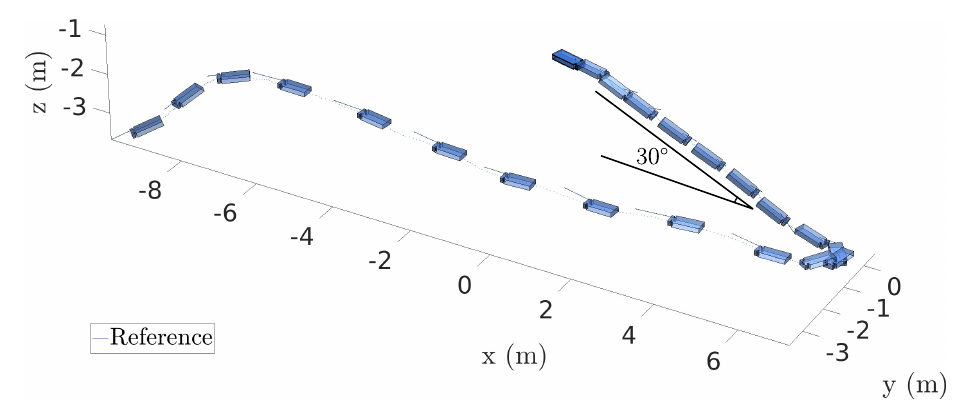}
    }
    \caption{a) Samples of the trajectory followed by the AnymalD robot in the \texttt{Pilatus-Hike2} dataset.
             b) Poses of the reference trajectory during its beginning, where it descends a staircase with an inclination of approx $30^\circ$.}
\end{figure}
\vspace{1cm}

The reference pose trajectory used for the evaluation (also shown in Fig.~\ref{fig:inclination}) was obtained with a $200\,\mathrm{Hz}$ 
CPT7 sensor\footnote{\url{https://novatel.com/products/gnss-inertial-navigation-systems/combined-systems/cpt7}} mounted on the robot. This sensor combines a high-end IMU with a dual-antenna GNSS receiver and provides accurate pose estimates 
in the global frame. To obtain reference velocity estimates, we applied a Savitzky--Golay filter to the reference pose trajectory. 
The filter parameters were chosen as a window length of $25$ and a polynomial order of $5$.  Since the estimation error is evaluated at every estimator timestamp, the reference pose must be interpolated to match the $400\,\mathrm{Hz}$ estimation rate.

The initial parameters for the filters are described as follows. First, we fixed the maximum number of iterations to $4$ for the iterated filters. 
This number was chosen to combine the results of the previous experiments presented in earlier sections with a small value that does not 
introduce any relevant increase in computational time. The other parameters are:
\begin{equation}
 \label{eq:simulation_noise_hike2}
 \begin{aligned}
 \mathbf{P}_{0} = \text{diag}\left(\left(\frac{\pi}{20}\right)^2 \mathbf{I}_3, (0.1)^{2} \mathbf{I}_3, (0.1)^{2} \mathbf{I}_3 \right), \\
 \mathbf{w}_{a,i} \sim \mathcal{N}(\mathbf{0}_{3,1}, \mathrm{diag}(1.0, 1.0, 2.25) ), \\
 \mathbf{w}_{g,i} \sim \mathcal{N}(\mathbf{0}_{3,1}, 2.5 \times 10^{-3}  \mathbf{I}_3), \\
 \mathbf{w}_{f,i} \sim \mathcal{N}(\mathbf{0}_{3,1}, 
 \mathrm{diag}(0.280, 0.009, 0.018)), \\ 
 \end{aligned}
\end{equation}
where the covariance of $\mathbf{w}_{f,i}$ was obtained in a first run by comparing the reference velocity with the nominal velocity obtained 
using~\eqref{eq:velocity_base_frame}. To reduce the amount of noise, we also employed a simple averaging between the current velocity and 
the one obtained in the previous step. Although this slightly reduced the uncertainty, the measurement noise remains significantly higher 
than the one obtained in simulation (see~\eqref{eq:simulation_noise}). We therefore expect this amount of noise to reduce the superconvergent 
behavior of the IterIEKF observed in the previous experiments. In Fig.~\ref{fig:real_measurement}, we compare the $z$-component of the velocity 
in the base frame obtained from the reference (CPT7 sensor) and the nominal velocity obtained using~\eqref{eq:velocity_base_frame}. 
Besides the higher noise level, we observe large intermittent peaks that occur during foot contact (see Fig.~\ref{fig:mujoco_contact}) and 
a change in variance over time, which suggests heteroscedasticity. These observations indicate that the measurement noise is non-Gaussian.
\begin{figure}[t!]
    \centering
    \includegraphics[width=0.35\textwidth]{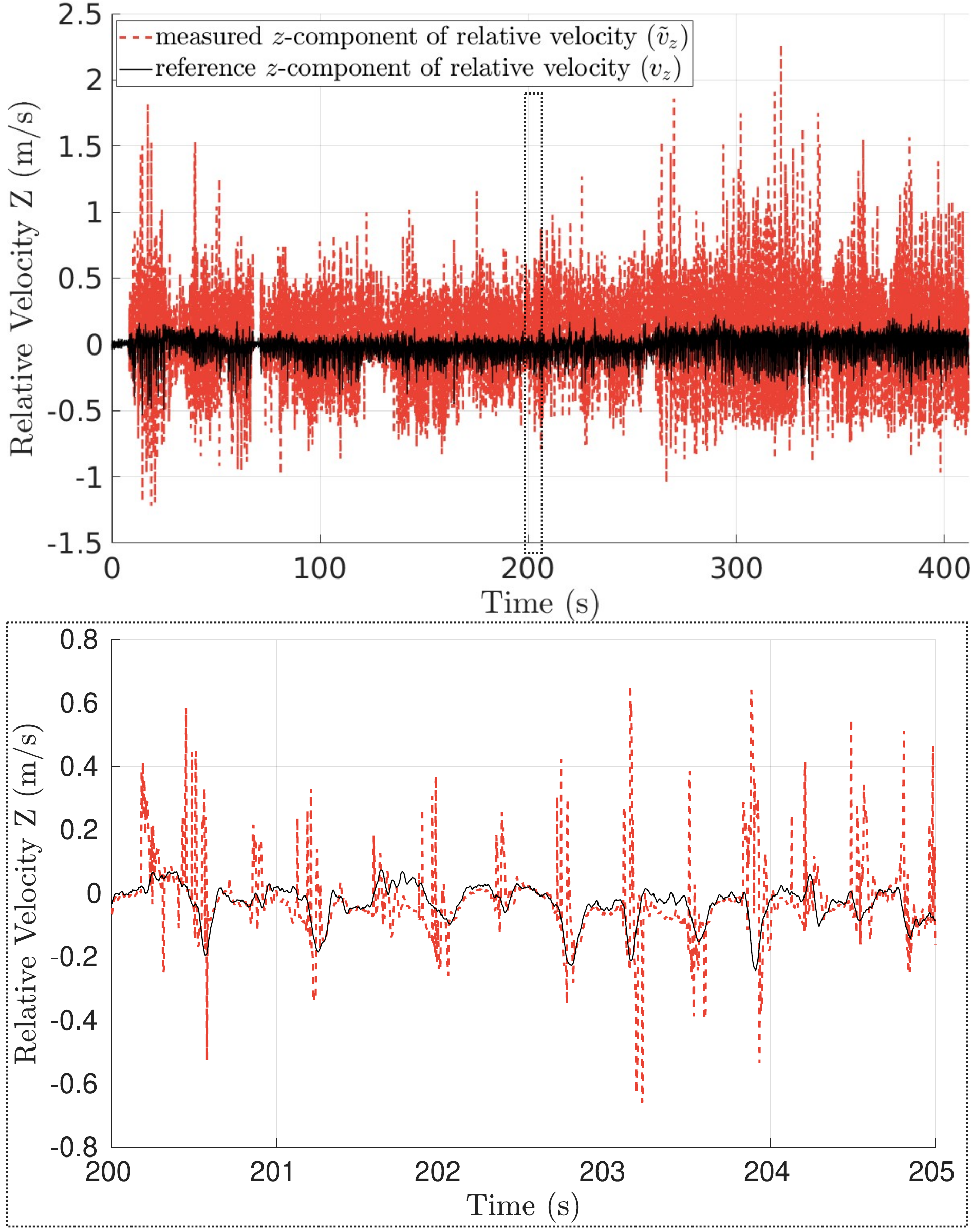}
    \caption{ Comparison of the $z$-component relative velocity in the base frame, obtained with the reference trajectory, and the nominal velocity obtained with~\ref{eq:velocity_base_frame}. The measurement noise is significantly higher than that obtained in the simulation.
    }
    \label{fig:real_measurement}
\end{figure}

We analyze the estimation performance of the four filters, IEKF, IterIEKF, $\mathrm{SO}(3)$-EKF, and Iter$\mathrm{SO}(3)$-EKF, under three distinct scenarios:

\begin{itemize}
    \item \textbf{Scenario 1:} Estimation of the complete trajectory using the parameters in~\eqref{eq:simulation_noise_hike2} and the same initial state as the reference trajectory.
    
    \item \textbf{Scenario 2:} Same parameters as in Scenario~1, but with a large error introduced in the initial pose.
    
    \item \textbf{Scenario 3:} Same configuration as Scenario~2, except that the nominal velocity is replaced by the reference velocity corrupted with additive white noise,
    $\mathbf{w}_{f,i} \sim \mathcal{N}(\mathbf{0}_{3 \times 1}, 10^{-4}\mathbf{I}_3)$.
\end{itemize}

The RMSE results for Scenario 1 are shown in Table~\ref{tab:scenario_x_1}, while the corresponding time evolution of the estimates is shown 
in Fig.~\ref{fig:result_scenario_i}. Overall, in this scenario, all filters performed similarly, with IterIEKF achieving slightly lower errors 
in all observables. From Fig.~\ref{fig:result_scenario_i}, we observe that the velocity estimates 
(Figs.~\ref{fig:vel_x_scenario_i}--\ref{fig:vel_z_scneraio_i}) are less accurate than the pitch 
(Fig.~\ref{fig:pitch_scenario_i}) and roll (Fig.~\ref{fig:roll_scenario_i}) angle estimates. More specifically, while the angular errors are 
less than half a degree, the velocity error is about $15\,\%$ of the average velocity of $0.625\,\mathrm{m/s}$. This may be due to the 
non-Gaussian nature of the measurement noise. The large intermittent peaks suggest the use of heavy-tailed distributions, such as Huber or 
Tukey losses, to model the measurement noise, as also suggested in \cite{santana2024}. However, identifying the best parameters for these 
distributions remains an active area of research. Despite the RMSE values, we emphasize that the filtered velocities significantly reduce 
the error compared to the non-filtered measurements, as can be seen by comparing Fig.~\ref{fig:real_measurement} and 
Fig.~\ref{fig:vel_z_scneraio_i}. Finally, Fig.~\ref{fig:traj_scenario_i} shows the estimated trajectories for all filters, where the invariant 
filters outperform the non-invariant ones. This is because the non-invariant filters are affected by the false observability issue, which 
causes them to be overconfident about the unobservable states \cite{barrau2015}. 
\begin{table}[h]
\centering
\caption{Comparison of RMSE (and standard deviation) of the observable state between different filters in Scenario 1.
}
\label{tab:scenario_x_1}
\scalebox{0.92}{
\begin{tabular}{lcccc}
\hline
 & RMSE$_{\text{vel}}(\text{m/s})$ & RMSE$_{\mathbf{u}}(^\circ)$ & RMSE$_{\phi}(^\circ)$ & RMSE$_{\theta}(^\circ)$ \\
\hline
IterIEKF
 & \textbf{0.097769} & \textbf{0.44342} & \textbf{0.34482} & \textbf{0.28518} \\
 & (0.0514) & (0.1336) & (0.1468) & (0.1347) \\

IEKF
 & \underline{0.097771} & \underline{0.44359} & \underline{0.34483} & \underline{0.28544} \\
 & (0.0514) & (0.1336) & (0.1469) & (0.1348) \\

$\mathrm{SO}_3$-EKF
 & 0.097819 & 0.46437 & 0.36711 & 0.29098 \\
 & (0.0513) & (0.1326) & (0.1496) & (0.1386) \\

Iter$\mathrm{SO}_3$-EKF
 & 0.097823 & 0.46307 & 0.36599 & 0.29029 \\
 & (0.0513) & (0.1325) & (0.1495) & (0.1383) \\
\hline
\multicolumn{5}{l}{\footnotesize 
Smallest RMSE is marked in bold, while the second smallest is highlighted.}
\end{tabular}
}
\end{table}
\begin{table}[h]
\centering
\caption{Comparison of MAE over the first $5\, \mathrm{s}$ of the observable state for different filters in Scenario 2. 
Percentages indicate increase relative to the lowest error in each column.}
\label{tab:scenario_x_2}
\scalebox{0.92}{
\begin{tabular}{lcccc}
\hline
 & MAE$_{\text{vel}}(\text{m/s})$ & MAE$_{\mathbf{u}}(^\circ)$ & MAE$_{\phi}(^\circ)$ & MAE$_{\theta}(^\circ)$ \\
\hline
IterIEKF
 & \underline{0.088708} & \underline{1.2244} & \textbf{0.68245} & \underline{1.0423} \\
 & (+0.43\%) & (+2.93\%) & (0\%) & (+13.22\%) \\

IEKF
 & \textbf{0.088332} & 1.2884 & \underline{0.76842} & 1.1141 \\
 & (0\%) & (+8.31\%) & (+12.59\%) & (+20.99\%) \\

$\mathrm{SO}_3$-EKF
 & 0.089225 & 1.4766 & 0.94817 & 1.198 \\
 & (+1.01\%) & (+24.12\%) & (+38.90\%) & (+30.15\%) \\

Iter$\mathrm{SO}_3$-EKF
 & 0.088866 & \textbf{1.1896} & 0.86155 & \textbf{0.9206} \\
 & (+0.60\%) & (0\%) & (+26.26\%) & (0\%) \\
\hline
\multicolumn{5}{l}{\footnotesize 
Smallest MAE is marked in bold, while the second smallest is underlined.}
\end{tabular}
}
\end{table}
\begin{table}[h]
\centering
\caption{Comparison of MAE over the first $5 \, \mathrm{s}$ of the observable state for different filters in Scenario 3. Percentages indicate increase relative to the lowest error in each column.}
\label{tab:scenario_x_3}
\scalebox{0.92}{
\begin{tabular}{lcccc}
\hline
 & RMSE$_{\text{vel}}(\text{m/s})$ & RMSE$_{\mathbf{u}}(^\circ)$ & RMSE$_{\phi}(^\circ)$ & RMSE$_{\theta}(^\circ)$ \\
\hline
IterIEKF
 & \textbf{0.015924} & \textbf{0.85114} & \textbf{0.56495} & \textbf{0.62606} \\
 & (0\%) & (0\%) & (0\%) & (0\%) \\

IEKF
 & \underline{0.016019} & \underline{0.9857} & \underline{0.7303} & \underline{0.65386} \\
 & (+0.60\%) & (+15.81\%) & (+29.27\%) & (+4.44\%) \\

$\mathrm{SO}_3$-EKF
 & 0.020007 & 6.954 & 2.8188 & 6.7823 \\
 & (+25.64\%) & (+716.89\%) & (+399.00\%) & (+983.00\%) \\

Iter$\mathrm{SO}_3$-EKF
 & 0.016519 & 1.7165 & 1.5552 & 0.70186 \\
 & (+3.74\%) & (+101.67\%) & (+175.33\%) & (+12.11\%) \\
\hline
\multicolumn{5}{l}{\footnotesize 
Smallest MAE is marked in bold, while the second smallest is underlined.}
\end{tabular}
}
\end{table}
\begin{figure*}[t!]
		\centering
        \includegraphics[scale=0.375]{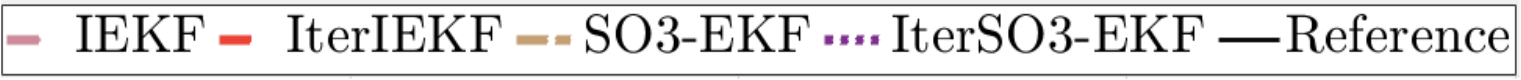}
    \subfloat[\protect\label{fig:vel_x_scenario_i}]{
        \includegraphics[scale=0.17]{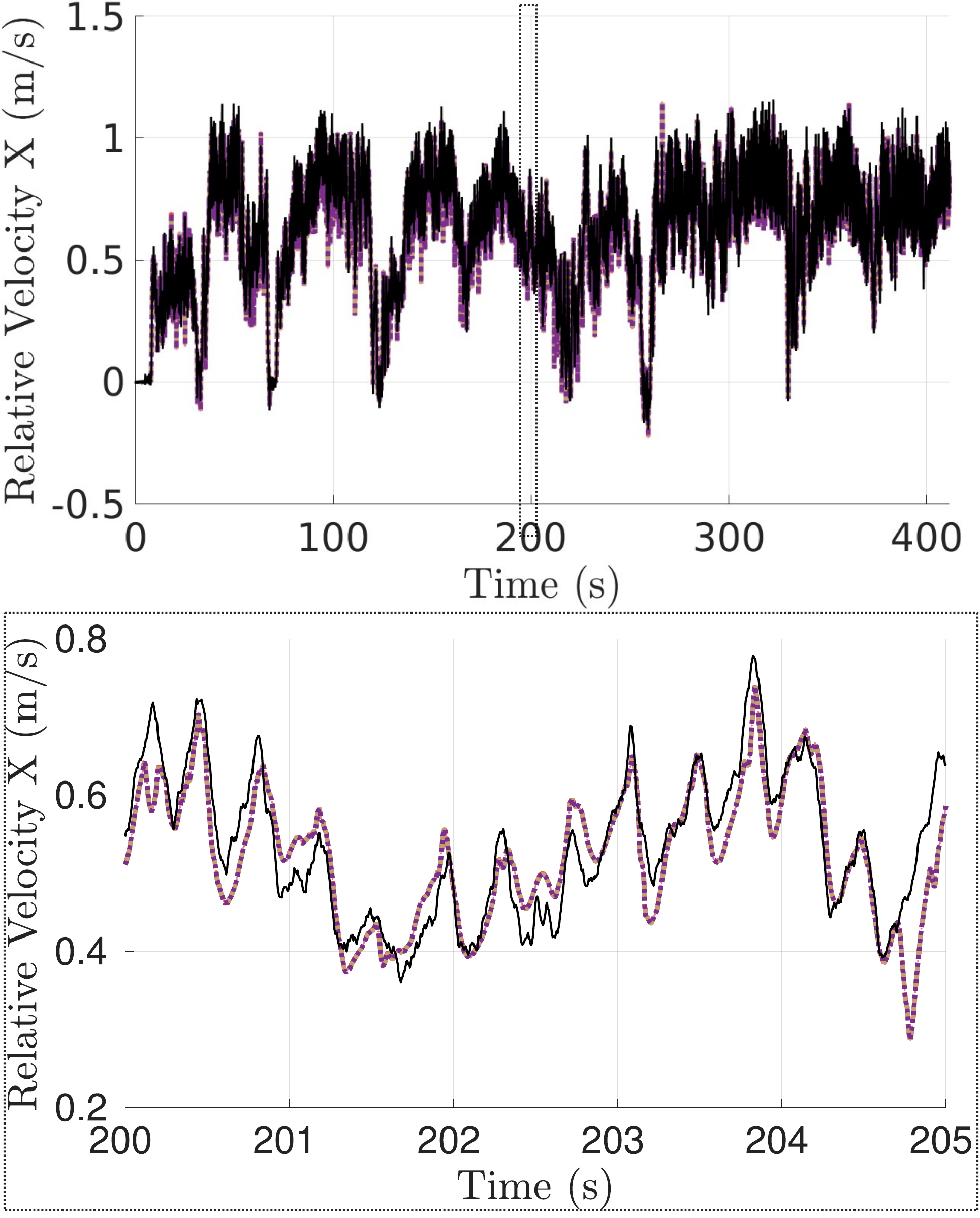}
    }
    \subfloat[\protect\label{fig:vel_y_scenario_i}]{
        \includegraphics[scale=0.17]{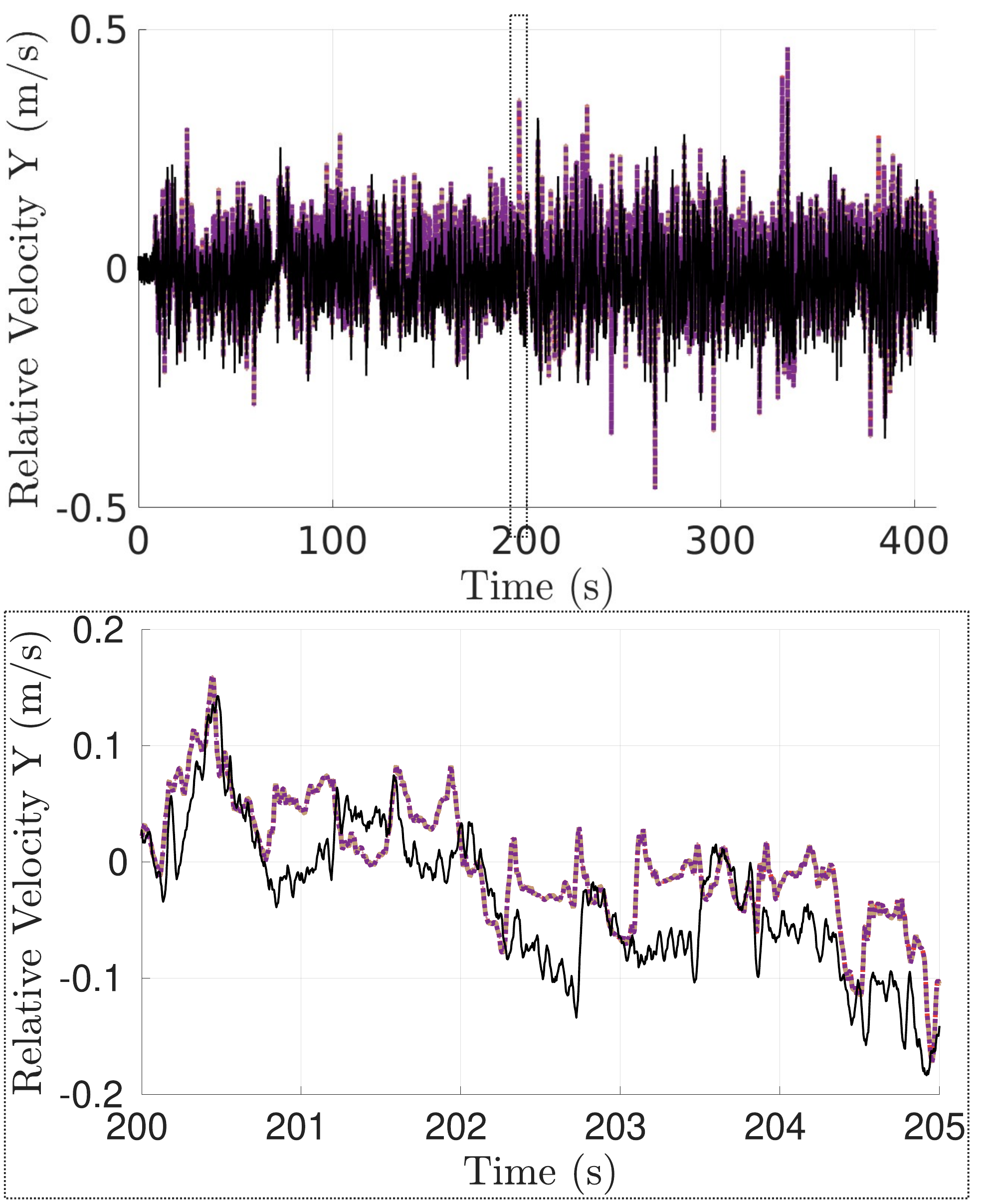}
    }
    \subfloat[\protect\label{fig:vel_z_scneraio_i}]{
        \includegraphics[scale=0.17]{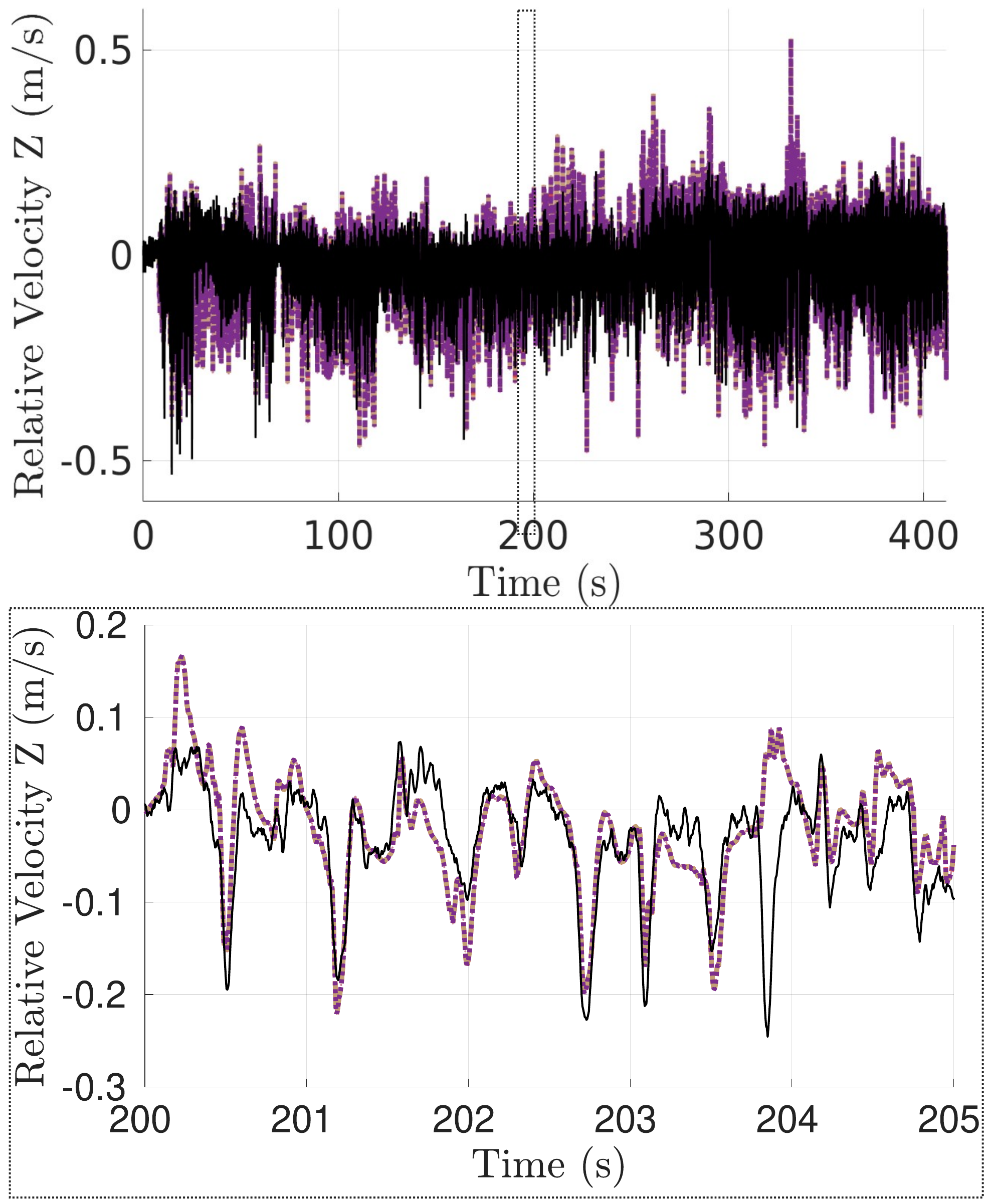}
    }
   \hfill
   \subfloat[\protect\label{fig:pitch_scenario_i}]{
       \includegraphics[scale=0.175]{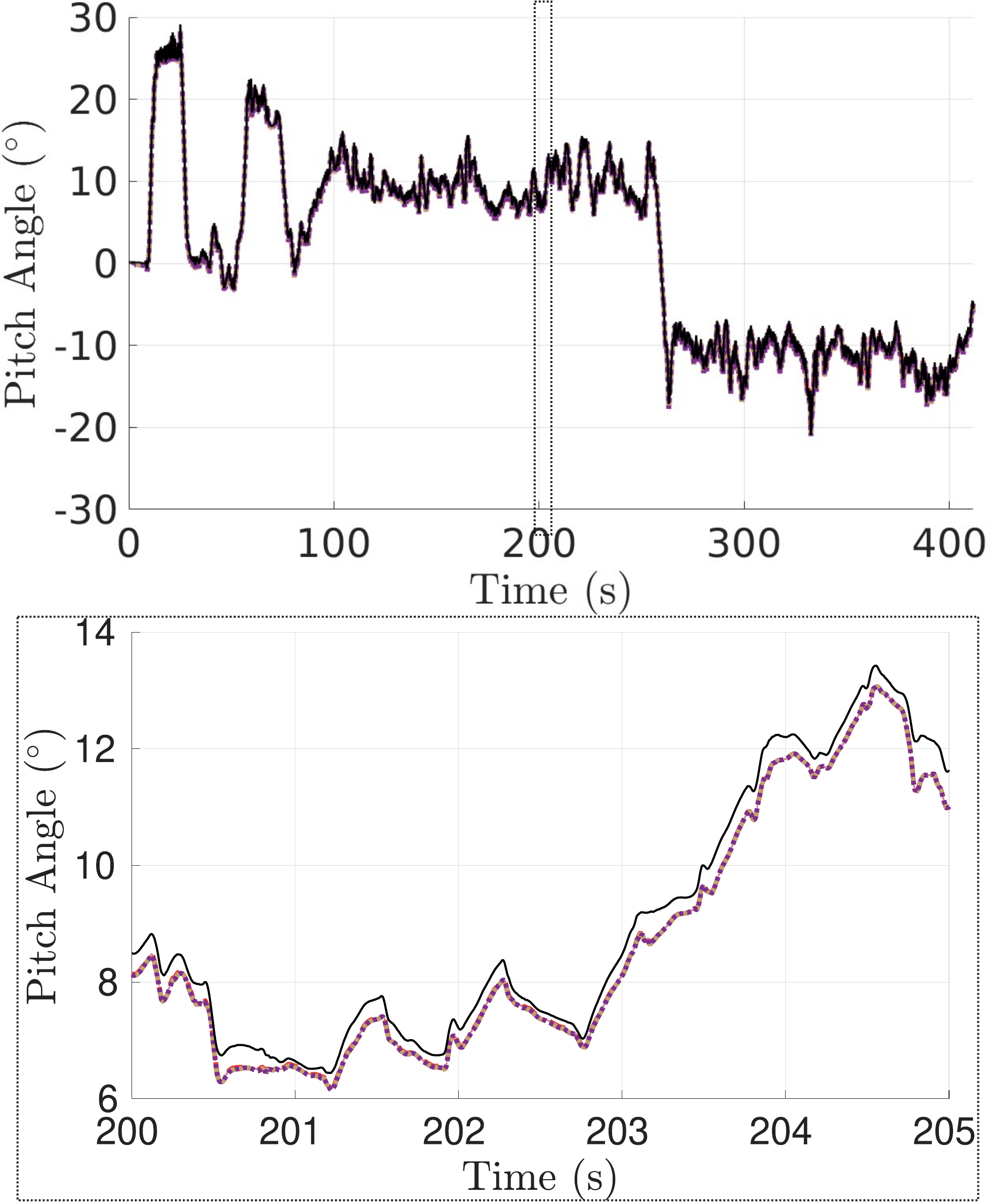}
    }
    \subfloat[\protect\label{fig:roll_scenario_i}]{
        \includegraphics[scale=0.175]{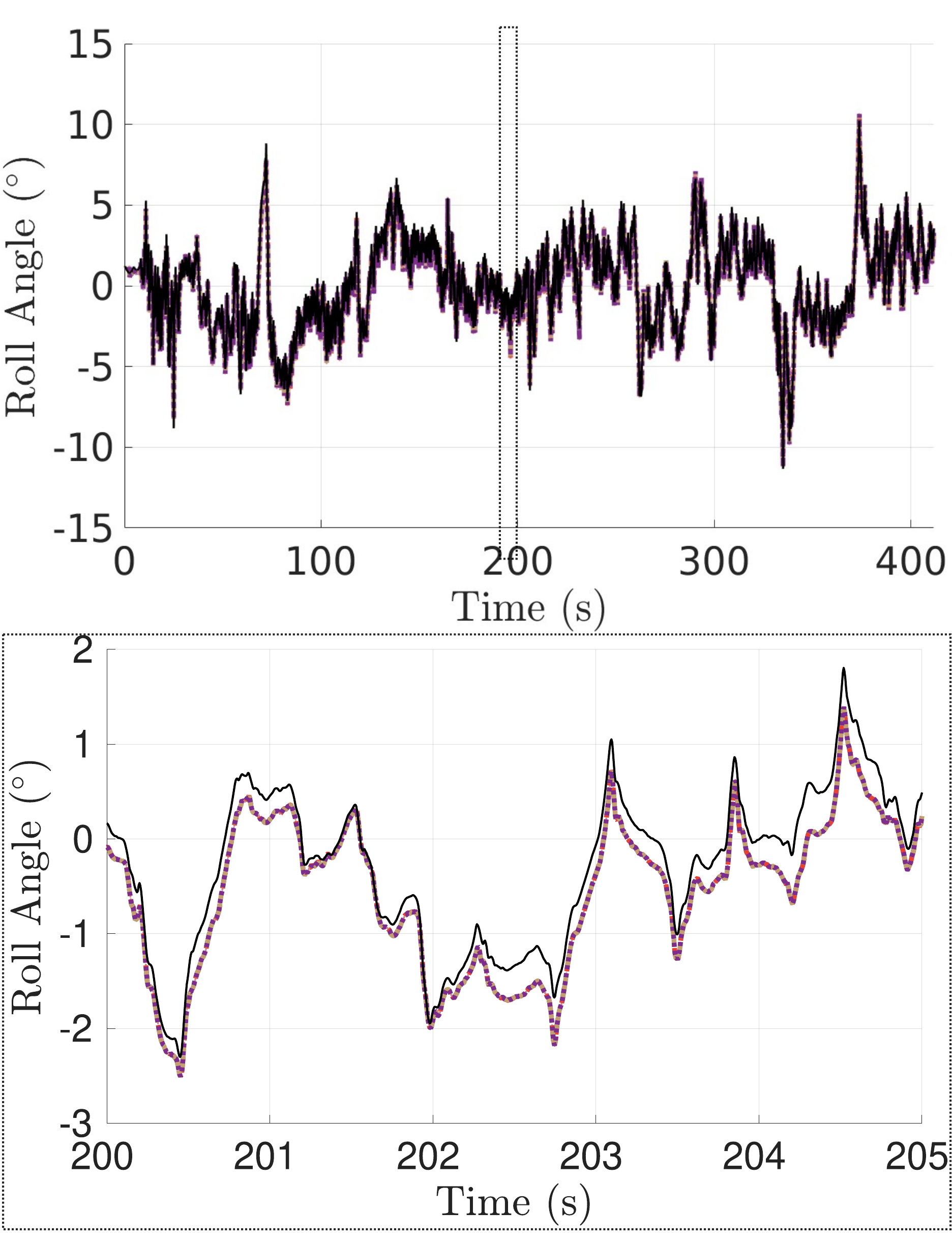}
    }
    \subfloat[\protect\label{fig:traj_scenario_i}]{
        \includegraphics[scale=0.23]{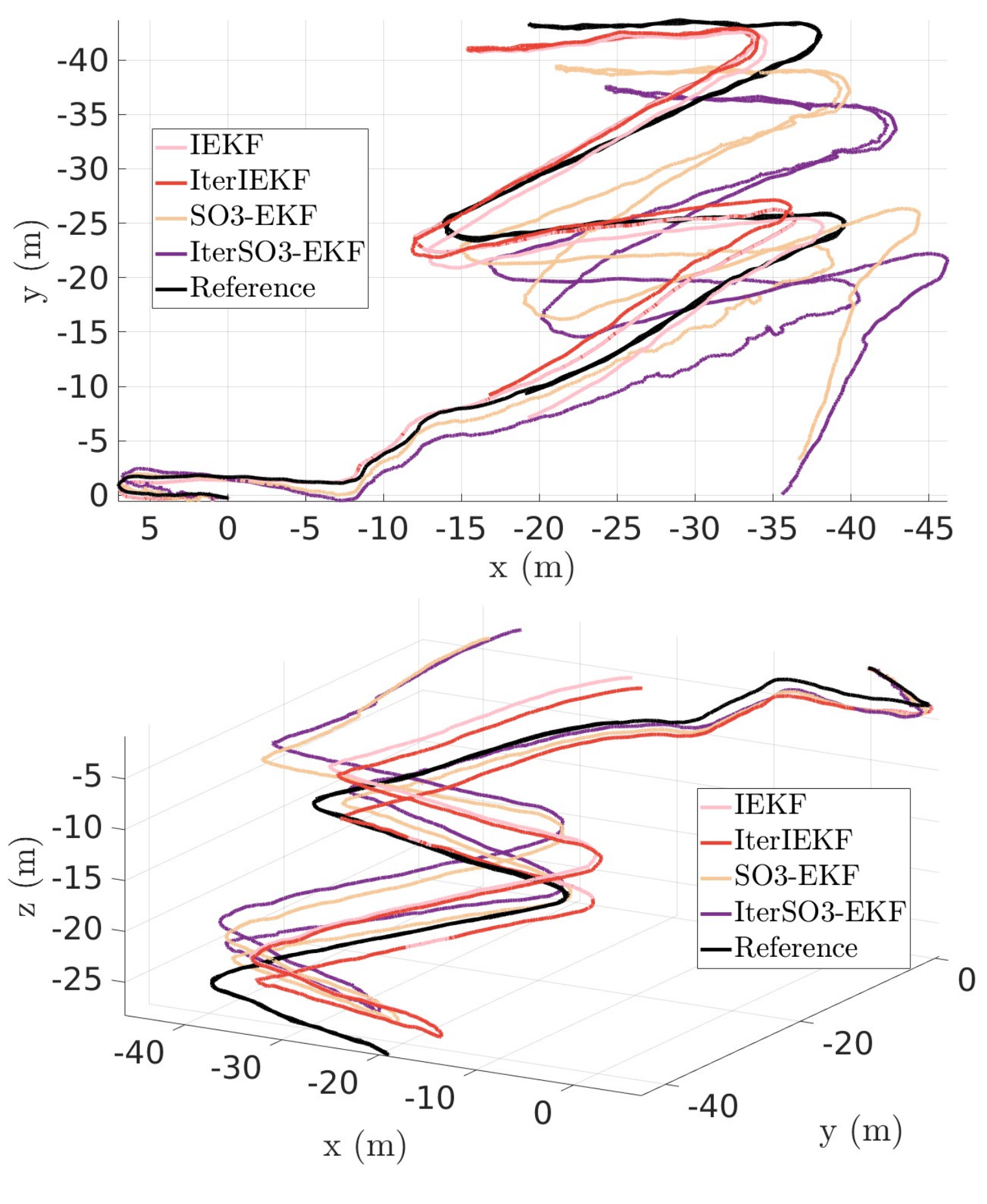}
    }
    \caption{
        Estimates obtained with different filters on the \texttt{Pilatus-Hike2} dataset under Scenario 1 conditions. 
            a) $x$-component,  b) $y$-component, c) $z$-component of the relative velocity $\bar{\mathbf{R}}_{i|i} \bar{\mathbf{v}}_{i|i}$. 
            d) Estimate of the pitch angle $\bar{\phi}_{i|i}$. e) Estimate of the roll angle $\bar{\theta }_{i|i}$. f) Estimate of the trajectory 
            $\bar{\mathbf{p}}_{i|i}$.
    }
    \label{fig:result_scenario_i}
\end{figure*}

In Table~\ref{tab:scenario_x_2} and Fig.~\ref{fig:result_scenario_ii}, we present the MAE results over the first $5\, \mathrm{s}$ and the estimates for Scenario~2. 
In this scenario, the MAE was adopted due to the large error introduced in the initial state. The estimation begins $17\,\mathrm{s}$ after the start of the trajectory 
in order to analyze the transient behavior of the filters. For the velocity components 
(Figs.~\ref{fig:vel_x_scenario_ii}--\ref{fig:vel_z_scenario_ii}), all filters performed similarly, with IEKF achieving the smallest error. 
Iter$\mathrm{SO}(3)$-EKF performed slightly better than $\mathrm{SO}(3)$-EKF, while IterIEKF performed slightly worse than IEKF, 
possibly due to the high level of measurement noise. For the pitch angle (Fig.~\ref{fig:pitch_scenario_ii}), IterIEKF achieved the smallest 
error, $12 \, \%$ less than IEKF. It also converged faster than all filters. Similarly, Iter$\mathrm{SO}(3)$-EKF converged faster than $\mathrm{SO}(3)$-EKF. 
For the roll angle (Fig.~\ref{fig:roll_scenario_ii}), Iter$\mathrm{SO}(3)$-EKF performed best, with $20 \, \%$ less error than IEKF.
For the gravity direction in Fig.~\ref{fig:gravity_dir_scneario_ii} we also observe a faster convergence rate of IterIEKF in relation to IEKF.

Finally, Fig.~\ref{fig:result_scenario_iii} and Table~\ref{tab:scenario_x_3} present the results for Scenario~3. This scenario is similar 
to Scenario~2, but the measurement function was artificially replaced by the reference measurement corrupted with white noise of small 
covariance, $\mathbf{w}_{f,i} \sim \mathcal{N}(\mathbf{0}_{3,1}, 10^{-4}\mathbf{I}_3)$. The MAE values in the velocities 
decrease considerably for all filters, 
highlighting the impact of non-Gaussian noise in the previous scenarios. For the velocity components 
(Figs.~\ref{fig:vel_x_scenario_iii}--\ref{fig:vel_z_scenario_iii}), all filters, but $\mathrm{SO}(3)$-EKF, continue to perform similarly, with IterIEKF achieving the 
lowest error. For the remaining states, the iterated invariant filter shows clear improvements. In particular, for the pitch angle and gravity direction 
(Figs.~\ref{fig:pitch_scenario_iii} and \ref{fig:gravity_dir_scneario_iii}), IterIEKF demonstrates the superconvergent behavior observed in 
previous experiments. This experiment confirms that when the measurement noise is truly white and has small covariance, IterIEKF outperforms the classical filters.

\begin{figure*}[t!]
		\centering
    \subfloat[\protect\label{fig:vel_x_scenario_ii}]{
        \includegraphics[scale=0.17]{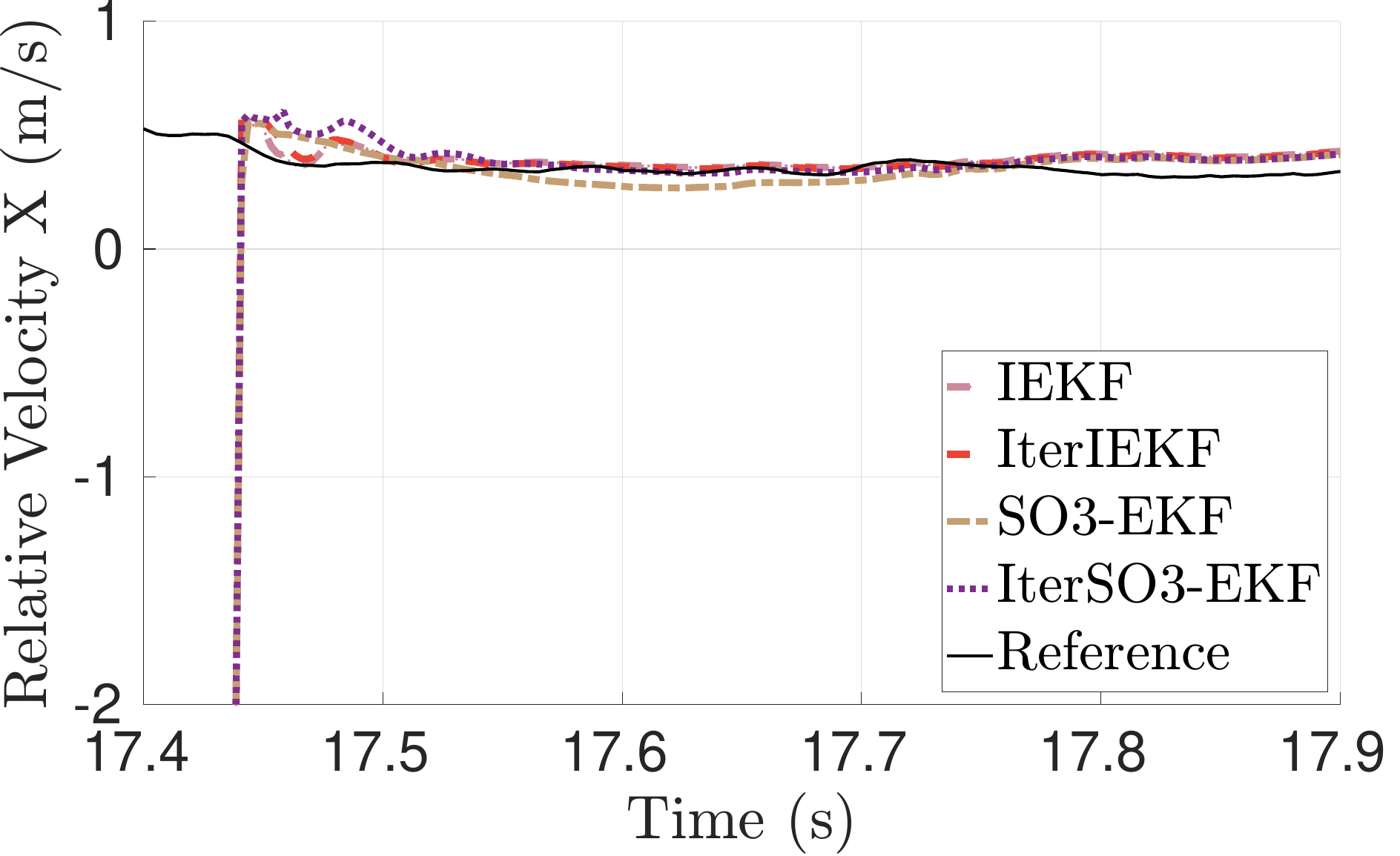}
    }
    \subfloat[\protect\label{fig:vel_y_scenario_ii}]{
        \includegraphics[scale=0.17]{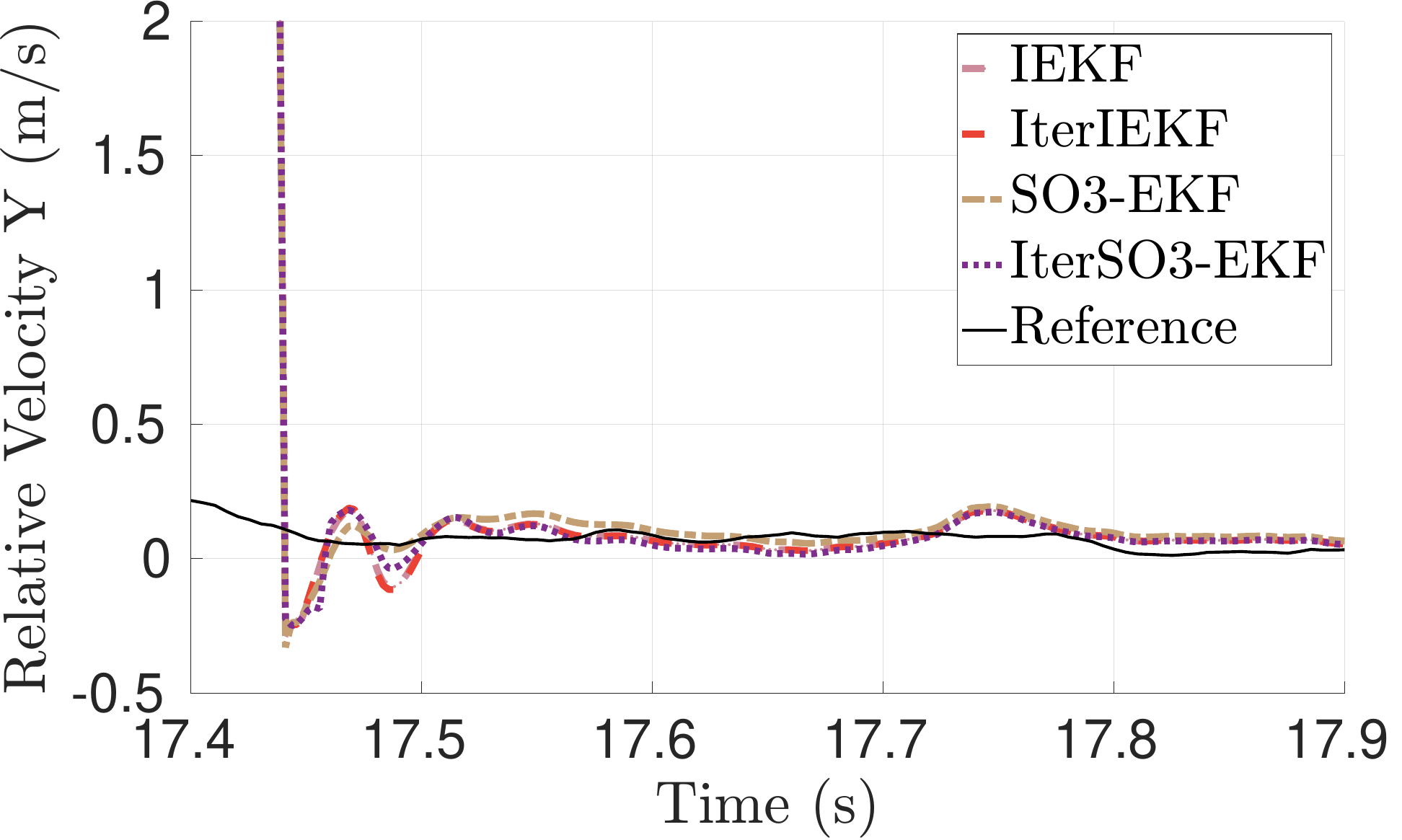}
    }
    \subfloat[\protect\label{fig:vel_z_scenario_ii}]{
        \includegraphics[scale=0.17]{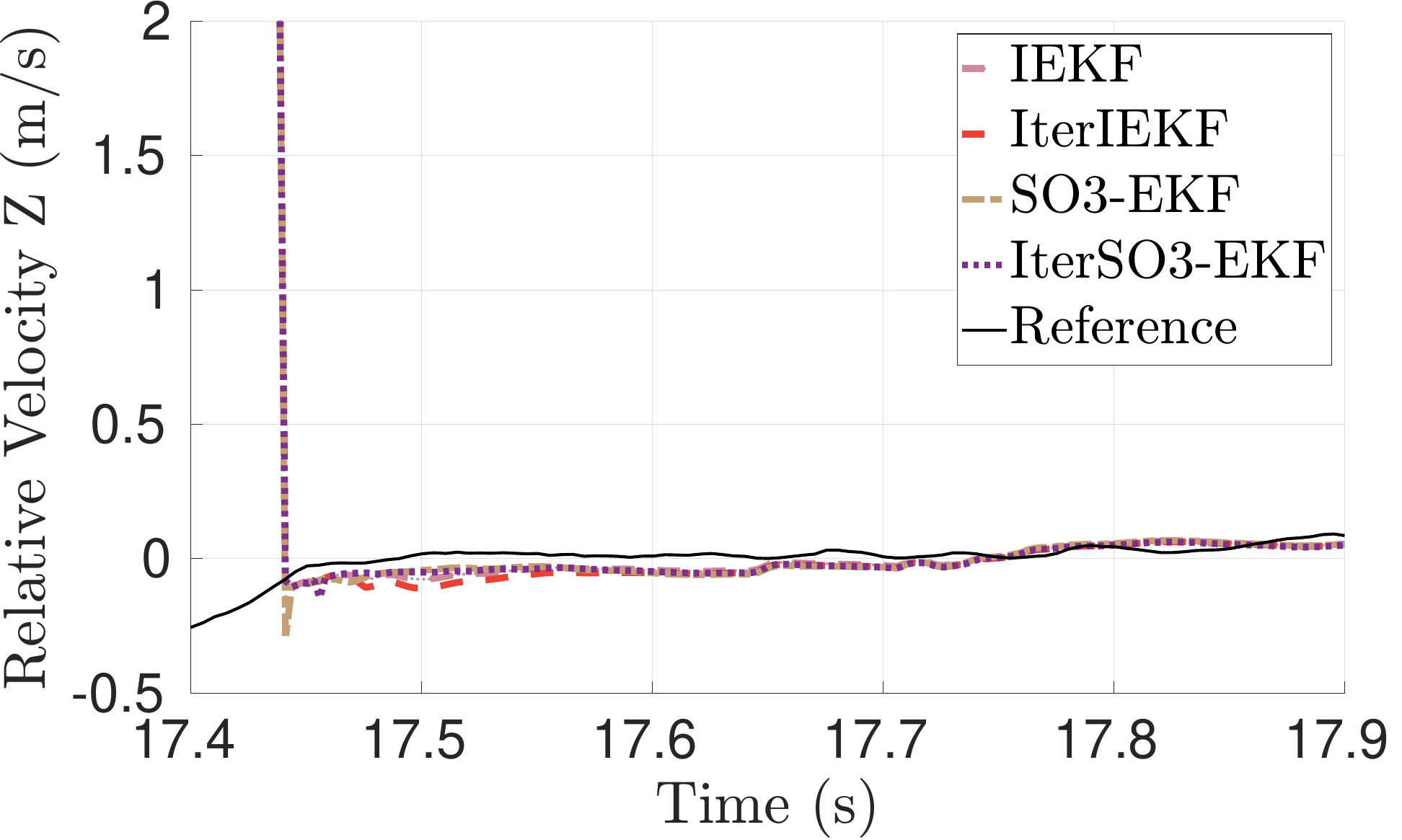}
    }
   \hfill
   \centering
   \subfloat[\protect\label{fig:pitch_scenario_ii}]{
       \includegraphics[scale=0.17]{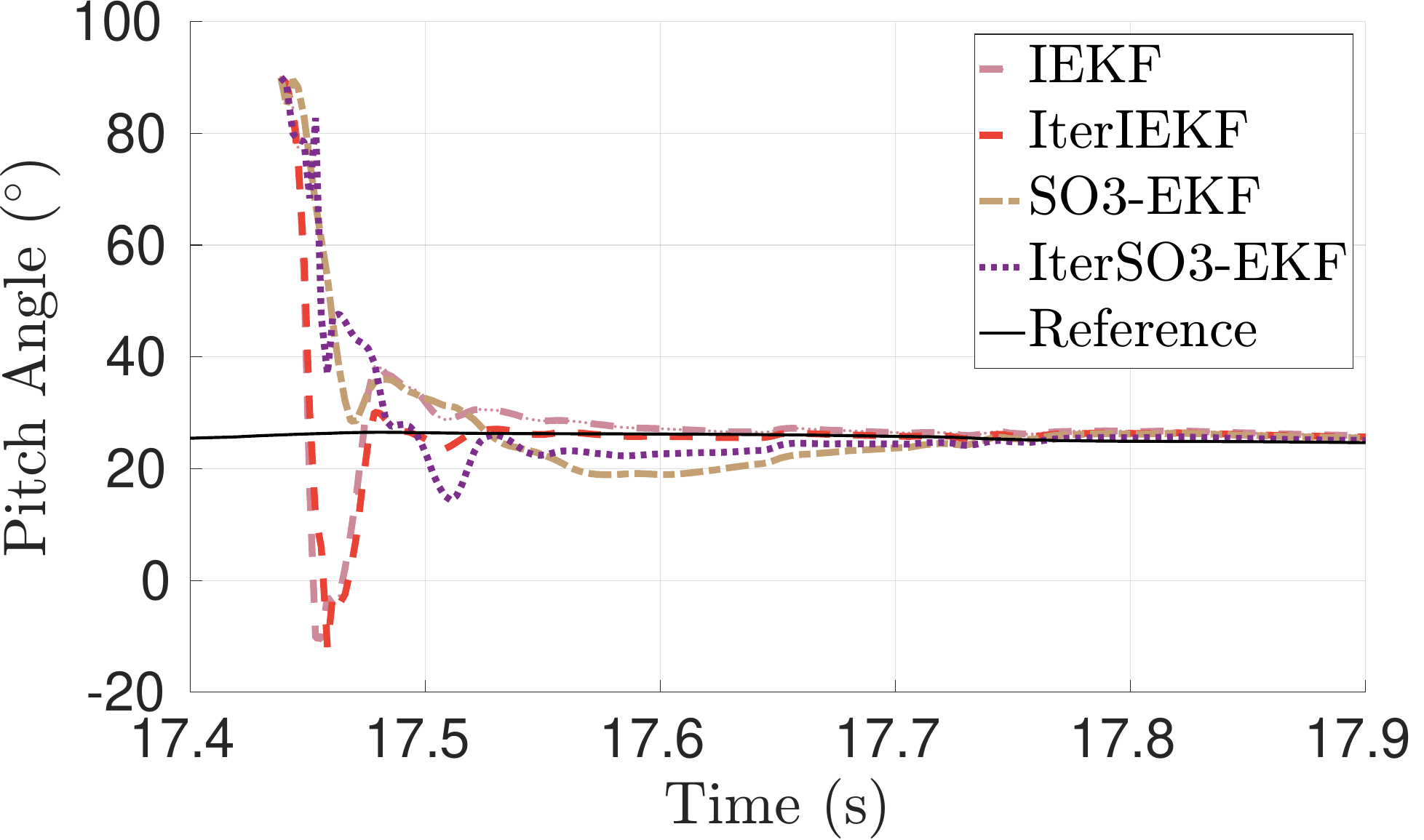}
    }
    \subfloat[\protect\label{fig:roll_scenario_ii}]{
        \includegraphics[scale=0.17]{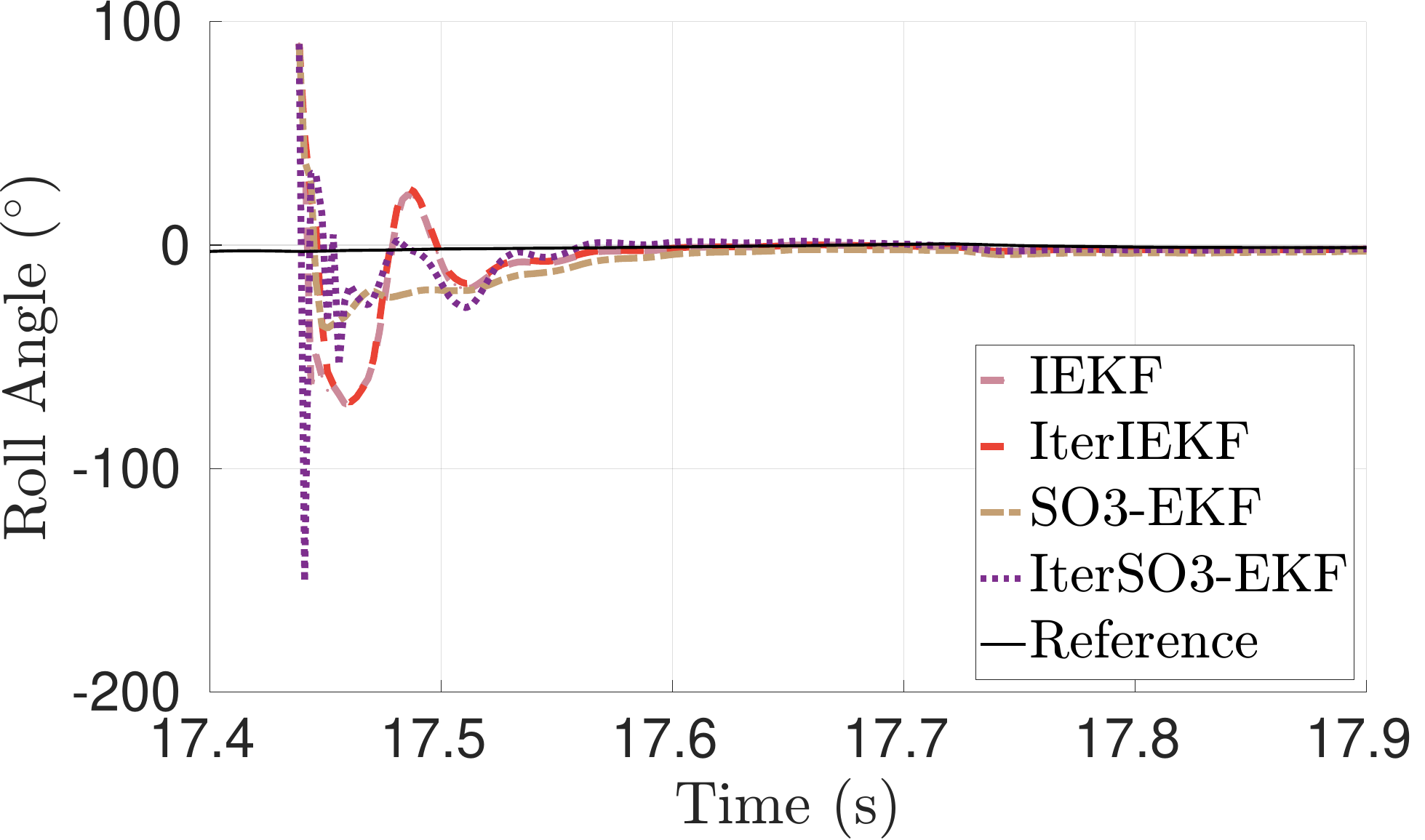}
    }
    \subfloat[\protect\label{fig:gravity_dir_scneario_ii}]{
        \includegraphics[scale=0.17]{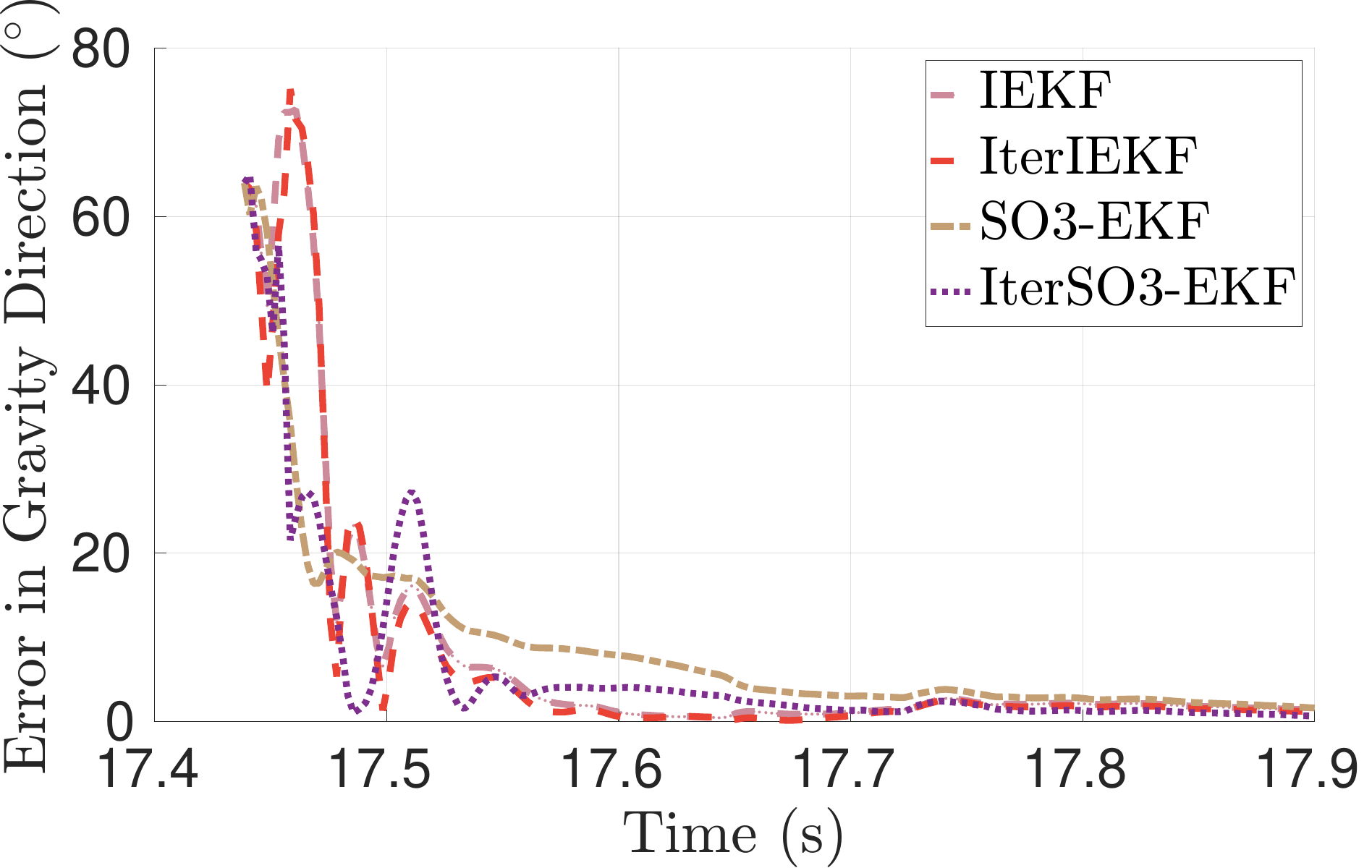}
    }
    \caption{
        Estimates obtained with different filters on the \texttt{Pilatus-Hike2} dataset under Scenario 2 conditions. 
            a) $x$-component,  b) $y$-component, c) $z$-component of the relative velocity $\bar{\mathbf{R}}_{i|i} \bar{\mathbf{v}}_{i|i}$. 
            d) Estimate of the pitch angle $\bar{\phi}_{i|i}$. e) Estimate of the roll angle $\bar{\theta }_{i|i}$. 
            f) Error in the gravity direction $\mathrm{acos}(\bar{\mathbf{u}}_{i|i} \cdot \mathbf{u}_{i})$.
    }
    \label{fig:result_scenario_ii}
\end{figure*}
\begin{figure*}[t!]
		\centering
    \subfloat[\protect\label{fig:vel_x_scenario_iii}]{
        \includegraphics[scale=0.172]{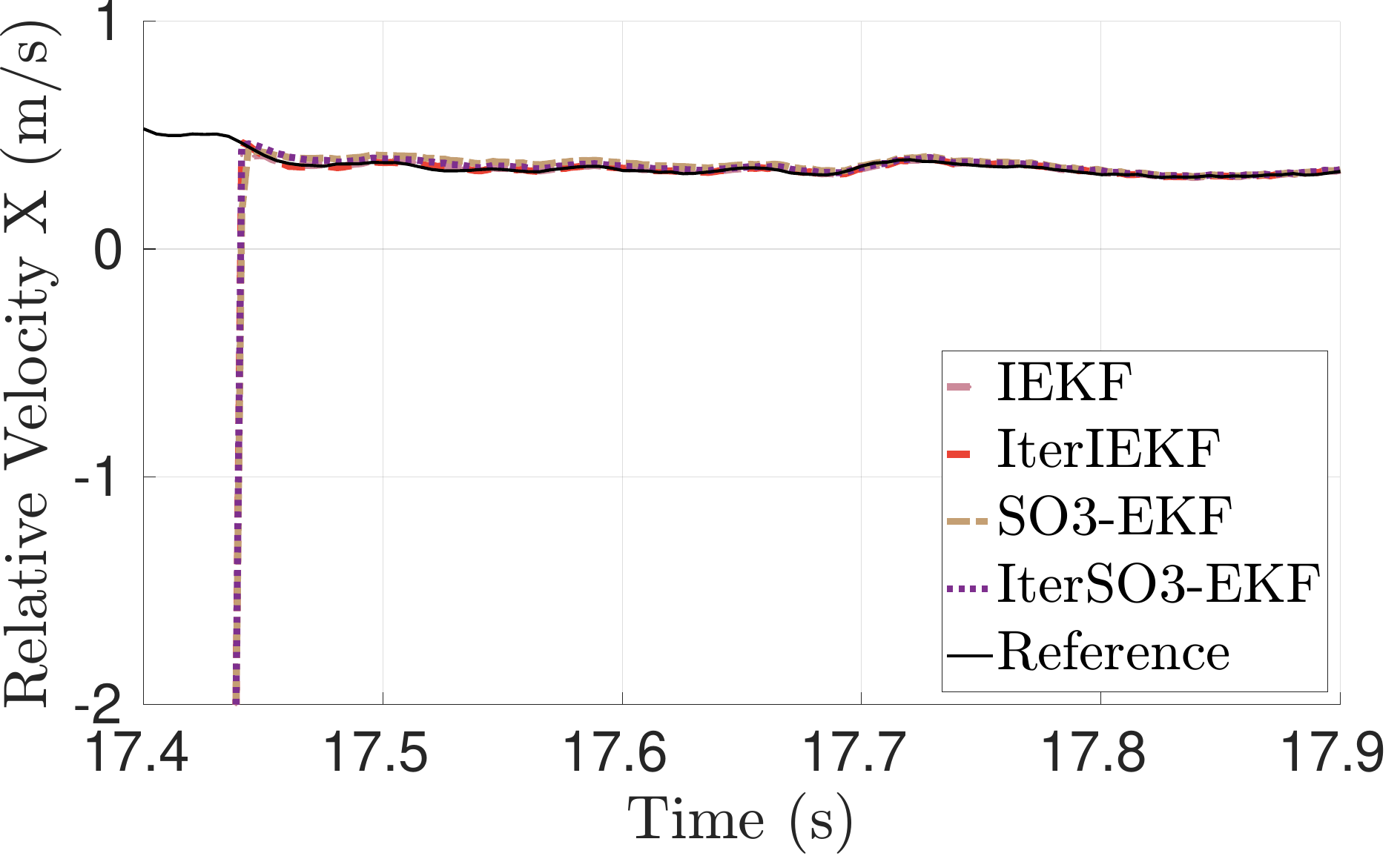}
    }
    \subfloat[\protect\label{fig:vel_y_scenario_iii}]{
        \includegraphics[scale=0.172]{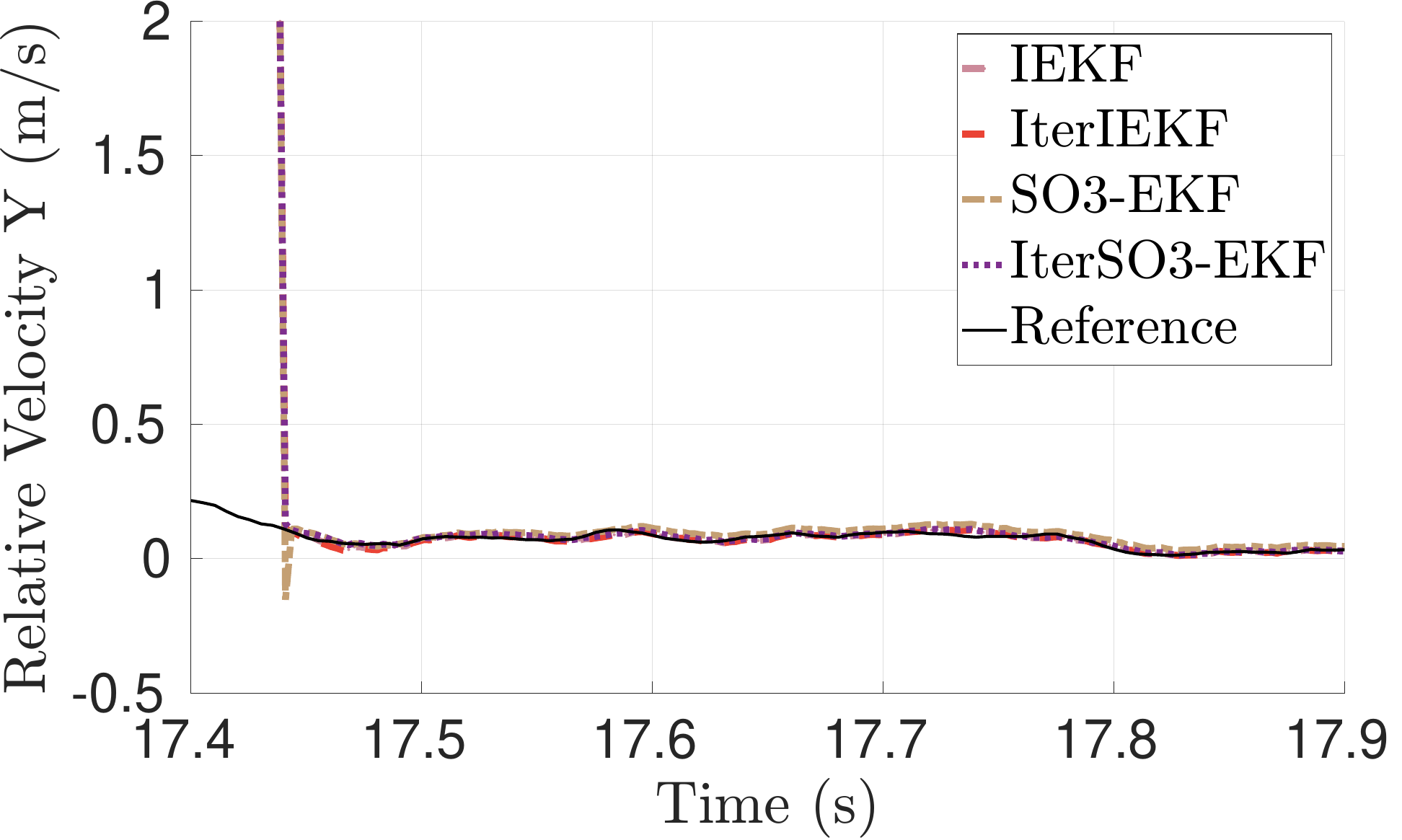}
    }
    \subfloat[\protect\label{fig:vel_z_scenario_iii}]{
        \includegraphics[scale=0.172]{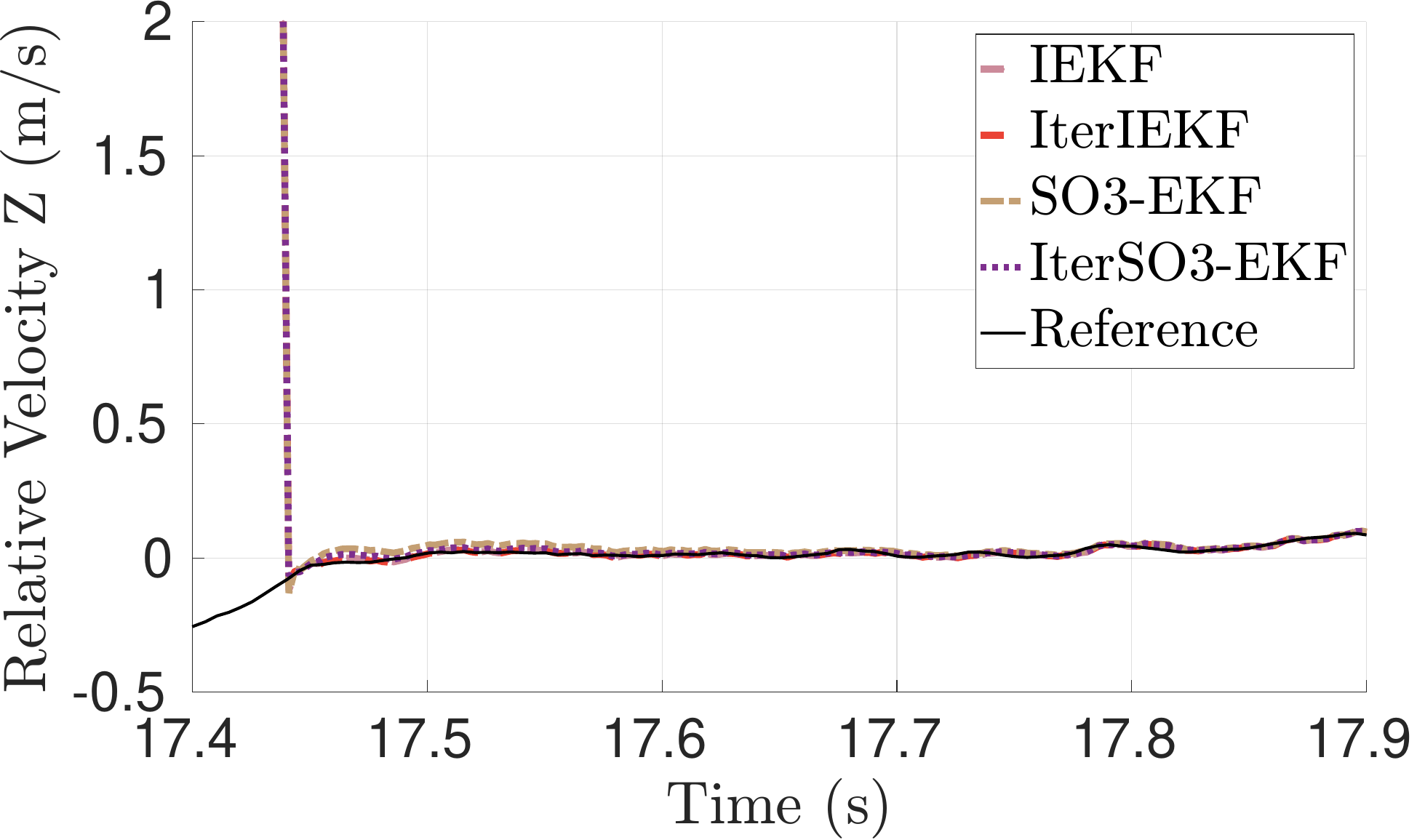}
    }
   \hfill
   \subfloat[\protect\label{fig:pitch_scenario_iii}]{
       \includegraphics[scale=0.172]{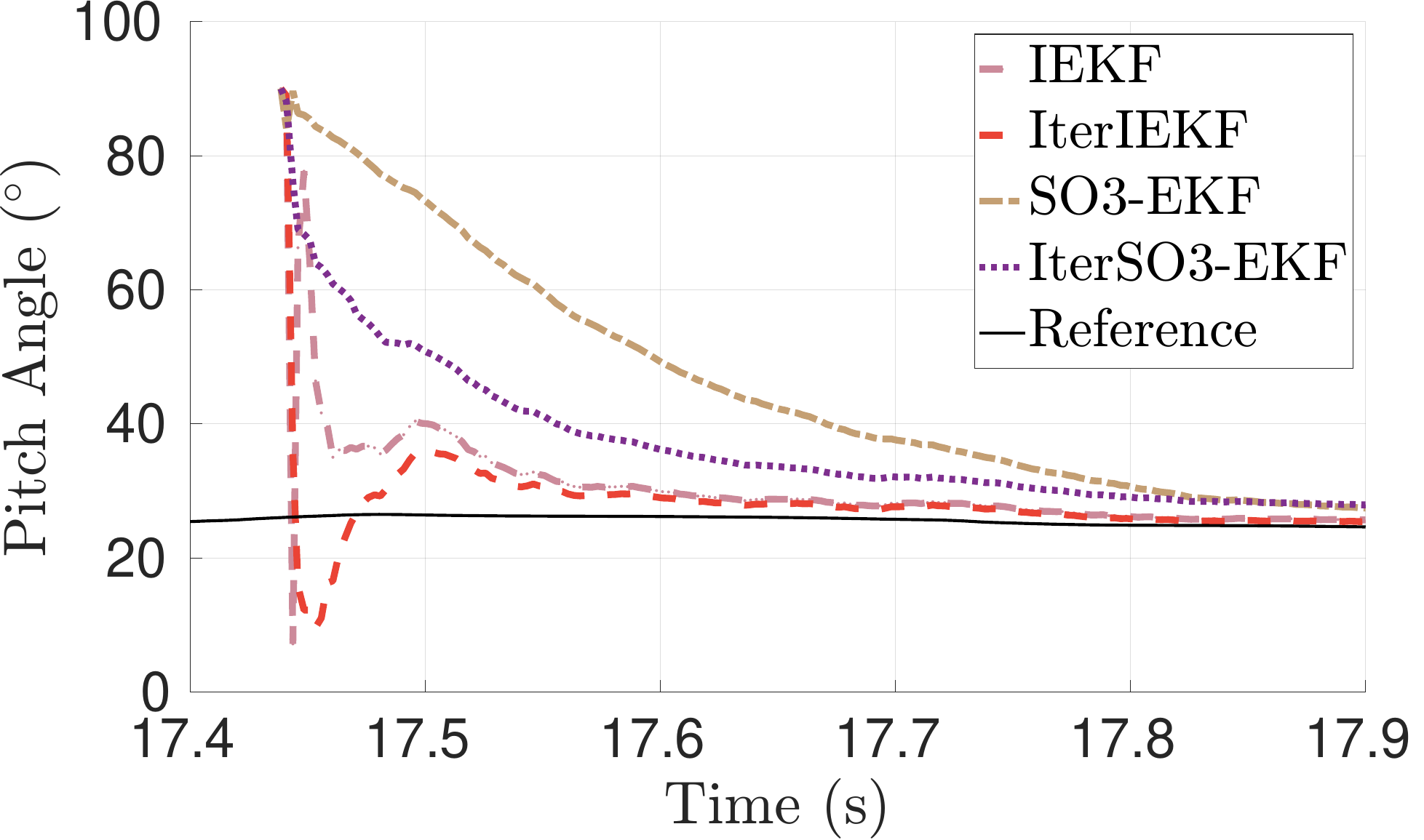}
    }
    \subfloat[\protect\label{fig:roll_scenario_iii}]{
        \includegraphics[scale=0.172]{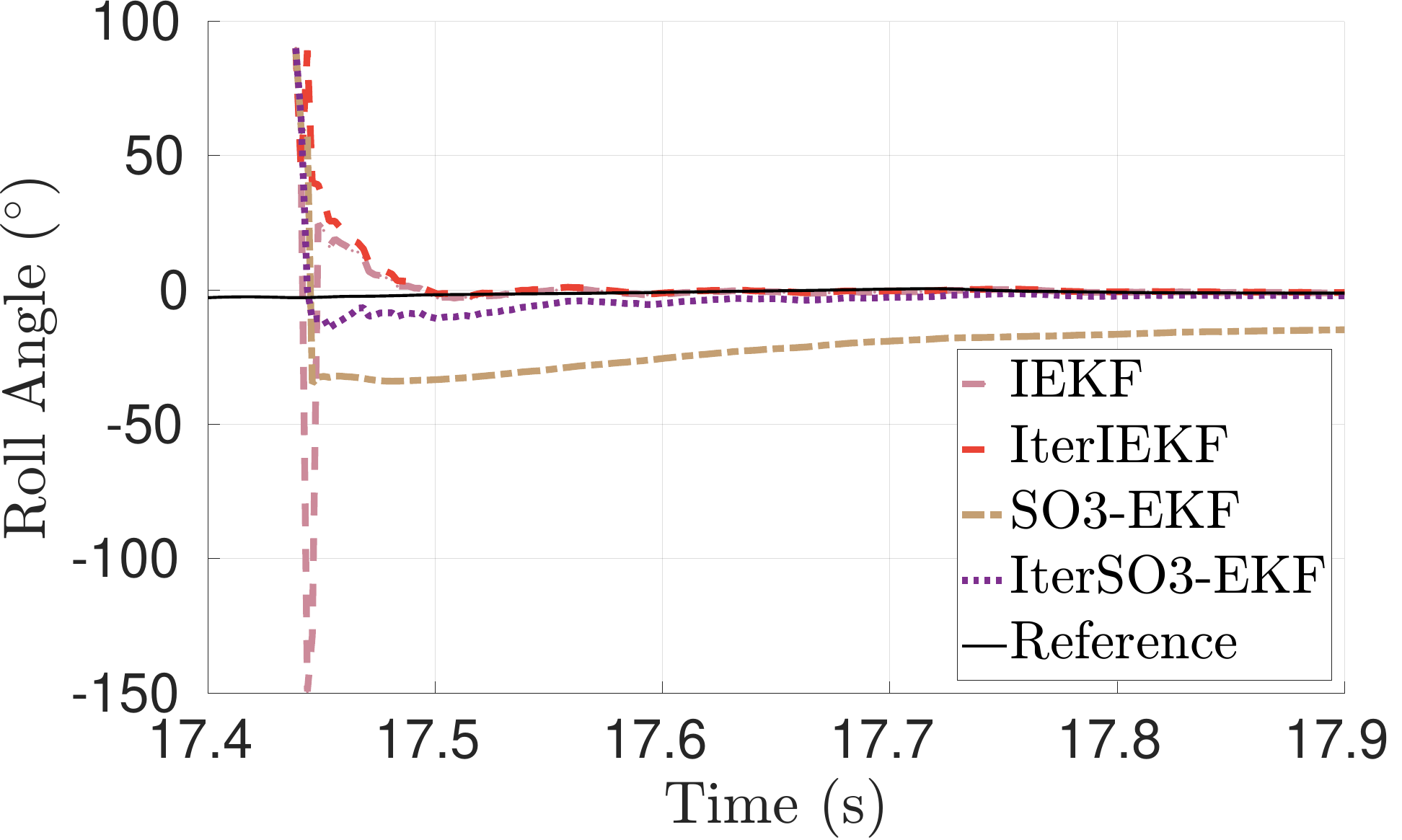}
    }
    \subfloat[\protect\label{fig:gravity_dir_scneario_iii}]{
        \includegraphics[scale=0.172]{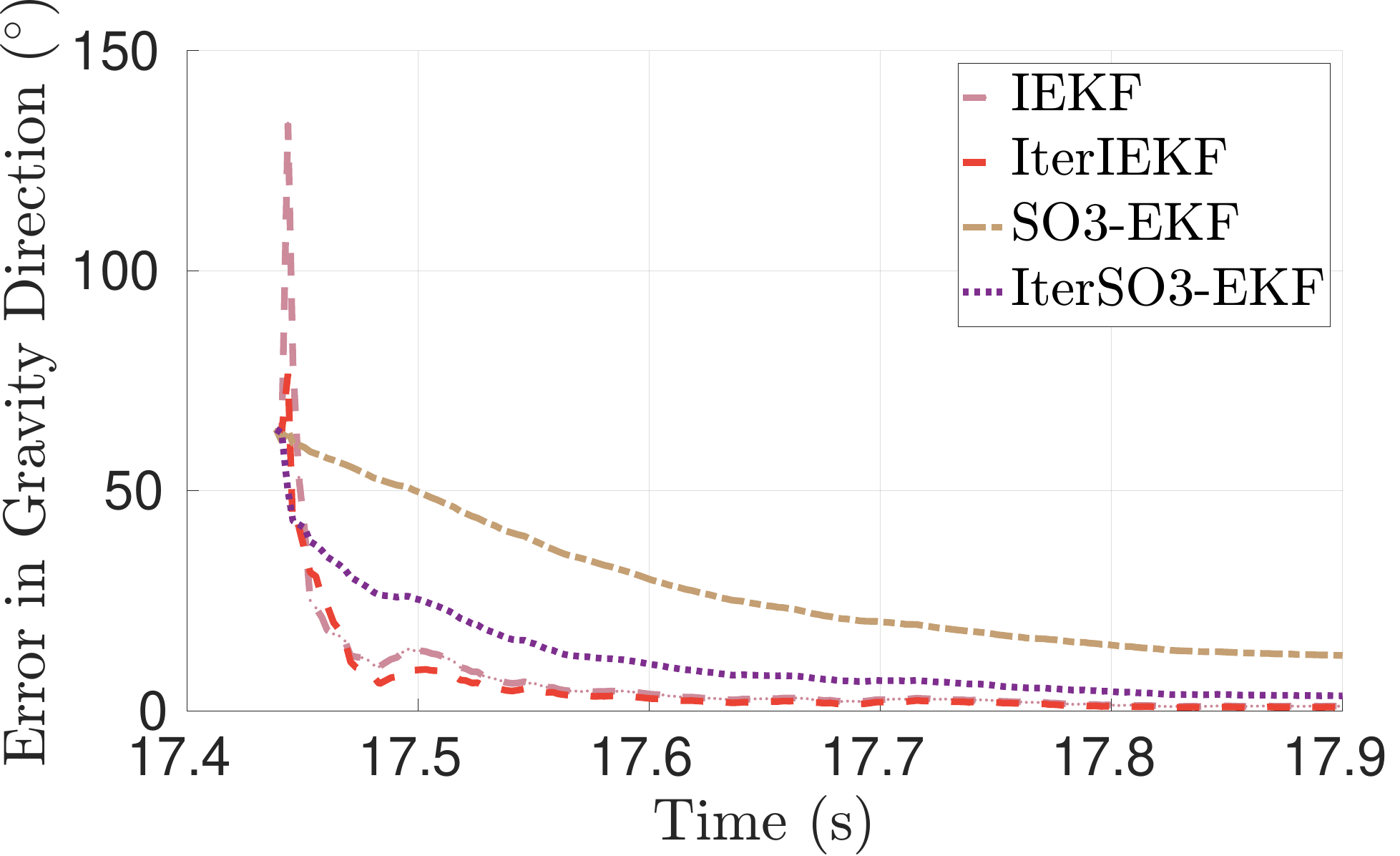}
    }
    \caption{
        Estimates obtained with different filters on the \texttt{Pilatus-Hike2} dataset under Scenario 3 conditions. 
            a) $x$-component,  b) $y$-component, c) $z$-component of the relative velocity $\bar{\mathbf{R}}_{i|i} \bar{\mathbf{v}}_{i|i}$. 
            d) Estimate of the pitch angle $\bar{\phi}_{i|i}$. e) Estimate of the roll angle $\bar{\theta }_{i|i}$. 
            f) Error in the gravity direction $\mathrm{acos}(\bar{\mathbf{u}}_{i|i} \cdot \mathbf{u}_{i})$.
    }
    \label{fig:result_scenario_iii}
\end{figure*}

\section{LIMITATIONS}
\label{sec:limitations}

There are two important limitations of the iterated filters. First, the covariance of the measurement function must be well known. 
Second, it should be small. Experimental results indicate that if an incorrect covariance matrix is used for the iterated filters, 
the estimate may be worse than that of their non-iterated counterparts. This is a classical feature of iterated filters, though, as they better incorporate the measurement by finding the actual maximum a posteriori: if the measurement does not bring much information, their advantage is limited and iterating on a bad measurement may prove counterproductive. Moreover, we observe that if the measurement covariance matrix of the measured velocity has relatively large values 
($> 10^{-2}\mathbf{I}_3 \,(\mathrm{m^2/s^2})$), the improvement achieved by the iterated filters is significantly reduced, as shown in the previous section.

\section{CONCLUSIONS AND FUTURE WORK}

In this work, we proposed a novel IterIEKF for state estimation in legged robots, building on the latest developments in invariant filtering theory. Through numerical simulations and a real-world dataset, we demonstrated that the proposed filter remains effective even in the presence of noisy measurements, thus extending the theoretical properties established in \cite{goffin2025} to realistic scenarios. Furthermore, we showed that as few as four iterations are sufficient to achieve a considerable improvement in the estimates, without introducing a significant increase in computational time.

As future work, we plan to extend the IterIEKF to other state-estimation scenarios. As shown in Fig.~\ref{fig:slam}, its superior accuracy in classical localization suggests strong potential for SLAM applications. We also aim to develop an IterIEKF formulation based on the Two-Frame Group~\cite{barrau2022} to incorporate accelerometer bias, and to investigate how additional measurements, such as forward kinematics, can further improve performance within an iterated filtering framework.



\textbf{}

\vspace{-0.7cm}

\appendix
\subsection{Lie Group formulas for $\mathrm{SE_2}(3)$ and $\mathrm{SO}(3)\times\mathbb{R}^{6}$}
\label{appendix:appendix_formulas}
Here we define exponential and logarithm maps for $\mathrm{SE}_2(3)$ and $\mathrm{SO}(3) \times \mathbb{R}^{6}$:
\begin{equation}
\begin{aligned}
	\mathbb{R}^{9} \mapsto  \mathrm{SE}_2(3) \\
\mathrm{Exp}_{\mathrm{SE}_2(3)}(\boldsymbol{\xi}) := 
\mathrm{Exp}([\boldsymbol{\xi}^{\mathbf{R}}\, \boldsymbol{\xi}^{\mathbf{v}}\, \boldsymbol{\xi}^{\mathbf{p}}]^{T}) =\\ 
\begin{bmatrix}
\mathrm{Exp}_{\mathrm{SO}(3)}(\boldsymbol{\xi}^{\mathbf{R}}) & \mathcal{J}_{l,\mathrm{SO}(3)}(\boldsymbol{\xi}^{\mathbf{R}})\boldsymbol{\xi}^{\mathbf{v}} & 
\mathcal{J}_{l,\mathrm{SO}(3)}(\boldsymbol{\xi}^{\mathbf{R}})\boldsymbol{\xi}^{\mathbf{p}} \\
\mathbf{0}_{1,3} & 1 & 0 \\
\mathbf{0}_{1,3} & 0 & 1
\end{bmatrix},
\end{aligned}
\end{equation}
\begin{equation}
\begin{aligned}
	\mathbb{R}^{9} \mapsto  \mathrm{SO}(3) \times \mathbb{R}^{6} \\
\mathrm{Exp}_{\mathrm{SO}(3) \times \mathbb{R}^{6}}(\delta \boldsymbol{x}) := 
\mathrm{Exp}([\delta \boldsymbol{x}^{\mathbf{R}}\, \delta \boldsymbol{x}^{\mathbf{v}}\, \delta \boldsymbol{x}^{\mathbf{p}}]^{T}) =\\ 
\{\mathrm{Exp}_{\mathrm{SO}(3)}(\delta \boldsymbol{x}^{\mathbf{R}}), \delta \boldsymbol{x}^{\mathbf{v}}, \delta \boldsymbol{x}^{\mathbf{p}}\},
\end{aligned}
\end{equation}
\begin{equation}
	\begin{aligned}
	\mathrm{SE}_2(3) \mapsto \mathbb{R}^{9} \\
\mathrm{Log}_{\mathrm{SE}_2(3)}(\mathcal{X}) :=
\begin{bmatrix}  
\mathrm{Log}_{\mathrm{SO}(3)}(\mathbf{R}) \\
\mathcal{J}_{l,\mathrm{SO}(3)}(\mathrm{Log}_{\mathrm{SO}(3)}(\mathbf{R}))^{-1} \mathbf{v} \\
\mathcal{J}_{l,\mathrm{SO}(3)}(\mathrm{Log}_{\mathrm{SO}(3)}(\mathbf{R}))^{-1} \mathbf{p}
\end{bmatrix} 
	,\end{aligned}
\end{equation}
\begin{equation}
\begin{aligned}
	\mathrm{SO}(3) \times \mathbb{R}^{6} \mapsto \mathbb{R}^{9} \\
\mathrm{Log}_{\mathrm{SO}(3) \times \mathbb{R}^{6}}(\boldsymbol{x}) :=
\begin{bmatrix}  
\mathrm{Log}_{\mathrm{SO}(3)}(\mathbf{R}) \\
\mathbf{v} \\
\mathbf{p}
\end{bmatrix} 
,\end{aligned}
\end{equation}
where 
\begin{equation}
\begin{aligned}
	\mathcal{J}_{l, \mathrm{SO}(3)}(\boldsymbol{\xi}^{\mathbf{R}}) = \mathbf{I}_3 + \frac{1-\cos(\|\boldsymbol{\xi}^{\mathbf{R}}\|)}
	{\|\boldsymbol{\xi}^{\mathbf{R}}\|^2} [\boldsymbol{\xi}^{\mathbf{R}}]^{\times} + \\
	\frac{\|\boldsymbol{\xi}^{\mathbf{R}}\| - \sin(\|\boldsymbol{\xi}^{\mathbf{R}}\|)}{\|\boldsymbol{\xi}^{\mathbf{R}}\|^3} ([\boldsymbol{\xi}^{\mathbf{R}}]^{\times})^{2},
\end{aligned}
\end{equation}
\begin{equation}
	\begin{aligned}
	\mathcal{J}_{l, \mathrm{SO}(3)}(\boldsymbol{\xi}^{\mathbf{R}})^{-1} = \mathbf{I}_3 - \frac{1}{2} [\boldsymbol{\xi}^{\mathbf{R}}]^{\times} + \\
	\left(\frac{1}{\|\boldsymbol{\xi}^{\mathbf{R}}\|^2} - \frac{1+\cos(\|\boldsymbol{\xi}^{\mathbf{R}}\|)}{2 \|\boldsymbol{\xi}^{\mathbf{R}}\| \sin(\|\boldsymbol{\xi}^{\mathbf{R}}\|)}\right) 
	([\boldsymbol{\xi}^{\mathbf{R}}]^{\times})^{2}
	,\end{aligned}
\end{equation}
\begin{equation}
	\begin{aligned}
	\mathrm{Exp}_{\mathrm{SO}(3)}(\boldsymbol{\xi}^{\mathbf{R}}) = \mathbf{I}_3 + \frac{\sin(\|\boldsymbol{\xi}^{\mathbf{R}}\|)}{\|\boldsymbol{\xi}^{\mathbf{R}}\|} [\boldsymbol{\xi}^{\mathbf{R}}]^{\times} + \\ \frac{1-\cos(\|\boldsymbol{\xi}^{\mathbf{R}}\|)}{\|\boldsymbol{\xi}^{\mathbf{R}}\|^2} 
	([\boldsymbol{\xi}^{\mathbf{R}}]^{\times})^{2},
	\end{aligned}
\end{equation}
\begin{equation}
\begin{aligned}
	\mathrm{Log}_{\mathrm{SO}(3)} (\mathbf{R}) = \frac{\theta}{2 \sin(\theta)} (\mathbf{R} - \mathbf{R}^{T})^{\vee}, \\
	\theta = \mathrm{acos}\left(\frac{\mathrm{tr}(\mathbf{R}) - 1}{2}\right)	
,\end{aligned}
\end{equation}
and $(\cdot)^{\times}$ is the skew-symmetric operator and $(\cdot)^{\vee}$ is its inverse. The right-jacobian for $\mathrm{SE}_2(3)$, $\mathcal{J}_{r, \mathrm{SE}_2(3)}$, is given by \cite{brossard2020}:
\begin{equation}
\mathcal{J}_{r, \mathrm{SE}_2(3)}(\boldsymbol{\xi}) = \mathcal{J}_{l, \mathrm{SE}_2(3)}(-\boldsymbol{\xi}),
\end{equation}
\begin{equation}
\begin{aligned}
	\mathcal{J}_{l, \mathrm{SE}_2(3)}(\boldsymbol{\xi}) = \\
\begin{bmatrix}
	\mathcal{J}_{l, \mathrm{SO}(3)}(\boldsymbol{\xi}^{\mathbf{R}}) & 0 & 0 \\
	Q(\boldsymbol{\xi}^{\mathbf{R}}, \boldsymbol{\xi}^{\mathbf{v}}) & \mathcal{J}_{l, \mathrm{SO}(3)}(\boldsymbol{\xi}^{\mathbf{R}}) & 0 \\
	Q(\boldsymbol{\xi}^{\mathbf{R}},\boldsymbol{\xi}^{\mathbf{p}}) & 0 & \mathcal{J}_{l, \mathrm{SO}(3)}(\boldsymbol{\xi}^{\mathbf{R}})
\end{bmatrix}
,\end{aligned}
\end{equation}
\begin{equation}
	\begin{aligned}
Q(\boldsymbol{\xi}^{\mathbf{R}}, \boldsymbol{\xi}^{\mathbf{u}})
=
\frac12 \mathbf{U} 
+ a(\mathbf{V} \mathbf{U} + \mathbf{U}\mathbf{V} + \mathbf{V}\mathbf{U} \mathbf{V}) \\
- b(\mathbf{V}\mathbf{V}\mathbf{U} + \mathbf{U}\mathbf{V}\mathbf{V} - 3\mathbf{V}\mathbf{U}\mathbf{V}) \\
+ c(\mathbf{V}\mathbf{U}\mathbf{V}\mathbf{V} + \mathbf{V}\mathbf{V}\mathbf{U}\mathbf{V}), \\
\mathbf{V} = (\boldsymbol{\xi}^{\mathbf{R}}{})^{\times}, \theta = \|\boldsymbol{\xi}^{\mathbf{R}}\|,
\mathbf{U} = (\boldsymbol{\xi}^{\mathbf{u}}{})^{\times}, \\
a = \frac{\theta - \sin\theta}{\theta^3}, 
b = \frac{1 - \tfrac{\theta^2}{2} - \cos\theta}{\theta^4}, \\
c = -\frac12
\left(
\frac{1 - \tfrac{\theta^2}{2} - \cos\theta}{\theta^4}
-
3\frac{\theta - \sin\theta - \tfrac{\theta^3}{6}}{\theta^5}
\right)
.\end{aligned}
\end{equation}

\subsection{Proof of Proposition ~\ref{prop:coset}}
\label{appendix:appendix_proof}
Let $G$ be the same Lie Group as in Definition~\ref{def:submanifold} and let 
$G_\mathbf{d}$ be a subgroup of $G$ defined as:
\begin{equation}
  G_\mathbf{d} := \{\mathcal{X} \in G \mid \mathcal{X} \mathbf{d} = \mathbf{d}\},
\end{equation}
known as the stabilizer of $\mathbf{d}$. Let $\mathcal{X}_0 \in S_{\mathcal{X}^{-1} \mathbf{y}= \mathbf{d}}$ such that 
$\mathcal{X}_0^{-1} \mathbf{d} = \mathbf{y}$. Then, for any element 
$\mathcal{X} \in G_\mathbf{d} \mathcal{X}_0$, we have:
\begin{equation}
    \mathcal{X}^{-1} \mathbf{d} 
    = (\mathcal{X}_0^{-1} \mathcal{H}^{-1}) \mathbf{d} 
    = \mathcal{X}_0^{-1} (\mathcal{H}^{-1} \mathbf{d}) 
    = \mathcal{X}_0^{-1} \mathbf{d} 
    = \mathbf{y},
\end{equation}
where $\mathcal{H} \in G_\mathbf{d}$, so that $G_\mathbf{d} \mathcal{X}_0 \subseteq S_\mathcal{X}$. Conversely, for any 
$\mathcal{X} \in S_{\mathcal{X}^{-1} \mathbf{d} = \mathbf{y}}$, let $\mathcal{H}:= \mathcal{X} \mathcal{X}_0^{-1}$, notice that:
\begin{equation}
    \mathcal{H} \mathbf{d} = 
    \mathcal{X} \mathcal{X}_0^{-1} \mathbf{d} =
    \mathcal{X} \mathbf{y} = \mathbf{d},
\end{equation}
then $\mathcal{H} \in G_\mathbf{d}$ and $S_{\mathcal{X}^{-1}\mathbf{d} = \mathbf{y}} = G_\mathbf{d} \mathcal{X}_0$.

\subsection{Robocentric Formulation of IterIEKF}
\label{appendix:appendix_robocentric}
In this section, we develop a robot-centric formulation of the IterIEKF. While the proposed world-centric formulation is well-suited for navigation, a robot-centric approach is more appropriate for control, since the base velocity in the state is naturally expressed in the base frame. 
We define
the robocentric state, $\boldsymbol{\chi}$, by using the inverse of $\mathcal{X}$ defined in~\eqref{eq:state_inverse}:
\begin{equation}
\boldsymbol{\mathcal{\chi}} := \mathcal{X}^{-1} = (\mathbf{R}^{T}, - \mathbf{R}^{T} \mathbf{v}, -\mathbf{R}^{T} \mathbf{p}) :=
(\boldsymbol{C}, \boldsymbol{v}, \boldsymbol{r}) := (\boldsymbol{C}, \mathbf{e}),
\end{equation}
where $\boldsymbol{C} := \mathbf{R}^{T}$ and $\mathbf{e} := [\boldsymbol{v}, \boldsymbol{r}]^{T}$. Notice
we keep the same notation as \cite{bloesch2013}.

We obtain the prediction step by inverting~\eqref{eq:prediction_step}:
\begin{equation}
 \begin{aligned}
     \bar{\boldsymbol{\chi}}_{i|i-1} = \mathcal{F}_{i-1}^{-1}(\bar{\mathcal{X}}_{i-1}) = 
     \mathcal{F}_{i-1}^{-1}(\bar{\boldsymbol{\chi}}_{i-1}) = \\
\bar{\mathcal{Y}}_{i-1}^{-1}(\mathbf{w}_{i-1})
\Phi_{i-1}^{-1}(\bar{\boldsymbol{\chi}}_{i-1}) \mathcal{W}_{i-1}^{-1} := \\
\bar{\Upsilon}_{i-1}(\mathbf{w}_{i-1}) \Phi_{i-1} (\bar{\boldsymbol{\chi}}_{i-1}) \Gamma_{i-1}
 \end{aligned}
\end{equation}
Each inverse is obtained individually, as follows:
\begin{equation}
\begin{aligned}
\Phi_{i-1}^{-1}(\bar{\mathcal{X}}) = [(\mathbf{I}_{3}, (\mathbf{F}_{i-1} - \mathbf{I}_6)*\mathbf{x}) \circ (\bar{\mathbf{R}}_{i-1}, \bar{\mathbf{x}}_{i-1})]^{-1} =  \\
(\bar{\mathbf{R}}_{i-1}, \bar{\mathbf{x}}_{i-1})^{-1} (\mathbf{I}_{3}, (\mathbf{F}_{i-1} - \mathbf{I}_{6})*\bar{\mathbf{x}}_{i-1})^{-1} =  \\
(\bar{\boldsymbol{C}}_{i-1}, \bar{\mathbf{e}}_{i-1}) (\mathbf{I}_{3}, (\mathbf{F}_{i-1} - \mathbf{I}_{6})* \bar{\boldsymbol{C}}_{i-1}^{T} \bar{\mathbf{e}}_{i-1}) = \\
(\bar{\boldsymbol{C}}_{i-1}, (\bar{\boldsymbol{C}}_{i-1}* \bar{\mathbf{F}}_{i-1}* \bar{\boldsymbol{C}}^{T}_{i-1})*\bar{\mathbf{e}}_{i-1}) = \\
(\bar{\boldsymbol{C}}_{i-1}, \mathbf{F}_{i-1}* \bar{\mathbf{e}}_{i-1}) = \Phi_{i-1}(\boldsymbol{\chi}_{i-1}),
\end{aligned}
\end{equation}
where the last equality followed because $\mathbf{F}_{i-1}$ 
commutes with $\bar{\boldsymbol{C}}_{i-1}$ (see Proposition 1 in \cite{barrau2022}). We also have:
\begin{equation}
    \Gamma_{i-1} := \mathcal{W}_{i-1}^{-1} = (\mathbf{I}_{3}, -\mathbf{d}_{i-1}),
\end{equation}
and
\begin{equation}
\begin{aligned}
    &\Upsilon_{i-1} := \bar{\mathcal{Y}}_{i-1}^{-1} = \\
(
&\mathrm{Exp}_{\mathrm{SO}(3)}(\mathbf{G}_{i-1}^{1} \mathbf{w}_{g,i-1})^{T}
\mathrm{Exp}_{\mathrm{SO}(3)}({\boldsymbol{\omega}}_{I,i-1}\Delta t)^{T} 
\\- 
&\mathrm{Exp}_{\mathrm{SO}(3)}(\mathbf{G}_{i-1}^{1} \mathbf{w}_{g,i-1})^{T}
\mathrm{Exp}_{\mathrm{SO}(3)}({\boldsymbol{\omega}}_{I,i-1}\Delta t)^{T} \\
&(\mathbf{s}_{i-1} + \mathbf{G}_{i-1}^{2} \mathbf{w}_{a,i-1} ).
\end{aligned}
\end{equation}
By composing all inverses we obtain the prediction step:
\begin{equation}
\label{eq:prediction_step_robocentric}
\begin{aligned}
\boldsymbol{\chi}_{i|i-1} =
\Bigl(
\bar{\mathbf{G}}_{i-1}\bar{\boldsymbol{C}}_{i-1},\;
-\bar{\mathbf{G}}_{i-1}*\bigl(\mathbf{s}_{i-1} + \mathbf{G}_{i-1}^{2}\mathbf{w}_{a,i-1}\bigr)
+ \\ \bar{\mathbf{G}}_{i-1}*\mathbf{F}_{i-1}\bar{\mathbf{e}}_{i-1} 
-(\bar{\mathbf{G}}_{i-1}\bar{\boldsymbol{C}}_{i-1})*\mathbf{d}_{i-1}
\Bigr),
\end{aligned}
\end{equation}
where
\begin{equation}
\bar{\mathbf{G}}_{i-1}
:=
\mathrm{Exp}_{\mathrm{SO}(3)}(\mathbf{G}_{i-1}^{1} \mathbf{w}_{g,i-1})^{T}
\mathrm{Exp}_{\mathrm{SO}(3)}(\boldsymbol{\omega}_{I,i-1}\Delta t)^{T}.
\end{equation}
By expanding~\eqref{eq:prediction_step_robocentric} we can obtain the same IMU integration as \cite{bloesch2013}.

In the robocentric framework, the appropriate measure of error is the left-invariant, $\boldsymbol{\xi}_{l}$. 
This happens because we have a measurement in the base frame; otherwise, if we have a measurement in the world frame, the appropriate error would be the right invariant \cite{barrau2022}. The opposite happens when we have a world-centric formulation.
The state part of the propagation of the left-invariant error was obtained in \cite{brossard2022} and given as follows:
\begin{equation}
    \boldsymbol{\xi}_{l, i|i-1} = \hat{\mathbf{A}}_{i-1} \boldsymbol{\xi}_{l, i}, \; \hat{\mathbf{A}}_{i-1} := \mathrm{Ad}_{\Gamma_{i-1}^{-1}}  \mathbf{F}_{i-1}.
\end{equation}
The noisy part is given by:
\begin{equation}
\begin{aligned} 
\bar{\boldsymbol{\chi}}_{i|i-1} \oplus \boldsymbol{\xi}_{l,i|i-1} = 
\mathrm{Ad}_{\bar{\boldsymbol{\chi}}_{i|i-1}} \boldsymbol{\xi}_{l,i|i-1} \oplus  \bar{\boldsymbol{\chi}}_{i|i-1} &= \\ 
\Upsilon_{i-1}(\mathbf{0}_{6,1}) \Phi_{i-1} (\bar{\boldsymbol{\chi}}_{i|i-1})\Gamma_{i-1} &\implies \\
\mathrm{Ad}_{\bar{\boldsymbol{\chi}}_{i|i-1}} \boldsymbol{\xi}_{l,i|i-1} \oplus \bar{\Upsilon}_{i-1} (\mathbf{w}_{i-1}) \Phi(\bar{\boldsymbol{\chi}}_{i|i-1}) \Gamma_{i-1} &= \\ 
\Upsilon_{i-1}(\mathbf{0}_{6,1}) \Phi_{i-1} (\bar{\boldsymbol{\chi}}_{i|i-1})\Gamma_{i-1} &\implies \\ 
\boldsymbol{\xi}_{l,i-1} =  \mathrm{Ad}_{\bar{\boldsymbol{\chi}}_{i|i-1}^{-1}}
\mathrm{Log}(\Upsilon_{i-1}(\mathbf{0}_{6,1}) \bar{\Upsilon}_{i-1}^{-1})
\end{aligned}
\end{equation}
To obtain $\mathrm{Log}(\Upsilon_{i-1}(\mathbf{0}_{6}) \bar{\Upsilon}_{i-1}^{-1})$, notice:
\begin{equation}
\begin{aligned}
\mathrm{Log}(\Upsilon_{i-1}(\mathbf{0}_{6,1}) \bar{\Upsilon}_{i-1}^{-1}) =  \\
\mathrm{Log}(\mathcal{Y}_{i-1}(\mathbf{0}_{6,1})^{-1} \bar{\mathcal{Y}}_{i-1}(\mathbf{w}_{i-1})) =  \\
\mathrm{Log}((\bar{\mathcal{Y}}_{i-1} \mathcal{Y}_{i-1})^{-1}) \stackrel{\eqref{eq:B_noise}}{\approx} \\
\mathrm{Log}(\mathrm{Exp}(\mathbf{G}_{i-1} \mathbf{w}_{i-1})^{-1}) = \\
- \mathbf{G}_{i-1} \mathbf{w}_{i-1}.
\end{aligned}
\end{equation}
The final error propagation is given as:
\begin{equation}
    \begin{aligned}
    \boldsymbol{\xi}_{l,i|i-1} = \hat{\mathbf{A}}_{i-1} \mathbf{P}_{i-1} \hat{\mathbf{A}}^{T} +\\ 
    (\mathrm{Ad}_{\boldsymbol{\chi}_{i|i-1}^{-1}}\mathbf{G}_{i-1})\mathbf{Q}_{i-1} 
    (\mathrm{Ad}_{\boldsymbol{\chi}_{i|i-1}^{-1}}\mathbf{G}_{i-1})^{T}
    .\end{aligned}
\end{equation}
The base velocity as a measurement can be obtained as an action of the state:
\begin{equation}
\begin{aligned}
\mathbf{y}_i := \begin{bmatrix} \tilde{\mathbf{v}}_{I,i} \\ -1 \\ 0 \end{bmatrix} = 
\begin{bmatrix} \mathbf{R}^{T}_i \mathbf{v}_i + \mathbf{w}_{f,i} \\ -1 \\ 0 \end{bmatrix} = \\
\begin{bmatrix} \boldsymbol{C}_i & \boldsymbol{v}_i & \boldsymbol{r}_i  \\
\mathbf{0}_{3,1} & 1 & 0 \\
0 & 0 & 1 \end{bmatrix}\begin{bmatrix}
\mathbf{0}_{3,1} \\
-1 \\
0
\end{bmatrix} + \begin{bmatrix}
\mathbf{w}_{f,i} \\
0 \\
0
\end{bmatrix} = \\
\boldsymbol{\chi}_{i} \mathbf{d} + \mathbf{n}_i,
\end{aligned}
\end{equation}
For the left-invariant error, the innovation is defined as~\cite{barrau2022}:
\begin{equation}
    \begin{aligned}
        \mathbf{z}_{i} := \bar{\boldsymbol{\chi}}_{i|i-1}^{-1} \mathbf{y}_{i} - \mathbf{d} = \\
        \mathrm{Exp}(\boldsymbol{\xi}_{l,i}) \mathbf{d} - \mathbf{d} + \bar{\boldsymbol{\chi}}_{i|i-1}^{-1} \mathbf{n}  \approx \\
        - \mathbf{H} \boldsymbol{\xi}_{l,i} + \bar{\boldsymbol{\chi}}_{i|i-1}^{-1} \mathbf{n} := \\
        \hat{\mathbf{H}} \boldsymbol{\xi}_{l,i} + \bar{\boldsymbol{\chi}}_{i|i-1}^{-1} \mathbf{n}
    .\end{aligned}
\end{equation}
Finally, to obtain the update step of the left-invariant IterIEKF, we minimize the equivalent of~\eqref{eq:minimization},
by means of the Gauss-Newton iterations proposed in \cite{goffin2025}:
\begin{equation}
    \tilde{\hat{\mathbf{H}}}_{i}^{j+1} = \mathrm{Exp}_{\mathrm{SO}(3)}([\boldsymbol{\xi}_{l,i}^{j}]_{1:3}) 
    \tilde{\hat{\mathbf{H}}}_{i} \mathcal{J}_{r}(\boldsymbol{\xi}_{l,i}^{j}),
\end{equation}
\begin{equation}
    \mathbf{S}_i^{j+1} = \tilde{\hat{\mathbf{H}}}_{i}^{j+1} \mathbf{P}_{i|i-1} 
    (\tilde{\hat{\mathbf{H}}}_{i}^{j+1})^{T} + \hat{\tilde{\mathbf{N}}}_{i},
\end{equation}
\begin{equation}
    \mathbf{K}^{j+1}_{i} = \mathbf{P}_{i|i-1}(\tilde{\hat{\mathbf{H}}}_{i}^{j+1})^{T}
    (\mathbf{S}_i^{j+1})^{-1},
\end{equation}
\begin{equation}
    \boldsymbol{\xi}_{l,i}^{j+1} = 
    \mathbf{K}_{i}^{j+1} (\tilde{\mathbf{z}}_{i} - [\mathrm{Exp}(\boldsymbol{\xi}_{l,i}^{j})\mathbf{d} - \mathbf{d}]_{1:3} + 
    \tilde{\hat{\mathbf{H}}}_{i}^{j+1} \boldsymbol{\xi}_{l,i}^{j}),
\end{equation}
\begin{equation}
    \tilde{\hat{\mathbf{H}}}_{i}^{0} := \tilde{\hat{\mathbf{H}}}_{i}, \boldsymbol{\xi}_{l,i}^{0} := \mathbf{0}_{9,1}.
\end{equation}
Let $k$ denote the last iteration, then the update step for the left-invariant error IterIEKF
is given as follows:
\begin{equation}
    \begin{aligned}
        \hat{\boldsymbol{\xi}}_{l,i} := \boldsymbol{\xi}_{l,i}^{k}, \\
    \end{aligned}
\end{equation}
\begin{equation}
    \bar{\boldsymbol{\chi}}_{i} := \bar{\boldsymbol{\chi}}_{i|i} := 
    \bar{\boldsymbol{\chi}}_{i|i-1}
    \oplus 
    \hat{\boldsymbol{\xi}}_{l,i}
\end{equation}
\begin{equation}
    \mathbf{P}_{i|i} = (\mathbf{I}_9 - \mathbf{K}_{i}^{0} \tilde{\hat{\mathbf{H}}}_{i}^{0}) \mathbf{P}_{i|i-1}.
\end{equation}
%



\bibliographystyle{IEEEtran}
\bibliography{main}

@inproceedings{i2ekflo,
  title={I2{EKF}-{LO}: A dual-iteration extended {K}alman filter based {L}i{DAR} odometry},
  author={Yu, Wenlu and Xu, Jie and Zhao, Chengwei and Zhao, Lijun and Nguyen, Thien-Minh and Yuan, Shenghai and Bai, Mingming and Xie, Lihua},
  booktitle={2024 IEEE/RSJ Int. Conf. Intell. Robots Syst.},
  pages={10453--10460},
  year={2024},
  organization={IEEE}
}

@article{foxlin2005pedestrian,
  title={Pedestrian tracking with shoe-mounted inertial sensors},
  author={Foxlin, Eric},
  journal={IEEE Comput. Graph. Appl.},
  volume={25},
  number={6},
  pages={38--46},
  year={2005},
  publisher={IEEE}
}

@article{goffin2025,
  title   = {Iterated {I}nvariant {E}xtended {K}alman {F}ilter ({I}ter{IEKF})},
  author  = {Goffin, Sven and Barrau, Axel and Bonnabel, Silv{\`e}re and Br{\"u}ls, Olivier and Sacr{\'e}, Pierre},
  journal = {IEEE Trans. Autom. Control},
  year    = {2025},
  pages   = {1--8},
  doi     = {10.1109/TAC.2025.3637661}
}

@article{Hartley2019ContactaidedIE,
  title={Contact-aided invariant extended {K}alman filtering for robot state estimation},
  author={Ross Hartley and Maani Ghaffari Jadidi and Ryan M. Eustice and Jessy W. Grizzle},
  journal={Int. J. Robot. Res},
  year={2019},
  volume={39},
  pages={402 - 430},
}

@inproceedings{santana2024,
  title={Proprioceptive state estimation for quadruped robots using invariant Kalman filtering and scale-variant robust cost functions},
  author={Santana, Hilton Marques Souza and Soares, Jo{\~a}o Carlos Virgolino and Nistic{\`o}, Ylenia and Meggiolaro, Marco Antonio and Semini, Claudio},
  booktitle={2024 IEEE-RAS Int. Conf. Humanoid Robots},
  pages={213--220},
  year={2024},
  organization={IEEE}
}

@article{barrau2016,
  title={The invariant extended Kalman filter as a stable observer},
  author={Barrau, Axel and Bonnabel, Silvere},
  journal={IEEE Trans. Autom. Control},
  volume={62},
  number={4},
  pages={1797--1812},
  year={2016},
  publisher={IEEE}
}

@article{barrau2018,
  title={Invariant kalman filtering},
  author={Barrau, Axel and Bonnabel, Silvere},
  journal={Annu. Rev. Control. Robot. Auton. Syst.},
  volume={1},
  number={1},
  pages={237--257},
  year={2018},
  publisher={Annual Reviews}
}

@article{barrau2022,
  title={The geometry of navigation problems},
  author={Barrau, Axel and Bonnabel, Silvere},
  journal={IEEE Trans. Autom. Control},
  volume={68},
  number={2},
  pages={689--704},
  year={2022},
  publisher={IEEE}
}

@article{wang2006error,
  title={Error propagation on the Euclidean group with applications to manipulator kinematics},
  author={Wang, Yunfeng and Chirikjian, Gregory S},
  journal={IEEE Trans. Robot.},
  volume={22},
  number={4},
  pages={591--602},
  year={2006},
  publisher={IEEE}
}

@article{bell1994iterated,
  title={The iterated Kalman smoother as a Gauss--Newton method},
  author={Bell, Bradley M},
  journal={SIAM J. Optim.},
  volume={4},
  number={3},
  pages={626--636},
  year={1994},
  publisher={SIAM}
}

@INPROCEEDINGS{bloesch2013,
  author={Bloesch, Michael and Gehring, Christian and Fankhauser, Péter and Hutter, Marco and Hoepflinger, Mark A. and Siegwart, Roland},
  booktitle={2013 IEEE/RSJ  Int. Conf. Intell. Robots Syst.}, 
  title={State estimation for legged robots on unstable and slippery terrain}, 
  year={2013},
  volume={},
  number={},
  pages={6058-6064},
  doi={10.1109/IROS.2013.6697236}}

@book{shalom2002,
author = {Bar-Shalom, Yaakov and Kirubarajan, Thiagalingam and Li, X.-Rong},
title = {Estimation with Applications to Tracking and Navigation},
year = {2002},
isbn = {0471221279},
publisher = {John Wiley \& Sons, Inc.},
address = {USA},
}

@book{farrell2008,
  title={Aided Navigation: GPS with High Rate Sensors},
  author={Farrell, J.A.},
  isbn={9780071642668},
  series={McGraw-Hill professional engineering: Electronic engineering},
  year={2008},
  publisher={McGraw Hill LLC}
}

@article{sola2018,
  title={A micro Lie theory for state estimation in robotics},
  author={Joan Sol{\`a} and J{\'e}r{\'e}mie Deray and Dinesh Atchuthan},
  journal={ArXiv},
  year={2018},
  volume={abs/1812.01537},
}

@article{ko2005,
  title={State estimation of linear systems with state equality constraints},
  author={Ko, Sangho and Bitmead, Robert R},
  journal={IFAC Proceedings Volumes},
  volume={38},
  number={1},
  pages={241--246},
  year={2005},
  publisher={Elsevier}
}

@article{barrau2019,
  title={Extended Kalman filtering with nonlinear equality constraints: A geometric approach},
  author={Barrau, Axel and Bonnabel, Silvere},
  journal={IEEE Trans. Autom. Control},
  volume={65},
  number={6},
  pages={2325--2338},
  year={2019},
  publisher={IEEE}
}

@ARTICLE{brossard2022,
  title={Associating uncertainty to extended poses for on lie group {IMU} preintegration with rotating earth},
  author={Brossard, Martin and Barrau, Axel and Chauchat, Paul and Bonnabel, Silv{\`e}re},
  journal={IEEE Trans. Robot.},
  volume={38},
  number={2},
  pages={998--1015},
  year={2021},
  publisher={IEEE}
}

@misc{sola2017,
  title={Quaternion kinematics for the error-state Kalman filter},
  author={Sola, Joan},
  journal={arXiv preprint arXiv:1711.02508},
  year={2017}
}

@article{forster2017,
  title = {On-Manifold Preintegration for Real-Time Visual--Inertial Odometry},
  volume = {33},
  ISSN = {1941-0468},
  DOI = {10.1109/tro.2016.2597321},
  number = {1},
  journal = {IEEE Trans. Robot.},
  publisher = {Institute of Electrical and Electronics Engineers (IEEE)},
  author = {Forster,  Christian and Carlone,  Luca and Dellaert,  Frank and Scaramuzza,  Davide},
  year = {2017},
  month = feb,
  pages = {1–21}
}

@inproceedings{bloesch2012,
author = {Bloesch, Michael and Hutter, Marco and Hoepflinger, Mark and Leutenegger, Stefan and Gehring, Christian and Remy, C and Siegwart, Roland},
booktitle={RSS},
year = {2012},
month = {07},
pages = {},
title = {State Estimation for Legged Robots - Consistent Fusion of Leg Kinematics and {IMU}},
doi = {10.15607/RSS.2012.VIII.003}
}

@article{lin2021,
  title={Legged robot state estimation using invariant Kalman filtering and learned contact events},
  author={Lin, Tzu-Yuan and Zhang, Ray and Yu, Justin and Ghaffari, Maani},
  journal={arXiv preprint arXiv:2106.15713},
  year={2021}
}

@article{wisth2023,
  title = {{VILENS}: Visual,  Inertial,  Lidar,  and Leg Odometry for All-Terrain Legged Robots},
  volume = {39},
  ISSN = {1941-0468},
  DOI = {10.1109/tro.2022.3193788},
  number = {1},
  journal = {IEEE Trans. Robot.},
  publisher = {Institute of Electrical and Electronics Engineers (IEEE)},
  author = {Wisth,  David and Camurri,  Marco and Fallon,  Maurice},
  year = {2023},
  month = feb,
  pages = {309–326}
}

@article{fahmi2021,
  title = {On State Estimation for Legged Locomotion Over Soft Terrain},
  volume = {5},
  ISSN = {2475-1472},
  DOI = {10.1109/lsens.2021.3049954},
  number = {1},
  journal = {IEEE Sens. Lett.},
  publisher = {Institute of Electrical and Electronics Engineers (IEEE)},
  author = {Fahmi,  Shamel and Fink,  Geoff and Semini,  Claudio},
  year = {2021},
  month = jan,
  pages = {1–4}
}

@INPROCEEDINGS{youm2024,
  author={Youm, Donghoon and Oh, Hyunsik and Choi, Suyoung and Kim, Hyeongjun and Jeon, Seunghun and Hwangbo, Jemin},
  booktitle={2025 IEEE Int. Conf. Robot. Autom.}, 
  title={Legged Robot State Estimation with Invariant Extended Kalman Filter Using Neural Measurement Network}, 
  year={2025},
  volume={},
  number={},
  pages={670-676},
  doi={10.1109/ICRA55743.2025.11127971}}

@INPROCEEDINGS{lee2025,
  author={Lee, Seokju and Kim, Hyun-Bin and Kim, Kyung-Soo},
  booktitle={2025 IEEE/RSJ Int. Conf. Intell. Robots Syst.}, 
  title={Legged Robot State Estimation Using Invariant Neural-Augmented Kalman Filter with a Neural Compensator}, 
  year={2025},
  volume={},
  number={},
  pages={15445-15452},
  doi={10.1109/IROS60139.2025.11247668}}

@book{stillwell2008,
  title={Naive Lie Theory},
  author={Stillwell, J.},
  isbn={9780387782157},
  lccn={2008927921},
  series={Undergraduate Texts in Mathematics},
  year={2008},
  publisher={Springer New York}
}

@article{brown1997,
  title={Introduction to random signals and applied Kalman filtering: with MATLAB exercises and solutions},
  author={Brown, Robert Grover and Hwang, Patrick YC},
  journal={Introduction to random signals and applied Kalman filtering: with MATLAB exercises and solutions},
  year={1997}
}

@article{camurri2020,
  title={Pronto: A multi-sensor state estimator for legged robots in real-world scenarios},
  author={Camurri, Marco and Ramezani, Milad and Nobili, Simona and Fallon, Maurice},
  journal={ Front. Robot. AI},
  volume={7},
  pages={68},
  year={2020},
  publisher={Frontiers Media SA}
}

@article{lin2023,
  title={Proprioceptive invariant robot state estimation},
  author={Lin, Tzu-Yuan and Li, Tingjun and Tong, Wenzhe and Ghaffari, Maani},
  journal={arXiv preprint arXiv:2311.04320},
  year={2023}
}

@article{bourmaud2016,
  title={From intrinsic optimization to iterated extended Kalman filtering on Lie groups},
  author={Bourmaud, Guillaume and M{\'e}gret, R{\'e}mi and Giremus, Audrey and Berthoumieu, Yannick},
  journal={J. Math. Imaging Vis.},
  volume={55},
  number={3},
  pages={284--303},
  year={2016},
  publisher={Springer}
}

@article{he2021,
  title={Kalman filters on differentiable manifolds},
  author={He, Dongjiao and Xu, Wei and Zhang, Fu},
  journal={arXiv preprint arXiv:2102.03804},
  year={2021}
}

@book{corke2023,
  title     = {Robotics, Vision and Control: Fundamental Algorithms in MATLAB},
  edition   = {3},
  author    = {Corke, Peter and Jachimczyk, Witold and Pillat, Remo},
  year      = {2023},
  publisher = {Springer International Publishing},
  isbn      = {978-3-031-07261-1},
  language  = {English}
}

@inproceedings{hutter2016anymal,
  title={ANYmal-a highly mobile quadrupedal robot},
  author={Hutter, Marco and Gehring, Christian and Jud, Dominic and Lauber, Andreas and Bellicoso, Carmelo C and Tsounis, Vassilios and Hwangbo, Jemin and Bodie, Karen and Fankhauser, Peter and Bloesch, Michael and others},
  booktitle={2016 IEEE/RSJ Int. Conf. Intell. Robots Syst.},
  pages={3136--3143},
  year={2016},
  organization={IEEE}
}

@article{barrau2015,
  title={An EKF-SLAM algorithm with consistency properties},
  author={Barrau, Axel and Bonnabel, Silvere},
  journal={arXiv preprint arXiv:1510.06263},
  year={2015}
}

@ARTICLE{nistico2025,
       author = {{Nistic\`o}, Ylenia and {Soares}, Joo Carlos Virgolino and {Amatucci}, Lorenzo and {Fink}, Geoff and {Semini}, Claudio},
        title = "{MUSE: A Real-Time Multi-Sensor State Estimator for Quadruped Robots}",
      journal = {IEEE Robot. Autom. Lett},
         year = 2025,
        month = jan,
       volume = {10},
       number = {5},
        pages = {4620-4627},
          doi = {10.1109/LRA.2025.3553047},
archivePrefix = {arXiv},
       eprint = {2503.12101},
 primaryClass = {cs.RO}
}

@book{jazwinski1970,
  address = {New York, NY [u.a.]},
  author = {Jazwinski, {Andrew H.}},
  isbn = {0123815509},
  number = 64,
  pagetotal = {XIV, 376},
  ppn_gvk = {021832242},
  publisher = {Acad. Press},
  series = {Math. Sci. Eng.},
  title = {Stochastic processes and filtering theory},
  year = 1970
}

@phdthesis{barrau2015thesis,
  title={Non-linear state error based extended Kalman filters with applications to navigation},
  author={Barrau, Axel},
  year={2015},
  school={Mines Paristech}
}

@phdthesis{brossard2020,
  title={Deep learning, inertial measurements units, and odometry: some modern prototyping techniques for navigation based on multi-sensor fusion},
  author={Brossard, Martin},
  year={2020},
  school={Universit{\'e} Paris sciences et lettres}
}

@incollection{hall2013,
  title={Lie groups, Lie algebras, and representations},
  author={Hall, Brian C},
  booktitle={Quantum Theory Math.},
  pages={333--366},
  year={2013},
  publisher={Springer}
}

@INPROCEEDINGS{mujoco,
  author={Todorov, Emanuel and Erez, Tom and Tassa, Yuval},
  booktitle={2012 IEEE/RSJ Int. Conf. Intell. Robots Syst.}, 
  title={{M}u{J}o{C}o: A physics engine for model-based control}, 
  year={2012},
  volume={},
  number={},
  pages={5026-5033},
  doi={10.1109/IROS.2012.6386109}}

@ARTICLE{Yoon2024,
  author={Yoon, Ziwon and Kim, Joon-Ha and Park, Hae-Won},
  journal={IEEE Trans. Robot.}, 
  title={Invariant Smoother for Legged Robot State Estimation With Dynamic Contact Event Information}, 
  year={2024},
  volume={40},
  number={},
  pages={193-212},
  doi={10.1109/TRO.2023.3328202}}

@INPROCEEDINGS{turrisi2024,
  author={Turrisi, Giulio and Modugno, Valerio and Amatucci, Lorenzo and Kanoulas, Dimitrios and Semini, Claudio},
  booktitle={2024 IEEE/RSJ Int. Conf. Intell. Robots Syst.}, 
  title={On the Benefits of {GPU} Sample-Based Stochastic Predictive Controllers for Legged Locomotion}, 
  year={2024},
  pages={13757-13764},
  doi={10.1109/IROS58592.2024.10801698}}

@article{frey2026grandtour,
  title={GrandTour: A Legged Robotics Dataset in the Wild for Multi-Modal Perception and State Estimation},
  author={Frey, Jonas and Tuna, Turcan and Fu, Frank and Patterson, Katharine and Xu, Tianao and Fallon, Maurice and Cadena, Cesar and Hutter, Marco},
  journal={arXiv preprint arXiv:2602.18164},
  year={2026}
}

@INPROCEEDINGS{fink2020iros,
  author={Fink, Geoff and Semini, Claudio},
  booktitle={2020 IEEE/RSJ Int. Conf. Intell. Robots Syst.}, 
  title={Proprioceptive Sensor Fusion for Quadruped Robot State Estimation}, 
  year={2020},
  volume={},
  number={},
  pages={10914-10920},
  doi={10.1109/IROS45743.2020.9341521}}


\end{document}